\documentclass{article}

\usepackage[final,main]{neurips_2026}

\usepackage[utf8]{inputenc}
\usepackage[T1]{fontenc}
\usepackage{hyperref}
\usepackage{url}
\usepackage{booktabs}
\usepackage{nicefrac}
\usepackage{microtype}
\usepackage{xcolor}
\definecolor{mydarkblue}{rgb}{0,0.08,0.45}
\hypersetup{colorlinks=true, linkcolor=mydarkblue, citecolor=mydarkblue, filecolor=mydarkblue, urlcolor=mydarkblue}
\usepackage{colortbl}
\usepackage{enumitem}
\usepackage{pifont}
\usepackage{wrapfig}

\providecommand{\tablestyle}[2]{\setlength{\tabcolsep}{#1}\renewcommand{\arraystretch}{#2}}
\definecolor{cOurs}{HTML}{c4553a}      
\definecolor{cRPNI}{HTML}{3d5a80}      
\definecolor{cAlergia}{HTML}{2a8f82}   
\definecolor{cOursBg}{HTML}{f5e3df}    
\definecolor{cRPNIBg}{HTML}{dfe6ed}    
\definecolor{cAlergiaBg}{HTML}{ddeeed} 
\definecolor{cGroupHd}{HTML}{f0eeeb}   
\definecolor{cBest}{HTML}{fef3e0}      
\definecolor{cHMMBg}{RGB}{237, 237, 247}
\definecolor{cProcessBg}{RGB}{247, 237, 247}
\definecolor{cNeutralBg}{RGB}{250, 250, 250}
\providecommand{\hlbest}[1]{\cellcolor{cBest}\textbf{#1}}

\newcommand{\cmark}{\ding{51}}
\newcommand{\xmark}{\ding{55}}

\usepackage{amsmath,amssymb,amsthm}
\usepackage{mathtools}
\usepackage{bm}

\usepackage[ruled,vlined,linesnumbered]{algorithm2e}

\usepackage{graphicx}
\usepackage{subcaption}
\usepackage{multirow}
\usepackage{adjustbox}
\usepackage{float}

\usepackage{tikz}
\usetikzlibrary{automata,positioning,arrows.meta,calc}

\newtheorem{theorem}{Theorem}

\newtheorem{proposition}[theorem]{Proposition}
\newtheorem{corollary}[theorem]{Corollary}
\newtheorem{definition}{Definition}
\newtheorem{remark}{Remark}

\newcommand{\cA}{\mathcal{A}}

\newcommand{\cM}{\mathcal{M}}
\newcommand{\cP}{\mathcal{P}}

\newcommand{\cD}{\mathcal{D}}

\title{Automata from Agent Traces:\\ Failure and Next-Step Prediction}

\author{%
  Seonglae Cho$^{1}$ \quad
  Franklin Cardenoso Fernandez$^{1,3}$ \quad
  Umar Mohammed$^{1}$ \quad
  Zekun Wu$^{2}$ \\
  Kleyton Da Costa$^{2}$ \quad
  Ilham Wicaksono$^{1}$ \quad
  Adriano Koshiyama$^{2}$ \\[3pt]
  $^{1}$Holistic AI \quad $^{2}$University College London \quad $^{3}$PUC-Rio
}

\newif\ifarxiv\arxivtrue
\def\seongland{\raisebox{-1.5pt}{\includegraphics[height=1.05em]{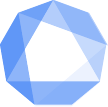}}}
\newcommand{\articlelink}{https://seongland.com/article/asg}

\begin{document}

\maketitle

\ifarxiv
\vspace{-0.32in}
\begin{center}
\begin{tabular}{rl}
\seongland & \url{\articlelink}\\
\end{tabular}
\end{center}
\fi

\begin{abstract}
\spaceskip=\fontdimen2\font plus \fontdimen3\font minus 0.05\fontdimen4\font
\emergencystretch=2em
LLM-based agents execute multi-step tasks, but their behavioral structure remains opaque: their long, unstructured traces resist the safety auditing and runtime monitoring that deployment requires.
Existing approaches look at one trace at a time or at successful runs only, missing the structure shared across runs that links \emph{next-step} and \emph{failure} prediction. We extract finite-state machines (FSMs) as the compact directly-follows automaton of the observed traces via prefix-tree construction and last-activity merging, providing a structural substrate for the unpredictable nature of LLM agent behavior.
Across twelve public datasets, the FSMs are compact (7--43 states), achieve high replay fitness on held-out data with zero structural variance across splits, and build in milliseconds.
This substrate addresses both prediction goals.
Provided as context, the FSM beats Agent Workflow Memory at next-step prediction on all eight ground-truth-matched datasets, six of them significantly.
Per-state behavioral features predict task failure at held-out AUROC up to 0.94, and an online monitor ranks failing SWE-agent runs above passing ones from a partial trace, beating the trivial flag-everything baseline on rank-AUROC and triggering early stopping well before completion.
On tau2-bench, one FSM built from the traces of four LLMs replays each of them at 1.000 fitness, evidence that within these task suites behavioral topology is shaped more by the deployment harness than by the LLM, giving a structural primitive for safety auditing and runtime monitoring at deployment.
\end{abstract}

\section{Introduction}\label{sec:intro}

LLM agents~\citep{wang2024survey_agents, sumers2024coala} execute tens to hundreds of steps per task, yet the state structure underlying their traces is never written down.
Recent systems span GitHub issue resolution~\citep{yang2024sweagent, yang2025swesmith}, web navigation~\citep{deng2024mind2web, zhou2024webarena}, desktop environments~\citep{xie2024osworld, wang2025opencua}, and multi-agent pipelines~\citep{wu2023autogen, hong2023metagpt}. The risk of undetected failures scales with their autonomy.
Following the ReAct paradigm~\citep{yao2023react}, they interleave chain-of-thought reasoning~\citep{wei2022chain} with tool calls~\citep{schick2023toolformer}, generating execution traces whose behavioral structure remains implicit.
A coding agent cycles through \texttt{search}$\to$\texttt{edit}$\to$\texttt{execute}; a customer service agent alternates between database queries and user communication.

Understanding this latent structure matters for safety auditing~\citep{zhang2025asb, ruan2024toolemu, chen2025shieldagent}, locating the bottleneck states where agents get stuck~\citep{yang2025whoandwhen, cemri2025multiagent_fail}, and monitoring behavioral drift in production~\citep{wang2026agentspec}.
Yet current approaches operate at the individual trace level, requiring task descriptions, manual specification, or success filters~\citep{chen2025metaagent, wu2024stateflow, wang2024agent_workflow_memory}.

\begin{figure}[!t]
\centering
\includegraphics[width=\linewidth]{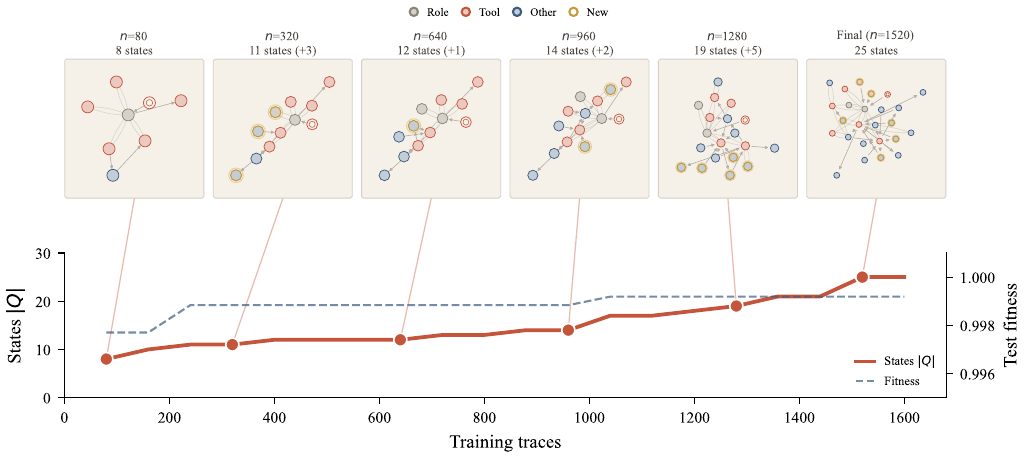}
\caption{\textbf{FSM evolution} on SWE-agent. State count $|Q|$ (red, left) and test fitness (blue, right) over training traces, with snapshots of the FSM at six milestones. The state space grows incrementally as new modes of behavior appear, while fitness saturates early; construction completes in milliseconds.}
\label{fig:evolution}
\end{figure}

We frame behavioral recovery as an inverse problem: given a corpus of execution traces, reconstruct a finite state machine (FSM) that explains the observed behavior.
Agent traces provide only positive examples in the Gold sense~\citep{gold1967language, angluin1980inductive} (observed runs from the agent's behavior, including both successful and failed trajectories; no synthetic counter-examples), and identifying the target language from positive examples alone is impossible in the limit.
Our key observation is that agent behavior is generated by a bounded set of tools and actions, producing traces with small activity alphabets (6--42 symbols).
Moreover, the resulting behavioral topology appears shaped more by the \emph{system} (tools, tasks, and prompts) than by the LLM, across 4 chat models on tau2-bench: a single FSM achieves perfect fitness on every model.
This structural constraint makes the problem tractable: a prefix tree merged by last activity produces deterministic FSMs that are compact ($|Q|=|\cA|+1$) in linear time, requiring no hyperparameters for learning (the only design choice is the function that extracts activities, whose robustness we verify in Appendix~\ref{app:granularity}).
We evaluate on twelve public datasets (Table~\ref{tab:datasets}) against nine baselines from automata learning (RPNI, EDSM, Alergia, k-Tails), HMMs, process mining, and workflow extraction (\S\ref{sec:related}):

\begin{itemize}[leftmargin=*,itemsep=0pt,parsep=0pt,topsep=2pt]
  \item \textbf{Workflow memory.} FSM-state context beats Agent Workflow Memory~\citep{wang2024agent_workflow_memory} on all 8 datasets in ground-truth-matched next-step prediction (7 held-out, 1 used for format selection).
  \item \textbf{Next-step prediction.} FSM state conditioning lowers cross-entropy in a controlled ablation, where RPNI overfits catastrophically.
  \item \textbf{Failure prediction and runtime monitoring.} FSM cross-entropy anomaly features predict task failure and support an online monitor that flags failing SWE-agent runs from a partial trace. Per-state features further improve MLP, GRU, and Transformer baselines, and beat Alergia-extracted FSMs under a matched feature pipeline, isolating per-state observation density. Per-model features transfer across four LLMs on tau2-bench.
  \item \textbf{Side capability: compression.} The same abstraction produces compact automata at high fitness and zero structural variance across splits, at or below the state count of Alergia, the strongest positive-only baseline.
\end{itemize}

We prove that the compactness of the directly-follows automaton (Theorem~\ref{prop:minimal}, $|Q|=|\cA|+1$) is exactly what makes per-state estimators concentrate at rate $O(1/\sqrt{n_q})$ (Proposition~\ref{prop:estimator}), which makes the FSM cross-entropy the asymptotically Neyman--Pearson optimal anomaly signal (Corollary~\ref{cor:surprise}).

\section{Related Work}\label{sec:related}

\paragraph{Grammatical inference.}
Learning finite automata from positive examples alone is impossible in the limit~\citep{gold1967language, angluin1980inductive}. RPNI~\citep{oncina1992rpni}, EDSM~\citep{lang1998edsm}, and L*~\citep{angluin1987learning} require negative examples or oracles unavailable in trace analysis. k-Tails~\citep{biermann1972ktails} merges states with identical $k$-length futures, but requires a hyperparameter and produces at least as many states as ours at equal or lower fitness. Among positive-only methods, Alergia~\citep{carrasco1994alergia} is the strongest competitor: it matches our fitness with 1.0--6.0$\times$ more states via statistical compatibility tests. HMMs~\citep{rabiner1989hmm} match state counts but yield non-interpretable latent states. Our approach exploits bounded activity alphabets to produce compact, interpretable FSMs without hyperparameters.

\paragraph{Process mining.}
Process discovery~\citep{vanderaalst2016process} recovers Petri nets from event logs. Applied to agent traces, standard miners produce ``flower models'' of low precision (Table~\ref{tab:pm4py_detail}), highest on constrained workflows~\citep{grohs2024process_mining_llm, berti2024processmining}. Our automaton is the directly-follows graph~\citep{vanderaalst2016process} made deterministic by a last-activity right congruence; the closest learning-based variant is stochastic directly-follows discovery via grammatical inference~\citep{alkhammash2024stochastic}, which tunes a soundness objective for business-process event logs, whereas we use a single deterministic pass with a convergence guarantee and apply the result to failure prediction, next-step prediction, and monitoring for LLM agents.

\paragraph{Agent workflow and FSM methods.}
Agent Workflow Memory~\citep{wang2024agent_workflow_memory} extracts workflow patterns from successful traces; ReasoningBank~\citep{ouyang2025reasoningbank} extends AWM with both successful and failed traces.
Reflexion~\citep{shinn2023reflexion} and ETO~\citep{song2024eto} learn from failures via verbal reflection or contrastive pairs; none produce structural models with state abstraction.
On the FSM side, AFlow~\citep{zhang2025aflow} searches workflows via MCTS, MetaAgent~\citep{chen2025metaagent} builds FSMs top-down from task descriptions, and StateFlow~\citep{wu2024stateflow} relies on manual specification.
Our method recovers a compact directly-follows automaton bottom-up from raw traces.

\paragraph{Agent safety and trajectory anomaly detection.}
AgentSpec~\citep{wang2026agentspec} and ShieldAgent~\citep{chen2025shieldagent} enforce predefined safety policies; AgentMonitor~\citep{chan2024agentmonitor} predicts task performance from step-level XGBoost features. ProbGuard~\citep{wang2025probguard} learns a DTMC for reachability-based monitoring (head-to-head in Appendix~\ref{app:pro2guard}); concurrent trajectory-anomaly detectors~\citep{liu2025traceaegis, deshpande2025trail, he2025sentinelagent} target the same problem with hierarchical or graph-based pipelines. \citet{cemri2025multiagent_fail} taxonomize multi-agent failure modes.

\section{Method}\label{sec:method}

\subsection{Problem Formulation}\label{sec:method:problem}

An agent trace $\tau = (m_1, \ldots, m_T)$ consists of $T \geq 1$ messages.
Each message has a role (system, user, assistant, tool) and content.
We map each message to an activity symbol $a_t$ with the activity extraction function $\phi: m_t \mapsto a_t \in \cA$ of Section~\ref{sec:method:extraction}, where $\cA$ is a finite alphabet.
We call $\sigma(\tau) = (\phi(m_1), \ldots, \phi(m_T)) = (a_1, \ldots, a_T)$ the activity sequence of $\tau$.

\spaceskip=\fontdimen2\font plus \fontdimen3\font minus 0.05\fontdimen4\font\relax
Given a corpus of $N$ traces $\cD = \{\tau_1, \ldots, \tau_N\}$, Algorithm~\ref{alg:fsm} constructs an FSM $\cM = (Q, \cA, \delta, q_0, F)$.
$Q$ is its set of states and $q_0\!\in\!Q$ its initial state.
Every state accepts ($F\!=\!Q$), and we report acceptance via the replay fitness of Definition~\ref{def:fitness} rather than a binary accept/reject.
The partial transition function $\delta: Q \times \cA \rightharpoonup Q$ is deterministic, so each $(q,a)$ has at most one successor.

\spaceskip=0pt\relax
\begin{definition}[Replay fitness]\label{def:fitness}
For activity sequence $\sigma = (a_1, \ldots, a_T)$, let $k$ be the number of symbols consumed by $\cM$ when replaying $\sigma$ from $q_0$ (i.e., steps where $\delta(q, a_t)$ is defined). The replay fitness is $\mathrm{fit}(\sigma, \cM) = k / T$, so we can define the corpus fitness as
\begin{equation}\label{eq:corpus_fitness}
\mathrm{Fit}(\cD, \cM) = \frac{1}{|\cD|} \sum_{\tau \in \cD} \mathrm{fit}(\sigma(\tau), \cM).
\end{equation}
\end{definition}
\vspace{-8pt}
We write $\delta^*(q, w)$ for the state reached from $q$ by consuming $w \in \cA^*$, defined whenever every intermediate transition is defined.

\subsection{Activity Extraction}\label{sec:method:extraction}

$\phi$ is a deterministic mapping from messages to activities, which we apply in priority order over (i)~tool-call fields (e.g., the function name when a message carries a \texttt{tool\_call}), (ii)~action tags such as \texttt{[ACTION] description}, and (iii)~command tokens for code-block agents. If no rule matches, we fall back to \texttt{role:content\_type} (we give the rules for each dataset in Appendix~\ref{app:datasets}).
Across the extraction granularities we tested, replay fitness stays $\geq\!0.999$ on every dataset, and our default granularity (role-type) matches or exceeds role-only held-out AUROC on 10 of 12 datasets, and differs from it by 0.05 or more only on datasets whose role-only alphabet degenerates to at most three symbols (Appendix~\ref{app:granularity}). $\phi$ is therefore one valid setting rather than the only one.

\subsection{FSM Construction}\label{sec:method:construction}

Given activity sequences $\{\sigma(\tau_i)\}_{i=1}^N$, construction proceeds in three steps.

\paragraph{Step 1: Prefix tree.}
Insert all activity sequences into a trie. Each unique prefix is a distinct state. The prefix tree has perfect training fitness but $O(\sum_i T_i)$ states.

\paragraph{Step 2: Merge by last activity.}
We merge all trie states reached by the same activity into one. Writing $\kappa(q)$ for the activity on the edge into $q$ (and $\kappa(q_\varepsilon) = \mathrm{init}$ for the root), we merge states by the last-activity right congruence
\begin{equation}\label{eq:struct_equiv}
q \sim q' \iff \kappa(q) = \kappa(q'),
\end{equation}
which has $|\cA| + 1$ classes. Adding up the trie's traversal counts over each class pair gives the directly-follows automaton in a single pass.

Cycles arise because the merged states sit at different trie depths. In the trace \texttt{user} $\to$ \texttt{search} $\to$ \texttt{edit} $\to$ \texttt{search} $\to$ \texttt{execute}, the trie holds two distinct \texttt{search} nodes, $q_2$ reached by \texttt{user}\,$\to$\,\texttt{search} and $q_4$ reached by $\ldots\to$\,\texttt{edit}\,$\to$\,\texttt{search}. Both carry incoming activity \texttt{search}, so $\kappa(q_2) = \kappa(q_4)$ and the two collapse into the single class $[\texttt{search}]$. The edge into $q_4$ then arrives at $[\texttt{search}]$ while the edge out of $q_2$ leaves it, so the segment \texttt{search}\,$\to$\,\texttt{edit}\,$\to$\,\texttt{search} becomes the two-cycle $[\texttt{search}] \to [\texttt{edit}] \to [\texttt{search}]$.

The construction therefore yields two distinct objects. The prefix tree of Step~1 accepts exactly the observed prefixes and is acyclic. The quotient of Step~2 accepts a superset of them, strict whenever two prefixes share a final activity but carry different continuations, and that over-approximation is what lets the automaton replay held-out traces at all. It is not vacuous: the extracted automaton still rejects every random trace and at least 99.9\% of permuted ones on the datasets tested (Appendix~\ref{app:precision}, Table~\ref{tab:precision}), so the generalization it buys is bounded by the observed directly-follows relation rather than unconstrained.

\paragraph{Step 3: Rare-transition filtering.}
A merged transition observed exactly once in the corpus is dropped unless it is its source state's only continuation. This removes one-off digressions while never removing a state.

Section~\ref{sec:method:theory} proves that this construction preserves fitness and yields a compact directly-follows automaton. Recurring tool-use patterns (search--edit--execute) collapse to loops and the state count tracks the number of distinct activities.
Figure~\ref{fig:example_fsm} (Appendix~\ref{app:algorithm}) shows the construction at role-level granularity for a customer service agent (6 states, 5 activities).
Full pseudocode is given in Algorithm~\ref{alg:fsm} (Appendix~\ref{app:algorithm}).

\subsection{Construction and Convergence Guarantees}\label{sec:method:theory}

\begin{theorem}[Fitness preservation]\label{prop:fitness} The last-activity merge preserves training fitness: if trace $\sigma$ is accepted by the prefix tree, it is accepted by the Step~2 quotient.
\end{theorem}

\begin{proof}[Proof sketch]
Merging only adds out-edges, since we give each class the union of its members' transitions, and so every trie edge survives: $\delta(q_i, a_{i+1}) = q_{i+1}$ implies $\delta([q_i], a_{i+1}) = [q_{i+1}]$, and any trace accepted by the prefix tree is accepted by the merged FSM. Step~3 filtering is the only source of fitness loss, and the replay-fitness columns measure it directly (full proof in Appendix~\ref{app:proofs}).
\end{proof}

\begin{theorem}[Determinism and compactness]\label{prop:minimal} Merging the prefix tree by the last-activity right congruence yields a deterministic FSM with $|Q| = |\cA| + 1$ states---one per activity plus the initial state---whose transitions are the directly-follows pairs retained from the corpus. The construction is a deterministic function of the corpus, so re-extraction from the same data is exact.
\end{theorem}

\begin{proof}[Proof sketch] Each trie state is assigned to the class of its incoming activity, giving $|\cA|+1$ classes; transitions are the retained directly-follows pairs, deduplicated, so each $(q, a)$ has at most one target and the FSM is deterministic. Step~3 removes transitions but never states, so determinism and the state count are unaffected. The class map and the transition set depend only on the multiset of observed (activity, next-activity) pairs, hence are invariant to trace order and sampling, so the output is unique for a fixed corpus. (Full proof in Appendix~\ref{app:proofs}.) \end{proof}

Theorems~\ref{prop:fitness} and~\ref{prop:minimal} describe the Step~2 quotient, not the acyclic prefix tree it is built from. We recover the directly-follows automaton of the observed traces, not the generating automaton, which is impossible to identify from positive examples alone~\citep{gold1967language}.

Empirically, all twelve datasets yield 7--43 states (Table~\ref{tab:main}), because agent activity sequences exhibit strong sequential regularity: conditioning on the previous symbol reduces entropy by 51--68\% (Appendix~\ref{app:entropy}).
Compact state spaces aggregate sufficient observations per state for reliable probability estimation, unlike RPNI's $10^3$--$10^5$ states.

\begin{remark}[Complexity] Prefix tree construction is $O(\sum_i T_i)$.
In the last-activity merge we group the $|Q_\cP|$ trie states by incoming activity in a single $O(|Q_\cP| \cdot |\cA|)$ pass, where $|Q_\cP|$ is the number of prefix tree states.
The total runtime is linear in the corpus size for bounded $|\cA|$.
In practice, all twelve datasets complete in under one second on a single CPU core.
\end{remark}

\begin{proposition}[Convergence guarantee]\label{prop:convergence} Let $\cM^*$ be the population directly-follows automaton, with $r$ transitions.
If traces are drawn i.i.d.\ and each transition appears in a trace with probability at least $p_{\min}$, then for any $\delta_\mathrm{f} > 0$, the extracted FSM equals $\cM^*$ with probability $\geq 1 - \delta_\mathrm{f}$ after $N \geq \frac{1}{p_{\min}} \ln(r/\delta_\mathrm{f})$ traces.
\end{proposition}

A union bound over $r$ transitions gives the failure-probability chain
\begin{equation}\label{eq:union_bound}
\begin{aligned}
\Pr[\text{some transition unobserved}] \;&\leq\; r(1-p_{\min})^N \\
\;&\leq\; r\,e^{-Np_{\min}} \;\leq\; \delta_\mathrm{f},
\end{aligned}
\end{equation}
which yields the bound (full proof in Appendix~\ref{app:proofs}).
The i.i.d.\ assumption is approximate in practice. The extractor depends only on the multiset of observed (activity, next-activity) pairs, so the automaton is invariant to trace order; dependence between traces bears on how many traces are needed, which we measure to rise by about 11\% on average.
The bound remains useful because the convergence requirement is weak ($N \approx 693$ for SWE-agent's $r=51$ transitions at $p_{\min} \approx 0.01$, $\delta_\mathrm{f} = 0.05$), and empirical convergence well inside the training set (Figure~\ref{fig:convergence}) suggests the bound is conservative even under mild distributional shift.

\subsection{Prediction via FSM State Conditioning}\label{sec:method:prediction}

Given the current FSM state $q_t = \delta^*(q_0, a_1 \ldots a_{t-1})$, we estimate $P(a_t \mid q_t)$ from transition counts:
\begin{equation}\label{eq:fsm_predict}
\hat{P}(a \mid q) = \frac{C(q, a) + \alpha}{\sum_{a' \in \cA} C(q, a') + \alpha |\cA|}
\end{equation}
where $C(q, a)$ counts how often activity $a$ follows state $q$ in training data and $\alpha\!\in\!(0,1]$ is a Lidstone smoothing parameter (we use $\alpha\!=\!1$, add-one smoothing). Higher-order context can be incorporated via prediction by partial matching (PPM)~\citep{cleary1984ppm} with absolute discounting~\citep{ney1994structuring}, blending FSM-conditioned predictions across context depths (Appendix~\ref{app:prediction}). We evaluate predictive quality via cross-entropy on the estimator conditioned on the FSM,
\begin{equation}\label{eq:ce}
\mathrm{CE}_{\text{FSM}} = -\frac{1}{T}\sum_{t=1}^{T} \log_2 \hat{P}(a_t \mid q_t),
\end{equation}
with the same form for any other conditioning context.

The plug-in estimator $\hat{P}(\cdot\mid q)$ is consistent under i.i.d.\ traces with total-variation error $O(1/\sqrt{n_q})$, where $n_q$ is the visit count of state $q$ (Proposition~\ref{prop:estimator}, Appendix~\ref{app:proofs}). Because $|Q|=O(|\cA|)$, each state aggregates $n_q=\Theta(\sum_i T_i/|Q|)$ transitions, while RPNI's $|Q_{\text{RPNI}}|\!\gg\!|\cA|$ partitions the same observations into sparsely visited states and vacuates the bound (matching the empirical failure of estimators conditioned on RPNI, Figure~\ref{fig:nextstep}).
Under FSM-structured success/failure mixtures $P^{\pm}(\cdot\mid q)$, the surprise difference $\mathrm{CE}^{-}(\tau)\!-\!\mathrm{CE}^{+}(\tau)$ equals the per-step average log-likelihood ratio of success to failure and is Neyman--Pearson-optimal for the population test, with plug-in power-loss $O(1/\sqrt{n_q})$ (Corollary~\ref{cor:surprise}, Appendix~\ref{app:proofs}). This explains why FSM cross-entropy anomaly features outperform hand-crafted trace statistics in Section~\ref{sec:exp:downstream}.

\paragraph{Outcome-aware quotient via Information Bottleneck.}
The last-activity quotient is purely syntactic.
With labeled traces, we can refine it via the Information-Bottleneck (IB) principle~\citep{tishby1999information}: cluster activities into $K$ super-states $C$ minimizing $\mathcal{L}_{\text{IB}}(q;\beta) = I(C;A) - \beta I(C;Y)$, trading compactness for outcome relevance, and iterate on the super-state transition matrix.

\begin{proposition}[IB-quotient outcome-information bound]\label{prop:ib_bound}
For any deterministic $K$-partition $C$ of $\cA$, $H(Y\mid C) \geq H(Y\mid A)$, with equality iff $C$ refines the activity partition by conditional outcome distribution. The IB minimizer realizes a rate-distortion frontier on which the last-activity quotient is the syntactic-rate corner ($\beta\!\to\!0$) and the discriminative quotient (Appendix~\ref{app:fsm_d}) is the outcome-rate corner.
\end{proposition}

Proposition~\ref{prop:ib_bound} follows from the data-processing inequality and rate-distortion theory. The IB hierarchy lifts held-out failure-prediction AUROC by $0.044$ on average over level 0 across eight labeled datasets, largest where the alphabet aligns weakly with outcomes (ATBench $+0.143$, Appendix~\ref{app:hierarchical_ib}).

\subsection{Evaluation Metrics}\label{sec:method:metrics}

We evaluate three further properties beyond \emph{replay fitness} and \emph{cross-entropy}.
\emph{Precision}: the fraction of invalid traces the FSM rejects.
We generate random traces drawn uniformly over $\cA^L$ and permuted traces that shuffle real sequences; low acceptance indicates meaningful constraints on order.
\emph{Compression}: $|Q_{\text{baseline}}| / |Q_{\text{ours}}|$, which measures compactness relative to baseline automata.
\emph{Stability}: variance in FSM structure across random train/test splits.


\section{Results}\label{sec:experiments}

We apply the metrics of \S\ref{sec:method:metrics} (replay fitness, cross-entropy, precision, compression, stability) to twelve datasets, evaluating \emph{next-step prediction} in two protocols: FSM-conditioned cross-entropy (Table~\ref{tab:prediction_full}) and LLM top-1 against the recorded next activity (Table~\ref{tab:fsm_context_llm}).

\subsection{Setup}\label{sec:exp:setup}

\begin{wraptable}{r}{0.5\textwidth}
\vspace{-12pt}
\centering
\caption{\textbf{FSM extraction results} on eight labeled real-trace datasets (excluding SWE-smith synthetic). $|Q|$: states. Fit: replay fitness. $^\dagger$RPNI timeout at 120\,s. Full baselines in Table~\ref{tab:full_baselines}; SWE-smith and unlabeled datasets in Appendix~\ref{app:unlabeled}.}
\label{tab:main}
\scriptsize
\setlength{\tabcolsep}{1.8pt}
\renewcommand{\arraystretch}{0.88}
\begin{tabular}{@{}l rr rr rr r@{}}
\toprule
& \multicolumn{2}{c}{\cellcolor{cOursBg}Ours} & \multicolumn{2}{c}{RPNI} & \multicolumn{2}{c}{Alergia} & \\
\cmidrule(lr){2-3} \cmidrule(lr){4-5} \cmidrule(lr){6-7}
Dataset & \cellcolor{cOursBg}$|Q|$ & \cellcolor{cOursBg}Fit & $|Q|$ & Fit & $|Q|$ & Fit & Compr. \\
\midrule
SWE-agent & \cellcolor{cOursBg}\textbf{25} & \cellcolor{cOursBg}0.999 & 59{,}510$^\dagger$ & 0.646 & 35 & 0.999 & 2{,}380$\times$ \\
WebArena & \cellcolor{cOursBg}\textbf{25} & \cellcolor{cOursBg}1.000 & 382 & 1.000 & 149 & 1.000 & 15$\times$ \\
AgentNet & \cellcolor{cOursBg}\textbf{25} & \cellcolor{cOursBg}1.000 & 62{,}495$^\dagger$ & 0.742 & 45 & 1.000 & 2{,}500$\times$ \\
tau2 (air) & \cellcolor{cOursBg}\textbf{18} & \cellcolor{cOursBg}1.000 & 6{,}506$^\dagger$ & 0.844 & 23 & 0.999 & 361$\times$ \\
tau2 (ret) & \cellcolor{cOursBg}\textbf{19} & \cellcolor{cOursBg}1.000 & 14{,}249$^\dagger$ & 0.837 & 25 & 1.000 & 750$\times$ \\
tau2 (tel) & \cellcolor{cOursBg}\textbf{43} & \cellcolor{cOursBg}1.000 & 63{,}897$^\dagger$ & 0.491 & 75 & 0.999 & 1{,}486$\times$ \\
ATBench    & \cellcolor{cOursBg}\textbf{15} & \cellcolor{cOursBg}1.000 &    899$^\dagger$ & 0.984 & 15 & 1.000 &    60$\times$ \\
OSWorld    & \cellcolor{cOursBg}\textbf{27} & \cellcolor{cOursBg}0.997 & 38{,}232$^\dagger$ & 0.706 & 31 & 0.999 & 1{,}416$\times$ \\
\bottomrule
\end{tabular}
\vspace{-8pt}
\end{wraptable}
We evaluate twelve datasets across eight agent domains (Table~\ref{tab:datasets}, Appendix~\ref{app:datasets}), with alphabets of 6--42 symbols and 80/20 train/test splits.
Nine datasets have outcome labels and are used for failure prediction (eight real LLM-trace datasets and SWE-smith, the lone synthetic dataset); three contribute compression and next-step prediction results only because they lack usable outcome labels (Appendix~\ref{app:unlabeled}).
Baselines: RPNI, EDSM, Alergia, k-Tails (automata learning via AALpy~\citep{muskardin2022aalpy}); HMM; Alpha, Inductive, Heuristic Miners (process mining via PM4Py~\citep{berti2019pm4py}), whose per-dataset fitness and conformance precision (0.00--0.80) appear in Appendix~\ref{app:pm4py}; AWM (workflow extraction). For k-Tails, each activity is encoded as the output at its step over a singleton input alphabet, so $k\!\geq\!2$ growth reflects $k$-future proliferation rather than input branching.
All receive identical training sequences with positive examples only.

\subsection{Main Results}\label{sec:exp:main}

Our FSM achieves 15--3{,}036$\times$ compression over RPNI at $\geq$0.997 test fitness on all datasets, with 7--43 states across the twelve datasets (Table~\ref{tab:main}, Appendix~\ref{app:unlabeled}).
Among positive-only methods, Alergia is the strongest competitor: it matches fitness but uses 1.0--6.0$\times$ more states. k-Tails~\citep{biermann1972ktails}, the classic software-engineering baseline, needs a hyperparameter and produces at least as many states at equal or lower fitness, with counts exploding at $k{\geq}2$.
HMM matches state count but is non-interpretable; EDSM collapses to 1 state without negatives (Appendix~\ref{app:baselines}).
Compression scales with dataset complexity: 15$\times$ on WebArena (short web traces) to 2{,}500$\times$ on AgentNet (long desktop traces that exhaust RPNI's 120s budget).

Fitness converges well inside the training set (Figure~\ref{fig:convergence}).
On SWE-agent, fitness saturates within the first few hundred traces, though the state space continues growing to 25 as rare command patterns appear (Figure~\ref{fig:evolution}).
Across five random train/test splits our method produces identical state counts on every dataset (zero variance), while RPNI state counts vary by 2--10\% (Appendix~\ref{app:stability_detail}).

Our FSM rejects 100\% of random traces and $\geq$99.9\% of permuted traces on all eight labeled real-trace datasets, while RPNI accepts 75\% of permuted traces on WebArena (Table~\ref{tab:precision}, Appendix~\ref{app:precision}).
Single-symbol mutations (substitution, insertion, adjacent swap) are rejected at 77--100\% across datasets; deletion is the weakest mutation (61--98\% rejection, lowest on tau2-bench telecom; Table~\ref{tab:adversarial}).
Process mining baselines achieve precision 0.00--0.80 (Appendix~\ref{app:pm4py}).

\paragraph{Next-step prediction.}
\begin{wraptable}{r}{0.5\textwidth}
\vspace{-12pt}
\centering
\caption{\textbf{Next-step prediction} cross-entropy (bits, $\downarrow$; 5$\times$5-fold CV; best per dataset in \textbf{bold}). ``FSM'' rows condition on the current state (last-activity class; ASG-minimal context, Table~\ref{tab:ctx_ablation}). \emph{K7}=$n$-gram order 7; \emph{H64}=ESN 64 units; \emph{D0/3/5/7}=PPM depths; \emph{AD}=absolute-discounting.}
\label{tab:prediction_full}
\scriptsize
\setlength{\tabcolsep}{2pt}
\renewcommand{\arraystretch}{0.88}
\begin{tabular}{@{}l ccccc c@{}}
\toprule
Method & SWE-sm & SWE-ag & W\&W & M2W & ATB & Avg \\
\midrule
Uniform & 3.170 & 4.585 & 3.000 & 2.807 & 3.807 & 3.474 \\
Unigram & 2.461 & 2.850 & 2.229 & 1.856 & 2.799 & 2.439 \\
RPNI & 3.851 & 4.284 & 2.855 & 3.549 & 2.448 & 3.397 \\
\midrule
Our FSM & 0.638 & 1.071 & 0.963 & 0.756 & 1.243 & 0.934 \\
FSM-PPM-AD-D5 & 0.463 & 0.741 & 0.624 & 0.782 & 1.309 & 0.784 \\
Ens(D0/3/5/7) & 0.465 & 0.700 & 0.621 & \hlbest{0.736} & 1.262 & 0.757 \\
\midrule
NGram-LR-K7 & 0.464 & 0.687 & 0.610 & 0.796 & 1.181 & 0.748 \\
ESN-H64 & 0.476 & 0.691 & 0.580 & 0.746 & 1.184 & 0.735 \\
\midrule
FSM-LR-K7 & \hlbest{0.460} & \hlbest{0.686} & 0.582 & 0.755 & \hlbest{1.162} & \hlbest{0.729} \\
FSM-ESN-H64 & 0.474 & 0.700 & \hlbest{0.546} & 0.753 & 1.185 & 0.732 \\
\bottomrule
\end{tabular}
\vspace{-8pt}
\end{wraptable}
We next ask how much of the activity sequence the FSM predicts. At each step $t$ a predictor estimates $P(a_t \mid \text{context})$; we report cross-entropy (CE, bits) via 5$\times$5-fold CV. Without any learning, our FSM (order-1 Markov) achieves 0.93 bits avg CE across the five-dataset table, a 62\% reduction from the Unigram baseline (2.44 bits; Table~\ref{tab:prediction_full}). This single step of conditioning on FSM state rather than activity frequencies accounts for 83--99\% of the total improvement in CE from Uniform to the best method on each dataset. On these traces the FSM's next-step conditional is equivalent to a start-aware bigram under the same Lidstone estimator, agreeing to within 0.0009 bits on all twelve datasets.

In a controlled ablation (absolute discounting, depth 5), FSM state conditioning provides +0.155 bits mean / +0.136 bits median (21\%) over raw context alone (FSM-AD: 0.580 vs.\ Pure-AD: 0.735 CE), positive on all 6 datasets (range $+0.016$ on SWE-agent to $+0.364$ on Mind2Web; Mind2Web's branching web-action vocabulary benefits most from state conditioning).
Adding FSM state to logistic regression yields only +0.019 (FSM-LR-K7: 0.729 vs.\ NGram-LR-K7: 0.748 in Table~\ref{tab:prediction_full}), because LR partially recovers FSM-like state from raw context.
Combining FSM state with learned models yields an average CE of 0.73 bits (FSM-LR-K7), the best across all methods (Table~\ref{tab:prediction_full}).
RPNI overfits catastrophically: 3.40 bits avg, worse than Unigram (2.44), because its 382--59{,}510 states observe too few transitions each (Figure~\ref{fig:nextstep}; Appendix~\ref{app:prediction}).

\begin{wraptable}{r}{0.55\textwidth}
\vspace{-12pt}
\centering
\caption{\textbf{FSM context format ablation} (tau2-bench retail, N=1{,}095, gpt-4.1-mini top-1 \%). ASG: our FSM context.}
\label{tab:ctx_ablation}
\scriptsize
\setlength{\tabcolsep}{3pt}
\begin{tabular}{@{}l r@{}}
\toprule
Context format & Top-1 (\%) \\
\midrule
No memory (trace prefix only) & 27.6 \\
AWM (linear success workflows)~\citep{wang2024agent_workflow_memory} & 52.9 \\
ASG-full (state + transitions + full graph) & 52.2 \\
ASG++ (ASG-full + multi-step continuations) & 49.2 \\
ASG-success (success-only ASG-full) & 50.3 \\
\cellcolor{cOursBg}\textbf{ASG (minimal: probabilities + top-15 continuations)} & \cellcolor{cOursBg}\hlbest{65.1} \\
\bottomrule
\end{tabular}
\vspace{-8pt}
\end{wraptable}

\paragraph{Context format ablation (why minimal wins).}
The gain over AWM is not automatic: four natural formats of FSM context produce widely different top-1 accuracy on tau2-bench retail (N=1,095, Table~\ref{tab:ctx_ablation}). The verbose ``state + transitions + full structure'' format (ASG-full, 52.2\%) underperforms AWM (52.9\%) because listing every state and transition drowns the next-step signal; adding multi-step continuations only (ASG++, 49.2\%) is worse, as does restricting to success-only traces (ASG-success, 50.3\%). The minimal format used in Table~\ref{tab:fsm_context_llm} (natural-language next-step probabilities plus top-15 multi-step continuations from the current state, with no ``current state / full structure'' headers) wins at 65.1\% (+12.2pp over AWM and +12.9pp over ASG-full). AWM here is its published default format from \citet{wang2024agent_workflow_memory}; identifying the right minimal context for a structural model is part of the contribution, in the same way that AWM's linear-workflow format is part of its. Note that tau2-bench retail is used for format selection and also appears in Table~\ref{tab:fsm_context_llm}; the format generalizes (FSM wins on the other 7 datasets it was not selected on, including the largest WebArena $N{=}4{,}800$ split), but the tau2-bench retail row should be read with this overlap in mind.

\paragraph{Judge robustness.} This tests whether the choice of predictor model changes the FSM-versus-AWM comparison; both are scored by top-1 accuracy against the next action recorded as ground truth, not by a subjective rating. Averaging \texttt{gpt-4.1-mini} and \texttt{gpt-4o-mini} on the most contested datasets (ATBench, tau2-bench airline) keeps FSM ahead of AWM by mean \textbf{8.7pp} (range +3.6 to +13.9), with FSM winning under both judges (example prompt comparison: Appendix~\ref{app:fsm_context}, Figure~\ref{fig:ctx_sidebyside}).

\paragraph{FSM as context for LLM agents.}
\begin{wraptable}{r}{0.52\textwidth}
\vspace{-12pt}
\centering
\caption{\textbf{FSM vs.\ AWM as context for an LLM next-step predictor} (gpt-4.1-mini, top-1 vs.\ recorded next activity). $N$: validation size. FSM wins on all 8 datasets (6/8 $z$-test significant; FSM never loses on the remaining 2). $^{\star}$tau2-bench (ret) was used to select the ASG-minimal context format (Section~\ref{sec:exp:main}); the format generalizes to the other 7 unseen datasets.}
\label{tab:fsm_context_llm}
\footnotesize
\setlength{\tabcolsep}{3pt}
\begin{tabular}{@{}l r rr r@{}}
\toprule
Dataset & $N$ & AWM & \cellcolor{cOursBg}FSM & $\Delta$ \\
\midrule
WebArena         & 4{,}800 & 65.5 & \cellcolor{cOursBg}\hlbest{81.2} & \textbf{+15.7} \\
SWE-smith        &   300   & 74.7 & \cellcolor{cOursBg}\hlbest{100.0} & \textbf{+25.3} \\
SWE-agent        & 1{,}200 & 67.7 & \cellcolor{cOursBg}\hlbest{70.5} & +2.8 \\
tau2-bench (tel) & 1{,}095 & 28.5 & \cellcolor{cOursBg}\hlbest{45.6} & \textbf{+17.1} \\
tau2-bench (ret)$^{\star}$ & 1{,}095 & 52.9 & \cellcolor{cOursBg}\hlbest{65.1} & \textbf{+12.2} \\
tau2-bench (air) &   480   & 56.5 & \cellcolor{cOursBg}\hlbest{57.3} & +0.8 \\
ATBench          &   600   & 47.8 & \cellcolor{cOursBg}\hlbest{62.5} & \textbf{+14.7} \\
OSWorld          & 1{,}286 & 55.0 & \cellcolor{cOursBg}\hlbest{70.7} & \textbf{+15.7} \\
\bottomrule
\end{tabular}
\vspace{-8pt}
\end{wraptable}
\spaceskip=\fontdimen2\font plus \fontdimen3\font minus 0.05\fontdimen4\font\relax
We test whether providing the FSM as context improves an LLM's next-step prediction, comparing against Agent Workflow Memory (AWM)~\citep{wang2024agent_workflow_memory}. A next-step predictor (gpt-4.1-mini) is prompted with AWM's linear workflows or the FSM's transition probabilities plus top-15 continuations; top-1 compares against the \emph{recorded} next activity (ground truth, not LLM-as-judge). The FSM beats AWM on all eight datasets (Table~\ref{tab:fsm_context_llm}; +0.8 to +25.3pp). A non-LLM statistical evaluation on full validation sets confirms the gap (Table~\ref{tab:fsm_context_stat}), with the largest margins on SWE-smith (100\% vs.\ 34.5\%) and tau2-telecom (61.8\% vs.\ 19.8\%); AWM's success-only extraction explains the low-success datasets.

\spaceskip=0pt\relax

\paragraph{Out-of-distribution detection.}
Cross-dataset replay produces fitness 0.00--0.51 (AUROC 1.000 on all pairs); within-alphabet perturbation yields AUROC $\geq$0.917 (Appendix~\ref{app:ood}). FSMs also transfer across models: a single FSM built from all four LLMs' traces achieves \textbf{1.000} fitness on each model individually, and this holds as the alphabet is refined rather than only at the coarse level (per-model FSMs share a near-identical state vocabulary and a dominant 80--92\% transition backbone, indicating largely model-invariant topology), and per-model failure-prediction features transfer at \textbf{0.786 mean cross-AUROC} vs.\ \textbf{0.877 self} across all three tau2-bench suites (36 off-diagonal pairs; per-suite breakdown in Appendix~\ref{app:crossmodel}).

\paragraph{Runtime.}
Our method constructs FSMs in 1--110\,ms across all datasets, compared to 7{,}000--36{,}000\,ms for RPNI (328--10{,}611$\times$ speedup). Per-trace replay completes in 0.003--0.015\,ms, enabling real-time monitoring of production agent systems (Appendix~\ref{app:runtime}).

\subsection{Failure Prediction from FSM Features}\label{sec:exp:downstream}

We predict task success/failure on eight real-trace and one synthetic dataset (SWE-smith), including ATBench~\citep{ai45research2026atbench}, a trajectory-level safety benchmark with balanced safe/unsafe outcomes.
Each trace is replayed through the FSM and we extract per-state features (visit frequency, message length, error rate, temporal) alongside five FSM cross-entropy anomaly features (trace CE, max surprise, half-to-half drift, minimum transition probability, high-surprise rate). A classifier with gradient boosting~\citep{friedman2001greedy} produces held-out AUROC on a fixed 80/20 split (hyperparameters and per-dataset numbers in Appendix~\ref{app:failure_prediction}, Figure~\ref{fig:failure}a, Table~\ref{tab:failure}).

Raw fitness is uninformative (AUROC $\approx$ 0.50); FSM cross-entropy anomaly features reach up to 0.941 held-out AUROC (tau2-bench telecom, 43 states), and the largest FSM predicts best, but size does not order the rest (telecom 0.941 at 43 states, against WebArena 0.903, AgentNet 0.890 and SWE-agent 0.799 at 25 states each and ATBench, with 15 states, at 0.894; Figure~\ref{fig:failure}a).
On ATBench, the only safety-labeled benchmark, AUROC reaches 0.894 with low variance across CV folds (0.864 ± 0.024).
Across all eight real-trace datasets the 5-fold $\times$ 10-repeat CV std stays in 0.012--0.031 (Appendix~\ref{app:failure_prediction}, Table~\ref{tab:failure}), so the held-out AUROCs are not single-split artifacts.

\begin{figure}[t]
\centering
\includegraphics[width=0.78\linewidth]{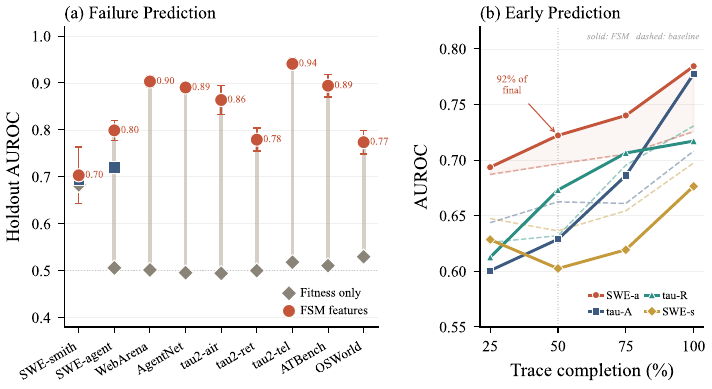}
\caption{\textbf{Failure prediction.} (a)~AUROC: FSM per-state features (red, holdout 0.799) vs.\ raw trace statistics (blue, length+entropy + bag-of-activities, holdout 0.720) vs.\ fitness alone (gray). FSM outperforms raw features on SWE-agent (+7.9pp); within the same in-domain LR pipeline (Table~\ref{tab:sequence}), the FSM-per-state vs all-sequence-features gap is +2.0pp (0.813 vs 0.793). (b)~Early prediction: FSM features at 50\% completion achieve 92\% of the full-trace AUROC on SWE-agent (CV 0.722/0.784; holdout 0.726/0.773). Solid: FSM; dashed: baseline.}
\label{fig:failure}
\end{figure}

FSM features at 50\% completion reach 92\% of full-trace AUROC (Figure~\ref{fig:failure}b); success and failure FSMs differ structurally (Jaccard 0.206 on SWE-agent), and the signal is not a length proxy (structure: 0.790 vs.\ length: 0.659; Appendices~\ref{app:early},~\ref{app:structural_divergence},~\ref{app:length_vs_struct}).
Applying the identical feature pipeline to FSMs extracted by Alergia yields lower AUROC on 8 of 9 datasets (Appendix~\ref{app:alergia_downstream}), confirming the gain comes from per-state observation density rather than feature engineering.
A label-aware variant (\emph{discriminative quotient}, FSM-D; Appendix~\ref{app:fsm_d}) replaces the standard last-activity merge with one that conditions on success/failure outgoing distributions, lifting AUROC over a length+entropy baseline by up to $+$0.152 on tau2-bench airline using only training-free per-state KL features.

\paragraph{Agent integration: FSM as runtime monitor.}
A two-rule monitor (cycle-rate $>$0.778 \& minimum unique-state count) reaches F1\,=\,0.904 on SWE-agent at 32\% trace completion (precision 85.9\%, recall 95.5\%; tau2-bench airline F1\,=\,0.776 at 56\%). SWE-agent's 84.3\% failure rate dominates F1 (flag-everything F1\,=\,0.914); under rank-based AUROC the FSM features reach 0.66 at the nearest 25\% checkpoint to the monitor's 32\% mean firing time (4/4 datasets above 0.5; Appendices~\ref{app:integration},~\ref{app:early}). FSM replay only (0.006\,ms/trace), no ML model.

\section{Discussion}\label{sec:discussion}
\textbf{Baseline landscape.}
Gold's theorem~\citep{gold1967language} forces every positive-only method onto a compression--fitness tradeoff: EDSM and GSM+AIC over-merge to universal acceptors, RPNI and k-Tails ($k{\geq}2$) under-merge to $10^1$--$10^5$ states, and Alergia matches our fitness with 1.0--6.0$\times$ more states via stochastic merges.
Our deterministic structural merge gives a zero-variance topology the downstream pipelines depend on.
\emph{Compactness makes per-state estimation reliable}. A small state set pools enough observations per state to make positive-only learning well-conditioned, and the same FSM serves workflow memory, next-step prediction, failure prediction, and runtime monitoring without four bespoke pipelines.

\textbf{FSM state is structural.}
The state $q_t$ is a deterministic summary of the prefix; two traces with identical activity counts but different orderings land in distinct states.
Length alone yields AUROC 0.659 on SWE-agent, while structural features reach 0.790 (Appendix~\ref{app:length_vs_struct}); AWM's success-only linear workflows collapse on low-success datasets (74.7\% on SWE-smith, 28.5\% on tau2-telecom) where our FSM reaches 100\% / 45.6\% (Table~\ref{tab:fsm_context_llm}).
This per-state decomposition locates \emph{where} a trace deviates and drives both the early-stopping monitor at 32\% completion and cross-model transfer (\textbf{0.786} mean cross-AUROC vs.\ 0.877 self; Appendix~\ref{app:crossmodel}).
The same per-state features improve MLP, GRU, and Transformer failure-prediction baselines on 20 of 21 dataset-architecture pairs over matched sequence features (Appendix~\ref{app:neural_probe}), confirming that the FSM acts as a structural feature primitive that benefits even high-capacity learners rather than competing with them.

\textbf{When and how invariant?}
We find the topology invariant to model choice across all three tau2-bench suites, and robust to extraction granularity: our default granularity matches or exceeds role-only held-out AUROC on 10 of 12 datasets, and differs from it by 0.05 or more only on datasets whose role-only alphabet degenerates to at most three symbols (Appendix~\ref{app:granularity}). This is consistent with system-level rather than model-level structure.
For agents with much larger action spaces or weaker conditional structure the same construction would still apply but no longer be compact, and the per-state observation density that drives our downstream gains would degrade accordingly.

\textbf{Stream-friendly deployment.}
Per-state estimates update incrementally without re-fit, and replay takes 0.006\,ms per trace, so the same FSM serves both offline analysis and online monitoring under one primitive---a property the per-trace and re-trained pipelines compared in \S\ref{sec:related} do not share.

The construction itself is classical---the directly-follows graph of process mining~\citep{vanderaalst2016process, hopcroft2006automata}, which we make deterministic and give a convergence guarantee. What is new is that the bounded alphabets of LLM agents make the resulting automaton both small and statistically informative (Proposition~\ref{prop:estimator}, Corollary~\ref{cor:surprise}), and the same FSM unifies workflow memory, next-step prediction, failure prediction, and runtime monitoring.

\section{Limitations}\label{sec:limitations}
\textbf{Observed directly-follows automaton, not generator identification.}
Theorem~\ref{prop:minimal} characterizes the directly-follows automaton of the observed traces, not the agent's generating language; adversarial traces preserving activity bigram statistics can replay (Appendix~\ref{app:precision}).
\textbf{Activity-extraction dependence.}
$\phi$ is specific to each dataset and requires per-domain knowledge; granularity ablation shows AUROC robust to $\phi$ choice, but fully automatic discovery of $\phi$ is future work. We therefore scope our contribution to structured tool-call interfaces.
\textbf{Cross-model transfer scope.}
The 0.786 mean cross-AUROC is computed on the three tau2-bench suites with multi-model traces available; transfer across architectures and domains beyond tau2-bench is future work.
\textbf{Order-1 conditioning.}
Corollary~\ref{cor:surprise} assumes each activity depends only on the current FSM state. Anomaly features built on an order-2 context change mean held-out AUROC by $-$0.021.
\textbf{Sequence-model and workflow-memory coverage.}
A single-split LSTM/Transformer matches FSM-LR-K7 within $\pm$0.05 bits on traces with small alphabets (Appendix~\ref{app:lstm_transformer}). The LLM-context comparison is against AWM~\citep{wang2024agent_workflow_memory} only, and extending it to concurrent memory methods that use both successes and failures such as ReasoningBank~\citep{ouyang2025reasoningbank} is future work.

\section{Conclusion}\label{sec:conclusion}
Using only positive examples, we extract compact directly-follows finite-state machines from LLM agent traces.
The resulting compact FSMs support workflow memory that beats AWM on all eight datasets, next-step prediction, failure prediction, and a runtime monitor for early stopping.
A single construction with no hyperparameters supports all four and runs in milliseconds, replacing four bespoke learned pipelines with one structural primitive.
Despite their apparent complexity, LLM agents admit compact structural abstractions, a substrate that can be deployed for safety auditing, runtime monitoring, and behavioral analysis.

\begin{ack}
This work received no external or grant funding. Computing resources for this work were provided by Holistic AI. The authors declare no competing interests.
\end{ack}

\bibliographystyle{plainnat}
\bibliography{references}

\ifarxiv\else
\newpage
\input{checklist}
\fi

\newpage
\appendix
\renewcommand{\thesection}{A.\arabic{section}}
\setcounter{section}{0}
\section{Theory and Proofs}\label{app:group:theory}

\subsection{Extraction Algorithm}\label{app:algorithm}

\begin{algorithm}[H]
\caption{FSM Extraction from Agent Traces}\label{alg:fsm}
\KwIn{Traces $\cD = \{\tau_1, \ldots, \tau_N\}$, extraction function $\phi$}
\KwOut{FSM $\cM = (Q, \cA, \delta, q_0)$}
\ForEach{$\tau \in \cD$}{
  $\sigma \gets (\phi(m_1), \ldots, \phi(m_T))$\;
  Insert $\sigma$ into the prefix tree, extending its states and transitions\;
}
\ForEach{trie state $q$}{
  assign $q$ to class $\kappa(q)$, its incoming activity (root $\mapsto \mathrm{init}$)\;
}
\ForEach{trie edge $q \xrightarrow{a} q'$}{
  add transition $\kappa(q) \xrightarrow{a} \kappa(q')$, aggregating counts\;
}
\ForEach{transition $t$ with aggregated count $1$}{
  remove $t$ unless it is the only transition leaving its source\;
}
\Return{$(Q, \cA, \delta, q_0)$}
\end{algorithm}

\begin{figure}[H]
\centering
\begin{tikzpicture}[
    ->,>=Stealth,
    node distance=1.6cm and 2.2cm,
    state/.style={circle, draw, minimum size=0.7cm, font=\tiny\sffamily, inner sep=1pt},
    every edge/.style={draw, font=\tiny},
    scale=0.9, transform shape
  ]
  \node[state, fill=gray!15] (init) {init};
  \node[state, fill=blue!12, right=of init] (sys) {sys};
  \node[state, fill=green!12, right=of sys] (usr) {usr};
  \node[state, fill=orange!12, above right=0.8cm and 2cm of usr] (atc) {a:tc};
  \node[state, fill=purple!12, right=of atc] (tool) {tool};
  \node[state, fill=red!12, below right=0.8cm and 2cm of usr] (atxt) {a:txt};

  \path (init) edge node[above] {sys} (sys);
  \path (sys)  edge node[above] {usr} (usr);
  \path (usr)  edge node[above left, pos=0.4] {a:tc} (atc);
  \path (usr)  edge node[below left, pos=0.4] {a:txt} (atxt);
  \path (atc)  edge[bend left=15] node[above] {tool} (tool);
  \path (tool) edge[bend left=15] node[below] {a:tc} (atc);
  \path (tool) edge node[right, pos=0.4] {a:txt} (atxt);
  \path (atxt) edge[bend right=40] node[below] {usr} (usr);
\end{tikzpicture}
\caption{\textbf{Extracted FSM} for a customer service agent at role-level granularity (tau2-bench airline, 6 states, $|\cA|{=}5$). The tool-level FSM (18 states, $|\cA|{=}17$, Table~\ref{tab:main}) further decomposes \texttt{a:tc}/\texttt{tool} into per-tool states. The tool-call loop (a:tc$\leftrightarrow$tool) captures repeated API invocations. The conversational loop (a:txt$\to$usr) captures dialogue turns.}
\label{fig:example_fsm}
\end{figure}

\subsection{Proofs}\label{app:proofs}

\begin{proof}[Proof of Theorem~\ref{prop:fitness}] Let $\sigma = (a_1, \ldots, a_T)$ be accepted by prefix tree $\cP$, visiting states $q_0, q_1, \ldots, q_T$ with $\delta(q_i, a_{i+1}) = q_{i+1}$.
For each of these trie edges, lines~6--7 of Algorithm~\ref{alg:fsm} add the transition $\kappa(q_i) \xrightarrow{a_{i+1}} \kappa(q_{i+1})$ to $\cM$, and $\kappa(q_{i+1}) = a_{i+1}$.
Hence $\cM$ has the path \[ \mathrm{init} = \kappa(q_0) \xrightarrow{a_1} \kappa(q_1) \xrightarrow{a_2} \cdots \xrightarrow{a_T} \kappa(q_T), \] and $\sigma$ is accepted by $\cM$.
Only lines~8--9 can remove one of these transitions, when its aggregated count is one and it is not the only transition leaving its source; this is the only way a trace can fail to replay, and we measure exactly this cost in the replay-fitness columns of Table~\ref{tab:main}.
\end{proof}

\begin{proof}[Proof of Theorem~\ref{prop:minimal}]
Line~7 of Algorithm~\ref{alg:fsm} adds a transition $\kappa(q) \xrightarrow{a} \kappa(q')$ only for a trie edge $q \xrightarrow{a} q'$, and then $\kappa(q') = a$.
Thus line~7 never gives a class two targets for the same activity, and $\delta$ is a partial deterministic transition function.
Since lines~6--7 add a transition for every trie edge, the automaton accepts every training trace before lines~8--9 run, as in Theorem~\ref{prop:fitness}.
Line~5 assigns every non-root trie state to the class of its incoming activity and the root to $\mathrm{init}$, so there are exactly $|\cA|+1$ classes, and since lines~8--9 remove only transitions, $|Q| = |\cA|+1$.
The output of Algorithm~\ref{alg:fsm} depends on the corpus only through the counts of observed $(\text{activity}, \text{next-activity})$ pairs.
These counts are invariant to the order in which traces are read and to resampling; hence the output automaton is a deterministic function of the corpus and re-extraction from the same data is exact.
This is the directly-follows automaton of the observed traces, not the minimal DFA of the finite prefix language $L_\cP$ (which would be acyclic). Recurring activities induce the loops of Figure~\ref{fig:example_fsm}.
\end{proof}

\begin{proof}[Proof of Proposition~\ref{prop:convergence}]
Let $e_1, \ldots, e_r$ be the $r$ transitions of $\cM^*$, with $p_j = \Pr[e_j \text{ appears in a random trace}] \geq p_{\min}$.  After $N$ i.i.d.\ traces:
\begin{align}
    \Pr[e_j \text{ not observed}] &= (1-p_j)^N \;\leq\; e^{-Np_{\min}}, \\
    \Pr[\exists\, j:\; e_j \text{ not observed}] &\leq r \cdot e^{-Np_{\min}}. \quad\text{(union bound)}
\end{align}
Setting the right-hand side $\leq \delta_\mathrm{f}$ and solving:
\[
    N \;\geq\; \frac{1}{p_{\min}} \ln\!\bigl(r / \delta_\mathrm{f}\bigr).
\]
When we observe all transitions, the prefix tree contains every transition of $\cM^*$, and by Theorem~\ref{prop:minimal} the last-activity merge yields $\cM^*$.
\end{proof}

\begin{proposition}[Transition estimator consistency and concentration]\label{prop:estimator}
Let $P^*(\cdot \mid q)$ denote the true transition distribution at state $q$ under an i.i.d.\ trace distribution, let $n_q$ be the visit count of state $q$ in the FSM-conditioned plug-in estimator $\hat{P}(\cdot\mid q)$ (Eq.~\ref{eq:fsm_predict}), and let $\alpha\in(0,1]$ be the Lidstone smoothing parameter. For any $\epsilon > 0$,
\begin{equation}\label{eq:concentration}
\Pr\left[\, \left\| \hat{P}(\cdot \mid q) - P^*(\cdot \mid q) \right\|_{\mathrm{TV}} \,\geq\, \epsilon + \tfrac{\alpha |\cA|}{n_q + \alpha |\cA|}\, \right] \;\leq\; 2|\cA| \exp\!\left(- \tfrac{8 n_q \epsilon^2}{|\cA|^2} \right).
\end{equation}
The smoothing bias vanishes as $n_q\to\infty$, and $\|\hat{P}-P^*\|_{\mathrm{TV}}\to 0$ almost surely. (Proof: Hoeffding's inequality for each activity at $2\epsilon/|\cA|$, a union bound over $\cA$, the smoothing bias added in TV, and Borel--Cantelli for the almost-sure limit.)
\end{proposition}

\spaceskip=\fontdimen2\font plus \fontdimen3\font minus 0.05\fontdimen4
\begin{corollary}[Surprise as Neyman--Pearson statistic]\label{cor:surprise}
Suppose success and failure traces are generated by FSM-structured mixtures $P^{+}(\cdot\mid q),P^{-}(\cdot\mid q)$ on the same state space, and the trace cross-entropies are $\mathrm{CE}^{\pm}(\tau)=\tfrac{1}{T}\sum_t -\log_2 \hat{P}^{\pm}(a_t\mid q_t)$. Given FSM state $q_t$, $a_t$ is conditionally independent of earlier steps, so the trace likelihood factors as $\prod_t P(a_t\mid q_t)$. The surprise difference $\mathrm{CE}^{-}(\tau)-\mathrm{CE}^{+}(\tau)$ equals the per-trace log-likelihood ratio $\sum_t \log[\hat{P}^{-}(a_t\mid q_t)/\hat{P}^{+}(a_t\mid q_t)]$ up to a factor of $-1/(T\ln 2)$. It is the Neyman--Pearson-optimal statistic for the population test $H_0\!:\tau\sim P^{+}$ vs.\ $H_1\!:\tau\sim P^{-}$. By Proposition~\ref{prop:estimator}, the plug-in version is asymptotically optimal with $O(1/\sqrt{n_q})$ power-loss error.
\end{corollary}
\spaceskip=0pt\relax

\subsection{Complexity Analysis}\label{app:complexity}

\spaceskip=\fontdimen2\font plus \fontdimen3\font minus 0.05\fontdimen4
\medmuskip=4mu plus 2mu minus 0.2mu\relax
Construction has two phases.
Building the prefix tree (Phase 1) takes $O(N\bar{T})$ time and space for $N$ traces of mean length $\bar{T}$, since each symbol extends a trie node via hash-map lookup.
In the last-activity merge (Phase 2) we then group the trie nodes by incoming activity in a single pass and aggregate trie-edge counts into a $(\text{class}, \text{symbol}, \text{class})$ transition multiset.
This gives $|\cA|+1$ classes.
In this pass we visit every node once and inspect each outgoing edge, which takes $O(|Q_\cP| \cdot |\cA|)$ time in the worst case, where $|Q_\cP|$ is the number of states in the prefix tree.
Since $|Q_\cP| \leq N\bar{T}$, construction time scales linearly with corpus size, as $O(N\bar{T} \cdot |\cA|)$.
In practice, $|\cA| \leq 42$ across all twelve datasets and the hash-map constant is small. FSM construction takes 1--110\,ms on a single CPU core, against 7,000--36,000\,ms for RPNI (Table~\ref{tab:runtime}).

\spaceskip=0pt\relax
\medmuskip=4mu plus 2mu minus 4mu\relax

\section{Datasets and Setup}\label{app:group:datasets}

\subsection{Dataset Details}\label{app:datasets}

\paragraph{Who\_and\_When}~\citep{yang2025whoandwhen} contains 184 multi-agent failure traces with 8 activity types (Table~\ref{tab:datasets}). All traces represent failures in agent delegation tasks. Activity extraction uses the actor role and action type fields.

\paragraph{SWE-smith}~\citep{yang2025swesmith} provides 500 coding agent traces (377 success, 123 failure) from SWE-bench task instances. Activities are extracted from \texttt{tool\_calls[].function.name} fields, yielding 9 unique activities (\texttt{bash}, \texttt{str\_replace\_editor}, \texttt{submit}, etc.).

\paragraph{Mind2Web}~\citep{deng2024mind2web} contains 500 web navigation traces across 137 websites. Activities are extracted from action representation strings in the format \texttt{[element] description $\to$ ACTION: value}, yielding 7 activity types (\texttt{CLICK}, \texttt{TYPE}, \texttt{SELECT}, etc.).

\spaceskip=\fontdimen2\font plus \fontdimen3\font minus 0.05\fontdimen4
\begin{wraptable}{r}{0.46\textwidth}
\caption{\textbf{Evaluation datasets.} $|\cA|$: alphabet. $^*$7 primary, 17 rare. $^\ddagger$4 LLMs. $^a$Trivially separable.}
\label{tab:datasets}
\footnotesize
\setlength{\tabcolsep}{3pt}
\begin{tabular}{@{}llrrl@{}}
\toprule
Dataset & Domain & Traces & $|\cA|$ & Labels \\
\midrule
\multicolumn{5}{@{}l}{\textit{Labeled (main results)}} \\
SWE-smith & Coding & 500 & 9 & \cmark \\
SWE-agent & Coding & 2{,}000 & 24$^*$ & \cmark \\
WebArena & Web nav. & 8{,}337 & 24 & \cmark \\
AgentNet & Desktop GUI & 5{,}000 & 24 & \cmark \\
tau2-bench (air) & Cust.\ svc. & 800$^\ddagger$ & 17 & \cmark \\
tau2-bench (ret) & Cust.\ svc. & 1{,}824$^\ddagger$ & 18 & \cmark \\
tau2-bench (tel) & Telecom & 1{,}824$^\ddagger$ & 42 & \cmark \\
ATBench          & Safety  & 1{,}000 & 14 & \cmark \\
OSWorld          & Desktop OS & 2{,}166 & 26 & \cmark \\
\midrule
\multicolumn{5}{@{}l}{\textit{Unlabeled (Appendix~\ref{app:unlabeled})}} \\
Who\_and\_When & Multi-agent & 184 & 8 & \xmark \\
Mind2Web & Web nav. & 500 & 7 & \xmark \\
GUI-Odyssey & Mobile GUI & 7{,}735 & 6 & \cmark$^a$ \\
\bottomrule
\end{tabular}
\vspace{-8pt}
\end{wraptable}

\paragraph{tau2-bench}~\citep{barres2025tau2} extends tau-bench with multi-model evaluation across three domains: airline (800 traces, 17 activities), retail (1{,}824 traces, 18 activities), and telecom (1{,}824 traces, 42 activities). Each domain contains traces from 4 LLMs (GPT-4.1, Claude~3.7 Sonnet, GPT-4.1-mini, o4-mini). Activities are extracted from \texttt{tool\_calls[].name}. The telecom domain introduces a significantly richer tool vocabulary (42 activities including network diagnostics, SIM operations, billing) than any other dataset, producing our largest FSM (43 states).

\spaceskip=0pt\relax

\paragraph{WebArena}~\citep{zhou2024webarena} provides 8{,}337 web agent traces across realistic web tasks (shopping, forums, maps, GitLab). Activities are 12 normalized action types (\texttt{click}, \texttt{type}, \texttt{scroll\_down}, etc.) combined with role prefixes, which yields $|\cA|{=}24$ activity symbols. Labels derive from task completion status (13.4\% success). This is our largest labeled dataset by trace count and produces the lowest compression ratio (15$\times$) because short web interaction traces (median 5 steps) give RPNI limited opportunity to overfit.

\paragraph{AgentNet}~\citep{wang2025opencua} provides 5{,}000 desktop computer-use agent trajectories from the OpenCUA Ubuntu subset, covering GUI automation across diverse applications. Activities are extracted from \texttt{pyautogui} action primitives (\texttt{click}, \texttt{typewrite}, \texttt{hotkey}, \texttt{screenshot}, \texttt{moveTo}, etc.), which yields 24 activities. Labels derive from task completion annotations (36.4\% success). This dataset produces the second-highest compression ratio (2{,}500$\times$) due to its large training set (4{,}000 traces) and diverse action vocabulary.

\paragraph{SWE-agent}~\citep{yang2024sweagent} contains 2{,}000 coding agent trajectories in Parquet format. Raw commands are extracted from code blocks in assistant messages and grouped into 7 actor categories: \texttt{search}, \texttt{navigate}, \texttt{edit}, \texttt{execute}, \texttt{submit}, \texttt{user}, and \texttt{assistant}. Combined with message-type suffixes, these yield $|\cA|{=}24$ activity symbols (Table~\ref{tab:datasets}). Labels derive from the \texttt{target} boolean field.

\paragraph{GUI-Odyssey}~\citep{lu2024guiodyssey} contains 7{,}735 cross-app mobile GUI navigation episodes across 201 apps on 6 Android devices. Each step records an action type (\texttt{CLICK}, \texttt{TEXT}, \texttt{SCROLL}, \texttt{LONG\_PRESS}, \texttt{COMPLETE}, \texttt{INCOMPLETE}), which yields 6 activities. Labels derive from episode success: \texttt{COMPLETE} (7{,}486 successes) vs.\ \texttt{INCOMPLETE} (249 failures). This is our largest mobile-GUI dataset and produces the highest compression ratio (3{,}036$\times$).

\subsection{Unlabeled Dataset Results}\label{app:unlabeled}

Three datasets lack success/failure labels suitable for failure prediction and are reported here for compression analysis only.

\begin{wraptable}{r}{0.58\textwidth}
\vspace{-12pt}
\caption{\textbf{Compression results} on unlabeled datasets. Same methodology as Table~\ref{tab:main}.}
\footnotesize
\setlength{\tabcolsep}{3pt}
\begin{tabular}{@{}l rr rr rr r@{}}
\toprule
& \multicolumn{2}{c}{\cellcolor{cOursBg}Ours} & \multicolumn{2}{c}{RPNI} & \multicolumn{2}{c}{Alergia} & \\
\cmidrule(lr){2-3} \cmidrule(lr){4-5} \cmidrule(lr){6-7}
Dataset & \cellcolor{cOursBg}$|Q|$ & \cellcolor{cOursBg}Fit & $|Q|$ & Fit & $|Q|$ & Fit & Compr. \\
\midrule
Who\_and\_When & \cellcolor{cOursBg}\textbf{9} & \cellcolor{cOursBg}1.000 & 971 & 0.984 & 12 & 1.000 & 108$\times$ \\
Mind2Web & \cellcolor{cOursBg}\textbf{8} & \cellcolor{cOursBg}1.000 & 476 & 0.970 & 8 & 1.000 & 60$\times$ \\
GUI-Odyssey & \cellcolor{cOursBg}\textbf{7} & \cellcolor{cOursBg}1.000 & 21{,}255$^\dagger$ & 0.929 & 24 & 0.999 & 3{,}036$\times$ \\
\bottomrule
\end{tabular}
\vspace{-8pt}
\end{wraptable}
Who\_and\_When contains only failure traces (no success examples), making failure prediction undefined.
Mind2Web provides ground-truth demonstrations without outcome labels.
GUI-Odyssey has labels, but failure prediction is trivially separable (AUROC 1.000): the terminal state \texttt{COMPLETE}/\texttt{INCOMPLETE} directly encodes the label.
Including these three datasets, compression ranges from 15$\times$ to 3{,}036$\times$ across all twelve datasets.

\section{Extended Results and Figures}\label{app:group:results}

\subsection{Cross-Model Transfer}\label{app:crossmodel}

The tau2-bench datasets contain traces from four LLMs executing identical tasks, enabling cross-model FSM transferability analysis.
For each source-target model pair, we build an FSM from the source model's training traces and evaluate replay fitness and failure prediction AUROC on the target model's test traces.

\paragraph{Structural transfer.} Building a single FSM from the traces of all four models achieves 1.000 replay fitness on every model individually. The behavioral topology is model-invariant: all models traverse the same tool-call sequences, differing only in transition probabilities.

\paragraph{Failure prediction transfer (3 suites).} Augmenting the feature set with the anomaly features from FSM cross-entropy (Section~\ref{sec:exp:downstream}) produces moderate transfer.
We measure on all three tau2-bench suites with 4 LLMs each (12 off-diagonal pairs per suite, 36 pairs total).

\begin{table}[H]
\caption{\textbf{Cross-model failure-prediction AUROC across three tau2-bench suites.}
Self-AUROC on the diagonal and mean cross-AUROC off the diagonal; $\sigma$ over the 4 self-AUROCs and 12 cross-pair AUROCs.
Fitness within and across models is 1.000 and 0.962 on airline, 1.000 and 0.972 on retail, and 1.000 and 0.990 on telecom.}
\label{tab:crossmodel_3suites}
\footnotesize
\setlength{\tabcolsep}{6pt}
\begin{tabular}{@{}l ccc ccc r@{}}
\toprule
& \multicolumn{3}{c}{Self-AUROC (diagonal)} & \multicolumn{3}{c}{Cross-AUROC (off-diagonal)} & \\
\cmidrule(lr){2-4}\cmidrule(lr){5-7}
Suite & mean & $\sigma$ & range & mean & $\sigma$ & range & gap \\
\midrule
$\tau_2$ airline & 0.926 & 0.051 & 0.84--0.97 & 0.773 & 0.102 & 0.56--0.92 & 0.154 \\
$\tau_2$ retail  & 0.749 & 0.052 & 0.68--0.82 & 0.681 & 0.075 & 0.54--0.79 & 0.068 \\
$\tau_2$ telecom & 0.956 & 0.028 & 0.92--0.99 & 0.905 & 0.045 & 0.82--0.99 & 0.051 \\
\midrule
\textbf{Mean (3 suites)} & \textbf{0.877} & -- & -- & \textbf{0.786} & -- & -- & 0.091 \\
\bottomrule
\end{tabular}
\end{table}

The transfer gap is consistent across suites (range 0.05--0.15), and cross-AUROC remains $\geq$0.68 in mean on every suite, well above chance.
GPT-4.1$\to$o4-mini is the strongest cross-pair (telecom 0.950, airline 0.890, retail 0.792). The weakest pair is o4-mini$\to$Claude~3.7, on airline and retail.
On all three suites, per-model FSMs reach 0.999--1.000 self-fitness (40--41 states each on telecom, matching airline and retail).
In every suite the combined FSM of all models reaches 1.000 fitness on each model individually, so one FSM can monitor heterogeneous deployments without retraining per model.
Transition probabilities under the source model's FSM remain partially informative when the target model's surface behavior differs, because the surprise signal $-\log_2 P(a_t\mid q_t)$ tracks structural anomalies rather than model-specific tokens. The residual gap reflects failure modes that are partly specific to each model.

\section{Baselines and Implementation}\label{app:group:baselines}

We compare against three categories of baselines, each representing a distinct approach to extracting models of behavior.
\textbf{Automata learning} (RPNI, EDSM, Alergia): classic algorithms that infer DFAs or probabilistic automata from traces.
RPNI and EDSM require negative examples for effective merging; without them, RPNI produces near-complete prefix trees ($10^3$--$10^5$ states) and EDSM collapses to a one-state universal acceptor.
Alergia uses statistical compatibility testing but still overestimates state counts by 1--6$\times$.
\textbf{Process mining} (Heuristic, Inductive, Alpha Miner via PM4Py): discover Petri nets from event logs.
These achieve high fitness by accepting all orderings of observed activities, but this permissiveness yields low precision (0.23--0.69).
\textbf{Workflow extraction} (AWM): extracts linear workflows from successful traces via longest common subsequence alignment.
AWM cannot represent cycles or branching and requires success labels, limiting its applicability to labeled datasets.

\subsection{Extended Baseline Results}\label{app:baselines}

Table~\ref{tab:full_baselines} presents the complete baseline comparison including all process mining and workflow extraction methods.

\noindent
\begin{minipage}[t]{0.42\textwidth}
\vspace{0pt}
\centering
\captionof{table}{\textbf{Full baseline comparison} (1/4). Fit: test replay fitness. $^\dagger$: RPNI timeout at 120s.}\label{tab:full_baselines}
\scriptsize
\setlength{\tabcolsep}{3pt}
\begin{tabular}{@{}l l rr l@{}}
\toprule
Dataset & Method & $|Q|$/Size & Fit & Notes \\
\midrule
\multirow{9}{*}{\rotatebox{90}{\scriptsize Who\_and\_When}}
 & \cellcolor{cOursBg}\textbf{Ours} & \textbf{9 st} & \textbf{1.000} & 108$\times$ \\
 & \cellcolor{cRPNIBg}RPNI & 971 st & 0.984 & Overfit \\
 & \cellcolor{cRPNIBg}EDSM & 1 st & 1.000 & Degen. \\
 & \cellcolor{cAlergiaBg}Alergia & 12 st & \textbf{1.000} & 1.3$\times$ \\
 & \cellcolor{cGroupHd}HMM & \textbf{9 st} & \textbf{1.000} & Latent \\
 & \cellcolor{cGroupHd}Heur.\ M & 10p,23t & 0.997 & P:0.31 \\
 & \cellcolor{cGroupHd}Ind.\ M & 16p,24t & 0.996 & P:0.27 \\
 & \cellcolor{cGroupHd}Alpha M & 3p,8t & 0.369 & Fails \\
 & \cellcolor{cGroupHd}AWM-all & 125 wf & 0.914 & LCS \\
\midrule
\multirow{9}{*}{\rotatebox{90}{\scriptsize SWE-smith}}
 & \cellcolor{cOursBg}\textbf{Ours} & \textbf{10 st} & \textbf{1.000} & 1{,}163$\times$ \\
 & \cellcolor{cRPNIBg}RPNI & 11{,}631$^\dagger$ & 0.762 & $^\dagger$ \\
 & \cellcolor{cRPNIBg}EDSM & 1 st & 1.000 & Degen. \\
 & \cellcolor{cAlergiaBg}Alergia & \textbf{10 st} & \textbf{1.000} & Same \\
 & \cellcolor{cGroupHd}HMM & \textbf{10 st} & \textbf{1.000} & Latent \\
 & \cellcolor{cGroupHd}Heur.\ M & 13p,22t & \textbf{1.000} & P:0.36 \\
 & \cellcolor{cGroupHd}Ind.\ M & 25p,33t & \textbf{1.000} & P:0.23 \\
 & \cellcolor{cGroupHd}Alpha M & 7p,9t & 0.155 & Fails \\
 & \cellcolor{cGroupHd}AWM & 1{,}208 wf & \textbf{1.000} & LCS \\
\midrule
\multirow{9}{*}{\rotatebox{90}{\scriptsize Mind2Web}}
 & \cellcolor{cOursBg}\textbf{Ours} & \textbf{8 st} & \textbf{1.000} & 60$\times$ \\
 & \cellcolor{cRPNIBg}RPNI & 476 st & 0.970 & Overfit \\
 & \cellcolor{cRPNIBg}EDSM & 1 st & 1.000 & Degen. \\
 & \cellcolor{cAlergiaBg}Alergia & \textbf{8 st} & \textbf{1.000} & Same \\
 & \cellcolor{cGroupHd}HMM & \textbf{8 st} & \textbf{1.000} & Latent \\
 & \cellcolor{cGroupHd}Heur.\ M & 15p,29t & 0.960 & P:0.45 \\
 & \cellcolor{cGroupHd}Ind.\ M & 23p,30t & \textbf{1.000} & P:0.49 \\
 & \cellcolor{cGroupHd}Alpha M & 7p,9t & 0.632 & Poor \\
 & \cellcolor{cGroupHd}AWM-all & 72 wf & 0.887 & LCS \\
\bottomrule
\end{tabular}
\end{minipage}%
\hfill
\begin{minipage}[t]{0.56\textwidth}
\vspace{0pt}
\centering
\includegraphics[width=\linewidth]{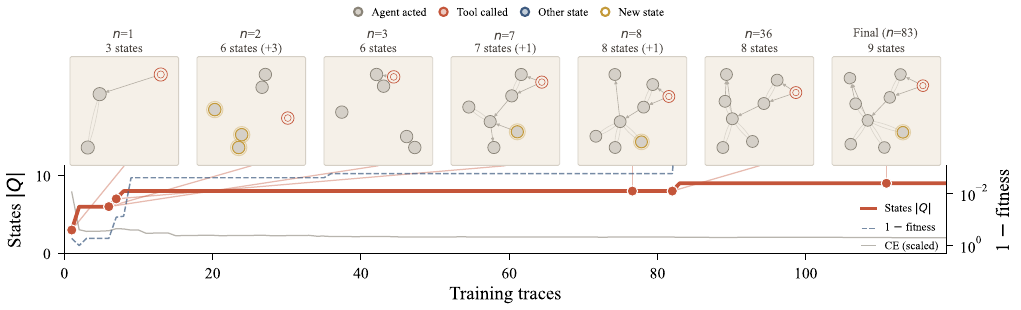}\\[2pt]
\includegraphics[width=\linewidth]{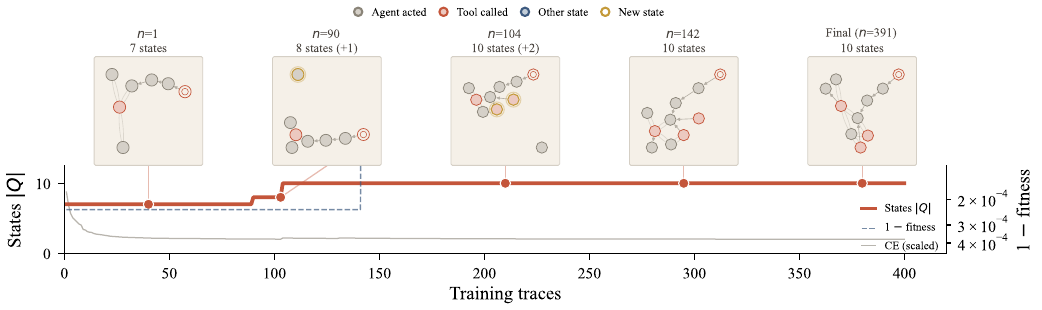}\\[2pt]
\includegraphics[width=\linewidth]{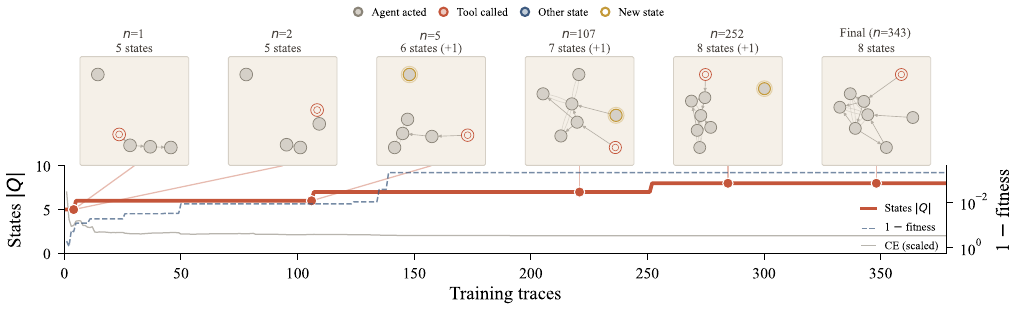}
\captionof{figure}{\textbf{FSM evolution} for Who\_and\_When, SWE-smith, and Mind2Web. State count (red), 1$-$fitness (blue dashed, reverse-log), CE (gray, scaled).}
\label{fig:evolution_group1}
\end{minipage}

\smallskip\noindent\textit{Small-alphabet datasets} ($|\cA|{=}7$--$9$).
All three converge within 5--10\% of training data and stabilize at 8--10 states.
Alergia matches our state count exactly; HMM independently confirms the same structure.
RPNI produces 476--11{,}631 states with degraded fitness ($\leq$0.984), demonstrating that even modest trace corpora cause catastrophic overfitting without structural merging.

\vspace{8pt}

\noindent
\begin{minipage}[t]{0.56\textwidth}
\vspace{0pt}
\centering
\includegraphics[width=\linewidth]{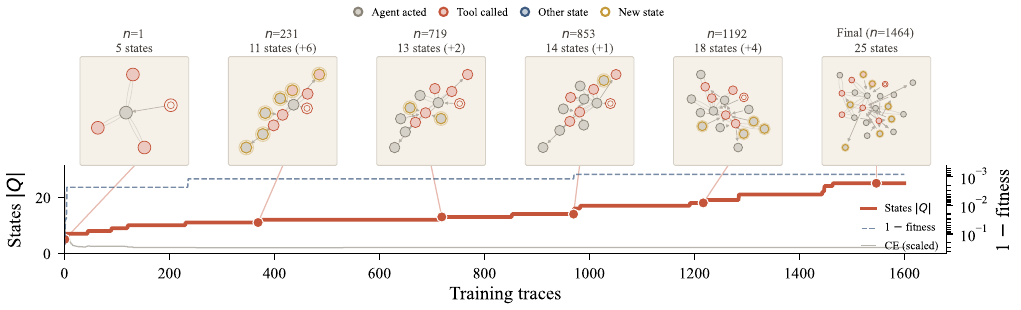}
\captionof{figure}{\textbf{FSM evolution} for SWE-agent. Lines: states (red), 1$-$fitness (blue dashed), CE (gray).}
\label{fig:evolution_group2}
\end{minipage}%
\hfill
\begin{minipage}[t]{0.42\textwidth}
\vspace{0pt}
\captionof{table}{\textbf{Full baseline comparison} (2/4).}\label{tab:full_baselines_1b}
\footnotesize
\setlength{\tabcolsep}{3pt}
\begin{tabular}{@{}l l rr l@{}}
\toprule
Dataset & Method & $|Q|$/Size & Fit & Notes \\
\midrule
\multirow{10}{*}{\rotatebox{90}{\scriptsize SWE-agent}}
 & \cellcolor{cOursBg}\textbf{Ours} & \textbf{25 st} & 0.999 & 2{,}380$\times$ \\
 & \cellcolor{cRPNIBg}RPNI & 59{,}510$^\dagger$ & 0.646 & $^\dagger$ \\
 & \cellcolor{cRPNIBg}EDSM & 1 st & 1.000 & Degen. \\
 & \cellcolor{cAlergiaBg}Alergia & 35 st & 0.999 & 1.4$\times$ \\
 & \cellcolor{cGroupHd}HMM & \textbf{25 st} & \textbf{1.000} & Latent \\
 & \cellcolor{cGroupHd}Heur.\ M & 21p,59t & 0.999 & PN \\
 & \cellcolor{cGroupHd}Ind.\ M & 44p,70t & 0.999 & PN \\
 & \cellcolor{cGroupHd}Alpha M & 12p,24t & 0.050 & PN \\
 & \cellcolor{cGroupHd}AWM & 371 wf & 0.978 & LCS \\
\bottomrule
\end{tabular}
\end{minipage}

\smallskip\noindent\textit{Coding agent} ($|\cA|{=}25$).
SWE-agent shows the most gradual evolution: core search-edit-execute structure emerges by $n{=}80$, but rare commands (e.g., \texttt{deactivate}, \texttt{cd}) continue appearing until $n{=}1{,}200$.

\vspace{8pt}

\noindent
\begin{minipage}[t]{0.42\textwidth}
\vspace{0pt}
\captionof{table}{\textbf{Full baseline comparison} (3/4).}\label{tab:full_baselines_1c}
\footnotesize
\setlength{\tabcolsep}{3pt}
\begin{tabular}{@{}l l rr l@{}}
\toprule
Dataset & Method & $|Q|$/Size & Fit & Notes \\
\midrule
\multirow{7}{*}{\rotatebox{90}{\scriptsize GUI-Odyssey}}
 & \cellcolor{cOursBg}\textbf{Ours} & \textbf{7 st} & \textbf{1.000} & 3{,}036$\times$ \\
 & \cellcolor{cRPNIBg}RPNI & 21{,}255$^\dagger$ & 0.929 & $^\dagger$ \\
 & \cellcolor{cRPNIBg}EDSM & 1 st & 1.000 & Degen. \\
 & \cellcolor{cAlergiaBg}Alergia & 24 st & 0.999 & 3.4$\times$ \\
 & \cellcolor{cGroupHd}HMM & \textbf{7 st} & \textbf{1.000} & Latent \\
 & \cellcolor{cGroupHd}Ind.\ M & -- & \textbf{1.000} & \\
 & \cellcolor{cGroupHd}AWM & 1{,}266 wf & 0.996 & LCS \\
\midrule
\multirow{5}{*}{\rotatebox{90}{\scriptsize WebArena}}
 & \cellcolor{cOursBg}\textbf{Ours} & \textbf{25 st} & \textbf{1.000} & 15$\times$ \\
 & \cellcolor{cRPNIBg}RPNI & 382 st & \textbf{1.000} & 15$\times$ larger \\
 & \cellcolor{cRPNIBg}EDSM & 1 st & 1.000 & Degen. \\
 & \cellcolor{cAlergiaBg}Alergia & 149 st & \textbf{1.000} & 6.0$\times$ \\
 & \cellcolor{cGroupHd}HMM & \textbf{25 st} & \textbf{1.000} & Latent \\
\midrule
\multirow{9}{*}{\rotatebox{90}{\scriptsize AgentNet}}
 & \cellcolor{cOursBg}\textbf{Ours} & \textbf{25 st} & \textbf{1.000} & 2{,}500$\times$ \\
 & \cellcolor{cRPNIBg}RPNI & 62{,}495$^\dagger$ & 0.743 & $^\dagger$ \\
 & \cellcolor{cRPNIBg}EDSM & 1 st & 1.000 & Degen. \\
 & \cellcolor{cAlergiaBg}Alergia & 45 st & \textbf{1.000} & 1.8$\times$ \\
 & \cellcolor{cGroupHd}HMM & \textbf{25 st} & \textbf{1.000} & Latent \\
 & \cellcolor{cGroupHd}Heur.\ M & 35p,84t & 0.987 & PN \\
 & \cellcolor{cGroupHd}Ind.\ M & 18p,39t & \textbf{1.000} & PN \\
 & \cellcolor{cGroupHd}Alpha M & 174p,--t & 0.297 & PN \\
 & \cellcolor{cGroupHd}AWM & -- & 0.975 & LCS \\
\bottomrule
\end{tabular}
\end{minipage}%
\hfill
\begin{minipage}[t]{0.56\textwidth}
\vspace{0pt}
\centering
\includegraphics[width=\linewidth]{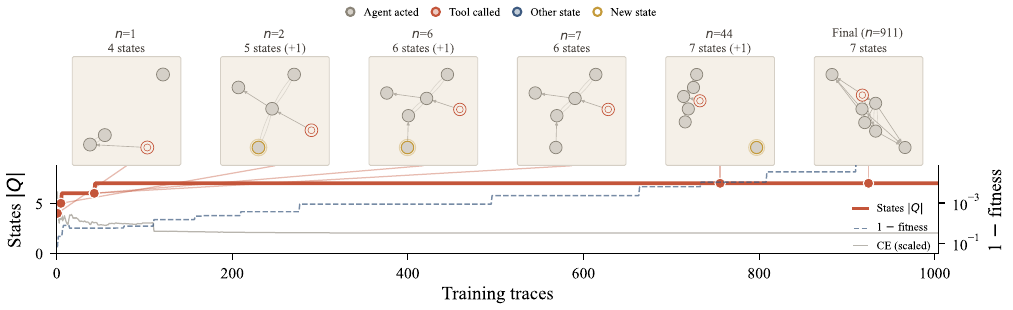}\\[2pt]
\includegraphics[width=\linewidth]{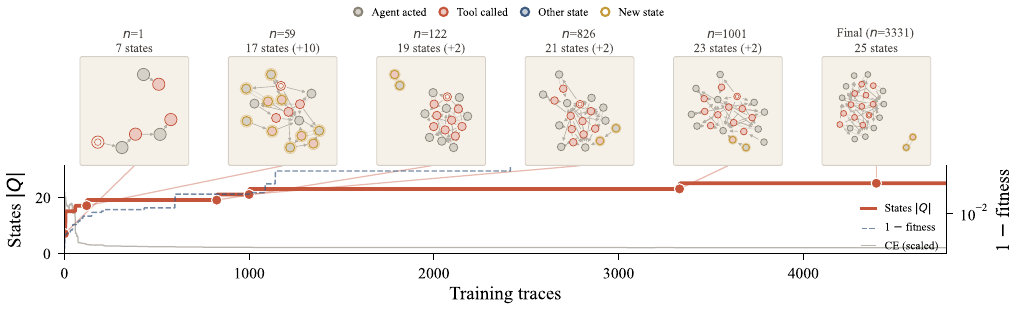}\\[2pt]
\includegraphics[width=\linewidth]{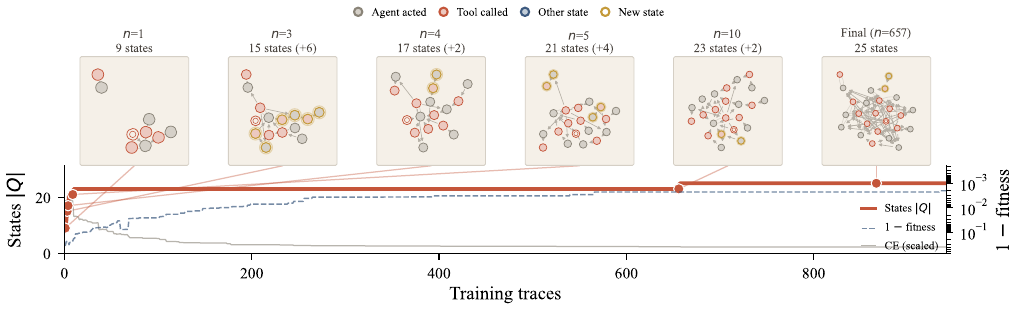}
\captionof{figure}{\textbf{FSM evolution} for GUI-Odyssey, WebArena, and AgentNet. Lines: states (red), 1$-$fitness (blue dashed), CE (gray).}
\label{fig:evolution_group3}
\end{minipage}

\smallskip\noindent\textit{Large-scale datasets} (5{,}000--8{,}337 traces).
Compression ratios peak here: GUI-Odyssey at 3{,}036$\times$ and AgentNet at 2{,}500$\times$.
RPNI completely collapses on AgentNet (62{,}495 states, 0.743 fitness).
WebArena is the exception: short traces ($\sim$8 steps) keep the prefix tree small (382 states), so RPNI achieves perfect fitness, though still 15$\times$ larger.
Alergia diverges most on GUI-Odyssey (3.4$\times$ our state count), where the statistical merge criterion becomes overly conservative with 7{,}735 traces.

\vspace{8pt}

\noindent
\begin{minipage}[t]{0.56\textwidth}
\vspace{0pt}
\centering
\includegraphics[width=\linewidth]{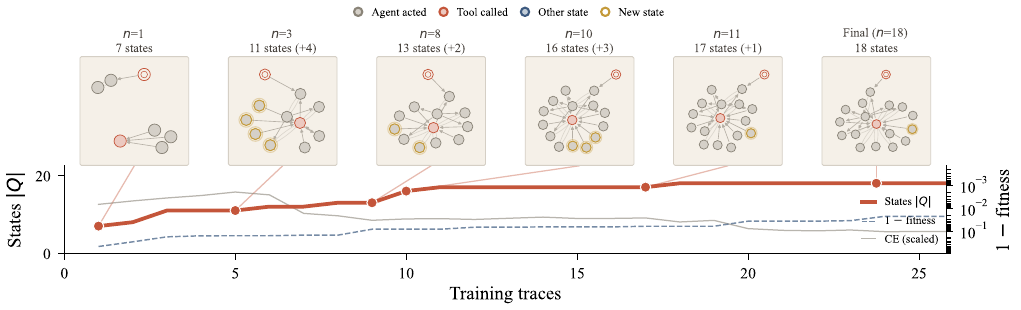}\\[2pt]
\includegraphics[width=\linewidth]{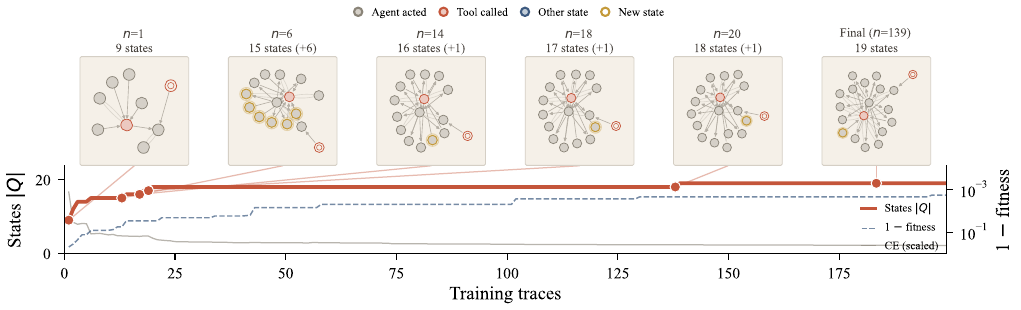}\\[2pt]
\includegraphics[width=\linewidth]{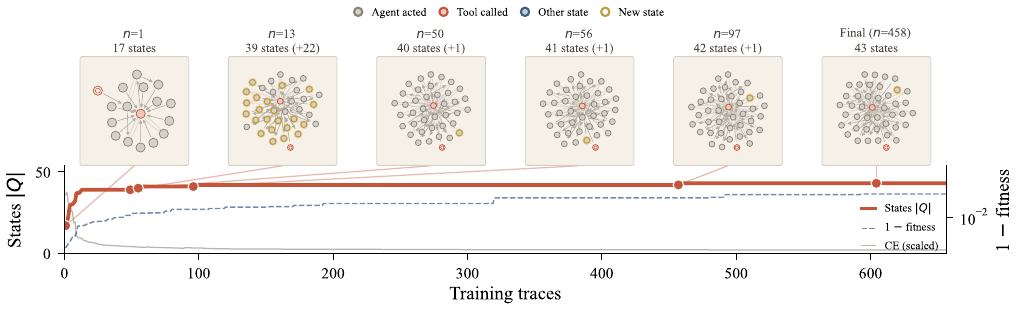}
\captionof{figure}{\textbf{FSM evolution} for tau2-bench (airline, retail, telecom). Multi-model datasets (4 LLMs). Lines: states (red), 1$-$fitness (blue dashed), CE (gray).}
\label{fig:evolution_group4}
\end{minipage}%
\hfill
\begin{minipage}[t]{0.42\textwidth}
\vspace{0pt}
\captionof{table}{\textbf{Full baseline comparison} (4/4).}\label{tab:full_baselines_2}
\footnotesize
\setlength{\tabcolsep}{3pt}
\begin{tabular}{@{}l l rr l@{}}
\toprule
Dataset & Method & $|Q|$/Size & Fit & Notes \\
\midrule
\multirow{5}{*}{\rotatebox{90}{\scriptsize tau2-air}}
 & \cellcolor{cOursBg}\textbf{Ours} & \textbf{18 st} & \textbf{1.000} & 361$\times$ \\
 & \cellcolor{cRPNIBg}RPNI & 6{,}506$^\dagger$ & 0.844 & $^\dagger$ \\
 & \cellcolor{cRPNIBg}EDSM & 1 st & 1.000 & Degen. \\
 & \cellcolor{cAlergiaBg}Alergia & 23 st & 0.999 & 1.3$\times$ \\
 & \cellcolor{cGroupHd}HMM & \textbf{18 st} & \textbf{1.000} & Latent \\
\midrule
\multirow{5}{*}{\rotatebox{90}{\scriptsize tau2-ret}}
 & \cellcolor{cOursBg}\textbf{Ours} & \textbf{19 st} & \textbf{1.000} & 750$\times$ \\
 & \cellcolor{cRPNIBg}RPNI & 14{,}249$^\dagger$ & 0.837 & $^\dagger$ \\
 & \cellcolor{cRPNIBg}EDSM & 1 st & 1.000 & Degen. \\
 & \cellcolor{cAlergiaBg}Alergia & 25 st & \textbf{1.000} & 1.3$\times$ \\
 & \cellcolor{cGroupHd}HMM & \textbf{19 st} & \textbf{1.000} & Latent \\
\midrule
\multirow{5}{*}{\rotatebox{90}{\scriptsize tau2-tel}}
 & \cellcolor{cOursBg}\textbf{Ours} & \textbf{43 st} & \textbf{1.000} & 1{,}486$\times$ \\
 & \cellcolor{cRPNIBg}RPNI & 63{,}897$^\dagger$ & 0.491 & $^\dagger$ \\
 & \cellcolor{cRPNIBg}EDSM & 1 st & 1.000 & Degen. \\
 & \cellcolor{cAlergiaBg}Alergia & 75 st & 0.999 & 1.7$\times$ \\
 & \cellcolor{cGroupHd}HMM & \textbf{43 st} & \textbf{1.000} & Latent \\
\bottomrule
\end{tabular}
\end{minipage}

\smallskip\noindent\textit{Multi-model benchmarks} (4 LLMs per dataset).
A single FSM achieves 1.000 fitness on all four models' traces, confirming model-invariant behavioral topology. tau2-bench telecom has the largest FSM (43 states, 42 tool types), reflecting the complex diagnostic workflow.
The evolution figures show that all three tau2-bench datasets reach structural convergence despite pooling traces from GPT-4.1, Claude~3.7, GPT-4.1-mini, and o4-mini.

\paragraph{Alergia behavior.} Alergia~\citep{carrasco1994alergia} learns a probabilistic DFA (PDFA) from positive examples using a Hoeffding bound to decide state merges. On smaller datasets ($\leq$500 traces), Alergia produces state counts matching our method (8--12 states). On larger datasets, Alergia diverges: 35 states on SWE-agent (1.4$\times$ ours) and 24 on GUI-Odyssey (3.4$\times$ ours), because the statistical test becomes more conservative with more data, splitting states that share structure but differ in probability distributions. Our structural merging is agnostic to transition frequencies, producing deterministic FSMs that are both smaller and faster to construct.

\paragraph{HMM behavior.} HMM (Baum-Welch~\citep{rabiner1989hmm}) with the same number of hidden states as our FSM achieves comparable fitness on all twelve datasets. However, HMM states are latent (unlabeled), making the model non-interpretable: one cannot inspect which behavioral mode a state corresponds to or extract per-state features for downstream analysis. Our FSM states have explicit activity-labeled transitions.

\paragraph{EDSM behavior.} EDSM~\citep{lang1998edsm} uses evidence-driven scoring to rank candidate merges. Without negative examples, the evidence score for every merge candidate is zero, so EDSM greedily merges all states into one. The resulting 1-state universal acceptor has perfect fitness (it accepts everything) but zero precision. This validates that our approach, which also uses only positive examples, achieves meaningful structure (7--43 states with high precision) rather than collapsing.

\paragraph{AWM behavior.} AWM~\citep{wang2024agent_workflow_memory} by default filters for successful traces only. On Who\_and\_When (all failures) and Mind2Web (no success labels), we use AWM-all which skips the success filter. AWM extracts linear workflows via longest common subsequence alignment; it cannot represent cycles or branching. Its coverage metric (0.887--1.000) is not directly comparable to replay fitness.

\paragraph{PM4Py precision.} The ``flower model'' problem is well-known in process mining: miners that accept all possible orderings achieve high fitness but low precision. Our precision analysis (Table~\ref{tab:precision}) confirms this: PM4Py miners achieve precision of 0.00--0.80 across datasets, while our FSM achieves near-zero random acceptance.

\paragraph{k-Tails behavior.} k-Tails~\citep{biermann1972ktails} merges states sharing identical $k$-length futures; as for Alergia, every trace is prefixed with a shared start symbol. Table~\ref{tab:ktails_app} reports results for $k \in \{1,2,3\}$. At $k{=}1$, k-Tails produces 1.0--10$\times$ as many states as our method with equal or lower test fitness (0.93--1.00 vs.\ $\geq$0.997). At $k{\geq}2$, state counts explode to hundreds or thousands on most datasets, with timeouts on tau2-bench telecom and on AgentNet at $k{=}3$. Outside WebArena, no value of $k$ simultaneously matches our compression and fitness, illustrating why hyperparameter-free structural merging is preferable.

\begin{table}[t]
\centering
\caption{\textbf{k-Tails results.} $|Q|$: states. Fit: test fitness. TO: timeout (300s). Our method requires no hyperparameter.}
\label{tab:ktails_app}
\scriptsize
\tablestyle{3pt}{1.02}
\begin{tabular}{@{}l rr rr rr@{}}
\toprule
& \multicolumn{2}{c}{\cellcolor{cOursBg}Ours} & \multicolumn{2}{c}{k-Tails ($k{=}1$)} & \multicolumn{2}{c}{k-Tails ($k{=}2$)} \\
\cmidrule(lr){2-3} \cmidrule(lr){4-5} \cmidrule(lr){6-7}
Dataset & \cellcolor{cOursBg}$|Q|$ & \cellcolor{cOursBg}Fit & $|Q|$ & Fit & $|Q|$ & Fit \\
\midrule
SWE-smith & \cellcolor{cOursBg}\textbf{10} & \cellcolor{cOursBg}1.000 & 23 & 1.000 & 54 & 1.000 \\
SWE-agent & \cellcolor{cOursBg}\textbf{25} & \cellcolor{cOursBg}0.999 & 75 & 0.996 & 333 & 0.971 \\
WebArena & \cellcolor{cOursBg}\textbf{25} & \cellcolor{cOursBg}1.000 & 25 & 1.000 & 164 & 1.000 \\
AgentNet & \cellcolor{cOursBg}\textbf{25} & \cellcolor{cOursBg}1.000 & 25 & 0.996 & 1{,}544 & 0.861 \\
\addlinespace[2pt]
tau2-bench (air) & \cellcolor{cOursBg}\textbf{18} & \cellcolor{cOursBg}1.000 & 38 & 0.973 & 732 & 0.773 \\
tau2-bench (ret) & \cellcolor{cOursBg}\textbf{19} & \cellcolor{cOursBg}1.000 & 47 & 0.997 & 875 & 0.888 \\
tau2-bench (tel) & \cellcolor{cOursBg}\textbf{43} & \cellcolor{cOursBg}1.000 & 211 & 0.930 & \multicolumn{2}{c}{TO} \\
\addlinespace[2pt]
Who\_and\_When & \cellcolor{cOursBg}\textbf{9} & \cellcolor{cOursBg}1.000 & 15 & 1.000 & 21 & 1.000 \\
Mind2Web & \cellcolor{cOursBg}\textbf{8} & \cellcolor{cOursBg}1.000 & 35 & 0.965 & 96 & 0.905 \\
GUI-Odyssey & \cellcolor{cOursBg}\textbf{7} & \cellcolor{cOursBg}1.000 & 70 & 0.998 & 803 & 0.969 \\
\bottomrule
\end{tabular}
\end{table}

\begin{figure}[H]
\centering
\includegraphics[width=\linewidth]{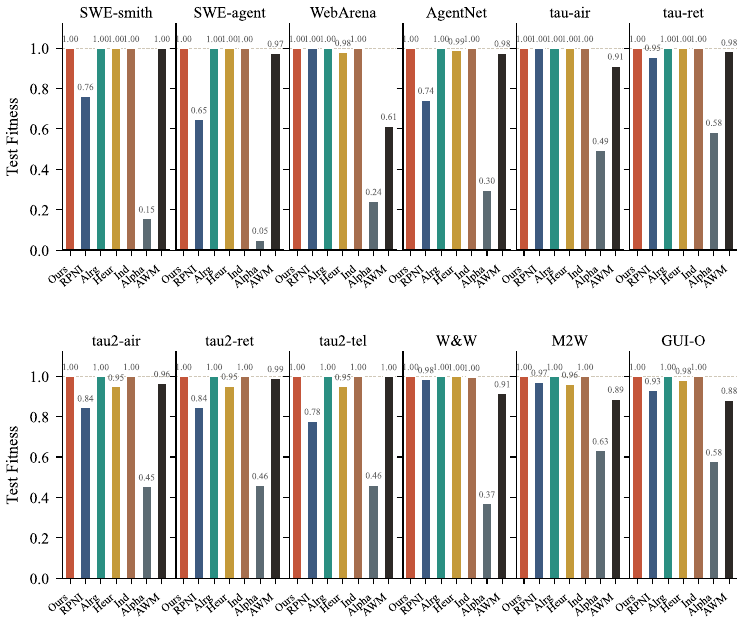}
\caption{\textbf{Per-dataset test fitness} across all methods including Alergia. Our FSM achieves $\geq$0.997 fitness on eleven datasets. Process mining baselines (Heur., Ind.) achieve competitive fitness but with low precision (Table~\ref{tab:precision}).}
\label{fig:baselines}
\end{figure}

\subsection{Implementation Details}\label{app:implementation}

All experiments use a single fixed seed for train/test splitting via seeded Fisher-Yates shuffle.
The 80/20 split is applied consistently across all datasets and baselines.

\paragraph{RPNI.} A TypeScript implementation with a 120-second timeout for the merge phase. Without negative examples, RPNI cannot safely merge states, which results in near-complete prefix trees. The timeout is necessary for SWE-agent (59{,}510 states), GUI-Odyssey (21{,}255 states), and SWE-smith (11{,}631 states).

\paragraph{EDSM.} AALpy~\citep{muskardin2022aalpy} version 1.3.4. Without negative examples, evidence-driven scoring produces no merge candidates, collapsing to a 1-state universal acceptor.

\paragraph{Alergia.} AALpy implementation with Markov Chain automaton type and $\varepsilon = 0.005$ (Hoeffding bound parameter). Lower $\varepsilon$ produces more merging (fewer states). Alergia learns from positive examples only, with statistical compatibility testing.

\paragraph{HMM.} Pure NumPy implementation of Baum-Welch EM~\citep{rabiner1989hmm} with 50 iterations and scaled forward-backward to prevent numerical underflow. The number of hidden states is set to match our FSM's state count for a fair comparison. Replay fitness is computed via Viterbi decoding: a symbol is consumed if the decoded state has nonzero emission probability ($>10^{-6}$). Effective transitions are counted as entries in the transition matrix with probability $>0.01$.

\paragraph{PM4Py.} Traces are converted to XES event log format. We run Alpha Miner, Inductive Miner, and Heuristic Miner with default parameters. Fitness and precision are computed via PM4Py's conformance checking.

\paragraph{AWM.} We re-implement workflow extraction from~\citet{wang2024agent_workflow_memory}. Coverage is measured via longest common subsequence alignment.

\paragraph{Failure prediction.} Feature selection uses L1-regularized logistic regression ($C=0.1$) on the training set to select features with non-zero weights. Final models use class-weighted loss (inverse frequency). Cross-validation uses stratified 5-fold with 10 repetitions, evaluated on training data only (no test data in CV folds).

\paragraph{Runtime.} FSM construction completes in under 1 second for all datasets. RPNI requires up to 120 seconds (timeout). PM4Py baselines complete in 5--30 seconds.

\subsection{Runtime Details}\label{app:runtime}

\noindent
\begin{minipage}[t]{0.38\linewidth}
\vspace{0pt}
Construction time scales linearly with corpus size (Table~\ref{tab:runtime}). RPNI's quadratic merge loop ($O(n^2)$ state pairs) causes timeouts at 120\,s on most datasets (only Mind2Web completes within the limit). On SWE-agent, RPNI produces 59{,}510 states vs.\ our 25. Per-trace replay is nearly instantaneous ($<$0.015\,ms), which makes FSM-based monitoring practical for real-time agent systems (Figure~\ref{fig:runtime}).
\end{minipage}\hfill
\begin{minipage}[t]{0.60\linewidth}
\vspace{0pt}
\captionof{table}{\textbf{Runtime comparison.} PTA: prefix tree construction. Merge: state merging. Replay: per-trace replay latency.}
\label{tab:runtime}
\scriptsize
\tablestyle{3pt}{1.02}
\footnotesize
\begin{tabular}{@{}l rr rr r@{}}
\toprule
& \multicolumn{2}{c}{\cellcolor{cOursBg}Ours (ms)} & \multicolumn{2}{c}{RPNI (ms)} & \\
\cmidrule(lr){2-3} \cmidrule(lr){4-5}
Dataset & \cellcolor{cOursBg}Build & \cellcolor{cOursBg}Replay/tr & Build & States & Speedup \\
\midrule
Mind2Web & \cellcolor{cOursBg}1.0 & \cellcolor{cOursBg}0.003 & 6{,}902 & 476 & 6{,}579$\times$ \\
Who\_and\_When & \cellcolor{cOursBg}4.6 & \cellcolor{cOursBg}0.005 & 30{,}005$^\dagger$ & 971 & 6{,}502$\times$ \\
SWE-smith & \cellcolor{cOursBg}30.4 & \cellcolor{cOursBg}0.015 & 30{,}253$^\dagger$ & 11{,}631 & 994$\times$ \\
SWE-agent & \cellcolor{cOursBg}109.8 & \cellcolor{cOursBg}0.008 & 36{,}008$^\dagger$ & 59{,}510 & 328$\times$ \\
\bottomrule
\end{tabular}
\end{minipage}

\begin{figure}[H]
\centering
\includegraphics[width=\linewidth]{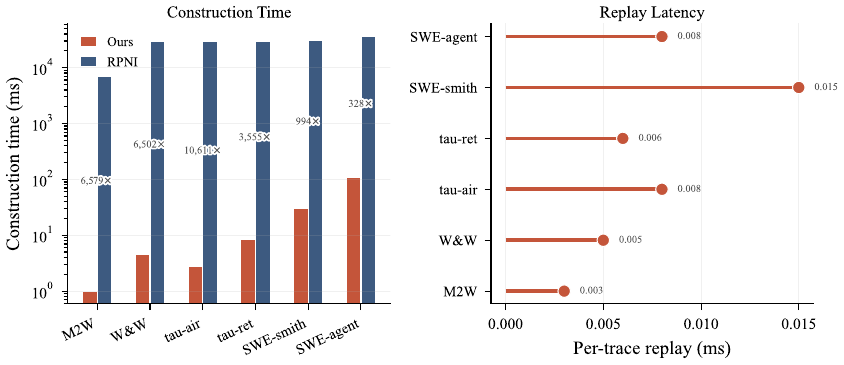}
\caption{\textbf{Runtime comparison.} Left: FSM construction time (1--110\,ms) vs.\ RPNI (7{,}000--36{,}000\,ms) with speedup ratios annotated. Right: per-trace replay latency (0.003--0.015\,ms), enabling real-time monitoring.}
\label{fig:runtime}
\end{figure}

\subsection{Case Study: FSM Visualizations}\label{app:case_studies}

We present FSM visualizations for representative datasets to illustrate the behavioral structure our method recovers.

\subsubsection{Customer Service Agent (tau2-bench airline)}

Figure~\ref{fig:example_fsm} illustrates the tau2-bench airline FSM at role-level granularity (6 states, $|\cA|{=}5$).
The full FSM at the tool level (18 states, $|\cA|{=}17$, Table~\ref{tab:main}) decomposes further into states per tool.
At role level, the FSM reveals two distinct behavioral loops:

\begin{enumerate}[leftmargin=*,itemsep=1pt]
\item \textbf{Tool-call loop} (\texttt{assistant:tool\_call} $\leftrightarrow$ \texttt{tool:text}): The agent queries customer databases (get\_reservation, search\_flights) and receives structured responses. This loop executes 2--8 times per trace.
\item \textbf{Conversation loop} (\texttt{assistant:text} $\to$ \texttt{user:text}): The agent communicates results to the user and receives follow-up requests. Failed traces show more turns of conversation (mean 4.2 vs 2.8 for successes), suggesting the agent struggles to complete the task.
\end{enumerate}

The FSM makes these patterns structurally visible: the tool-call loop appears as a tight 2-state cycle, while the conversation loop passes through the user state.
This decomposition enables per-state analysis (e.g., error rates in the tool state, message lengths in the assistant state) that would be obscured in raw trace analysis.

\subsubsection{Coding Agent (SWE-smith)}

\begin{figure}[H]
\centering
\begin{tikzpicture}[
    ->,>=Stealth,
    node distance=1.3cm and 1.8cm,
    state/.style={circle, draw, minimum size=0.6cm, font=\tiny\sffamily, inner sep=1pt},
    every edge/.style={draw, font=\tiny},
    scale=0.85, transform shape
  ]
  \node[state, fill=gray!15] (init) {init};
  \node[state, fill=blue!12, right=of init] (sys) {sys};
  \node[state, fill=green!12, right=of sys] (usr) {usr};
  \node[state, fill=orange!12, right=1.5cm of usr] (bash) {bash};
  \node[state, fill=purple!12, below=1.2cm of bash] (tool) {tool};
  \node[state, fill=yellow!15, left=1.5cm of tool] (edit) {edit};
  \node[state, fill=red!12, below=1.2cm of edit] (sub) {sub};
  \node[state, fill=cyan!12, right=of tool] (txt) {a:txt};
  \node[state, fill=gray!30, below right=0.6cm and 0.4cm of edit] (tstr) {t:str};
  \node[state, fill=gray!30, below left=0.6cm and 0.4cm of bash] (ttc) {t:tc};

  \path (init) edge node[above] {} (sys);
  \path (sys) edge node[above] {} (usr);
  \path (usr) edge node[above] {} (bash);
  \path (bash) edge node[right] {} (tool);
  \path (tool) edge node[above] {} (edit);
  \path (edit) edge[bend left=15] node[left] {} (tool);
  \path (tool) edge node[above] {} (bash);
  \path (tool) edge node[above] {} (sub);
  \path (tool) edge node[above] {} (txt);
  \path (edit) edge node[right] {} (tstr);
  \path (tstr) edge[bend left=15] node[right] {} (edit);
  \path (tstr) edge node[above] {} (bash);
  \path (bash) edge node[left] {} (ttc);
  \path (ttc) edge[bend left=15] node[above] {} (bash);
  \path (ttc) edge node[above] {} (edit);
  \path (edit) edge node[above] {} (ttc);
  \path (sub) edge node[left] {} (tool);
\end{tikzpicture}
\caption{\textbf{SWE-smith FSM} (10 states, 17 transitions). In the core cycle \texttt{bash}$\to$\texttt{tool}$\to$\texttt{edit}$\to$\texttt{tool}$\to$\texttt{bash}, agents execute commands, inspect results, make edits, and repeat. After successful editing, the agent reaches the \texttt{submit} state, a terminal action.}
\label{fig:swesmith_fsm}
\end{figure}

The SWE-smith FSM (Figure~\ref{fig:swesmith_fsm}) captures the workflow of the coding agent across 10 states.
The dominant pattern is the \texttt{bash}$\to$\texttt{tool}$\to$\texttt{str\_replace\_editor}$\to$\texttt{tool} cycle: the agent executes bash commands, inspects output, applies code edits, and verifies results.
Failed traces (123/500) show higher visit counts in the edit-tool loop (mean 8.4 vs 5.1 for successes) and elevated error rates in the \texttt{str\_replace\_editor} state (0.31 vs 0.12), which suggests repeated failed attempts at editing.

\subsubsection{Coding Agent at Scale (SWE-agent)}

The SWE-agent FSM has 25 states. We derive it from 2{,}000 traces and group their 175 raw commands into 7 categories.
Its larger state space (compared to SWE-smith's 10) reflects the richer command vocabulary and longer traces (mean 87 steps vs 23).
Key structural features include:

\begin{itemize}[leftmargin=*,itemsep=1pt]
\item A \textbf{search-navigate cycle} (\texttt{search} $\leftrightarrow$ \texttt{navigate}): agents find relevant files and navigate to specific locations.
\item An \textbf{edit-execute cycle} (\texttt{edit} $\leftrightarrow$ \texttt{execute}): agents modify code and run tests.
\item A \textbf{submit terminal}: successful traces reach the \texttt{submit} state, which is the strongest failure prediction feature (weight 0.29).
\end{itemize}

The 2{,}380$\times$ compression (59{,}510 RPNI states $\to$ 25 FSM states) shows that even agents with complex command vocabularies exhibit a small number of distinct behavioral modes when commands are grouped by semantic function.
Alergia, the strongest positive-only baseline, needs 35 states (1.4$\times$) at the same fitness.

\section{Failure Analysis}\label{app:group:failure}

\subsection{Structural Divergence of Success vs.\ Failure}\label{app:structural_divergence}

Building separate FSMs for successful and failed traces reveals qualitative differences in structure.
On SWE-agent, the success FSM uses only 9 of 25 states with 16 transitions (focused path: search-edit-execute-submit), while the failure FSM spans all 25 states with 60 transitions (chaotic exploration).
The transition Jaccard similarity is 0.206, which indicates largely disjoint behavioral structures; all 16 failure-only states correspond to rare tool output parsing variants (e.g., \texttt{user:tool:of}, \texttt{user:versioneer}) that successful traces never encounter.
On tau2-bench, both classes produce structurally identical FSMs (discrimination gap $= 0$), confirming that failure on constrained API-calling tasks manifests purely in transition frequencies rather than novel states.
Failed traces are consistently longer across all datasets (+29 steps on SWE-agent, +18 on SWE-smith) but visit the same or fewer unique states, suggesting that failure manifests as cycling through familiar states rather than exploring new ones.

These structural differences enable simple rules for monitoring without ML models.
On SWE-agent, ``cycle rate $> 0.885$'' achieves 95.6\% precision (\S\ref{app:monitoring}).

\subsection{Failure Prediction Numerical Results}\label{app:failure_prediction}

\begin{table}[H]
\centering
\caption{\textbf{Failure prediction} from FSM features. Held-out: fixed 80/20 split test AUROC. CV: 5-fold $\times$ 10 repeats (50 model fits) on the train split, mean ± std. Single GBT classifier (200 trees, depth 3, class-weighted). tau2-bench aggregates 4 models.}
\label{tab:failure}
\scriptsize
\tablestyle{3pt}{0.95}
\begin{tabular}{@{}l l r r@{}}
\toprule
Dataset & Type & Holdout & CV (mean ± std) \\
\midrule
tau2-bench (tel) & real          & \hlbest{0.941} & 0.946 ± 0.012 \\
WebArena         & real          & 0.903          & 0.882 ± 0.012 \\
AgentNet         & real          & 0.890          & 0.886 ± 0.012 \\
ATBench          & real (safety) & 0.894          & 0.864 ± 0.024 \\
tau2-bench (air) & real          & 0.864          & 0.826 ± 0.031 \\
SWE-agent        & real          & 0.799          & 0.806 ± 0.021 \\
tau2-bench (ret) & real          & 0.779          & 0.752 ± 0.024 \\
OSWorld          & real          & 0.774          & 0.779 ± 0.025 \\
\midrule
SWE-smith        & synthetic     & 0.703          & 0.636 ± 0.060 \\
\bottomrule
\end{tabular}
\end{table}

\subsection{Discriminative Quotient (FSM-D)}\label{app:fsm_d}

\spaceskip=\fontdimen2\font plus \fontdimen3\font minus 0.05\fontdimen4\font\relax
The standard construction (Theorem~\ref{prop:minimal}) merges trie states by their incoming activity (the last-activity right congruence).
We additionally explore a \emph{discriminative} variant that conditions on the trace label: at training time, partition traces into success and failure subsets, and for each state $q$ (we abuse notation to write $P^\pm(\cdot\mid q)$ instead of the body's $P^\pm(\cdot\mid q_t)$) compute the outgoing transition distributions $P^{+}(\cdot \mid q)$ and $P^{-}(\cdot \mid q)$ via add-$\alpha$ smoothing (same $\alpha{=}1$ as Eq.~\ref{eq:fsm_predict}) so that supports coincide and KL is finite.
The Kullback--Leibler divergence $\mathrm{KL}(P^{+}\,\|\,P^{-}) = \sum_{a\in\cA} P^+(a\mid q)\log[P^+(a\mid q)/P^-(a\mid q)]$ measures how much the success distribution deviates from the failure distribution at state $q$; large values indicate that the state is a behavioral choke point separating the two classes.

\spaceskip=0pt\relax

\begin{table}[H]
\centering
\caption{\textbf{FSM-D discriminative features} on 9 labeled datasets. Length+Entropy: trace length and per-trace activity entropy only. FSM-D adds to Length+Entropy the per-state success rate, mean KL$(P^+\|P^-)$, mean total variation, min success rate, max KL, weighted sums. Logistic regression with L2 regularisation, 80/20 split, fixed seed.}
\label{tab:fsm_d}
\scriptsize
\tablestyle{3pt}{0.95}
\begin{tabular}{@{}l rrr r@{}}
\toprule
Dataset & $|\cA|$ & Len+Ent & \cellcolor{cOursBg}FSM-D & $\Delta$ \\
\midrule
SWE-smith        &  9 & 0.695 & \cellcolor{cOursBg}0.702 & +0.006 \\
SWE-agent        & 24 & 0.690 & \cellcolor{cOursBg}0.689 & $-$0.001 \\
WebArena         & 24 & 0.718 & \cellcolor{cOursBg}\hlbest{0.773} & +0.056 \\
AgentNet         & 24 & 0.724 & \cellcolor{cOursBg}0.724 & +0.000 \\
tau2-bench (air) & 17 & 0.625 & \cellcolor{cOursBg}\hlbest{0.777} & \textbf{+0.152} \\
tau2-bench (ret) & 18 & 0.596 & \cellcolor{cOursBg}\hlbest{0.660} & +0.064 \\
tau2-bench (tel) & 42 & 0.725 & \cellcolor{cOursBg}\hlbest{0.857} & \textbf{+0.132} \\
ATBench          & 14 & 0.555 & \cellcolor{cOursBg}\hlbest{0.703} & \textbf{+0.147} \\
OSWorld          & 26 & 0.770 & \cellcolor{cOursBg}\hlbest{0.778} & +0.008 \\
\bottomrule
\end{tabular}
\end{table}

With no learned classifier beyond logistic regression on 10 inputs, the discriminative features alone lift held-out AUROC over a length+entropy baseline by $+$0.06 to $+$0.15 on five of nine datasets, with the largest gains on datasets whose activities have the most discriminative outgoing distributions (mean KL on tau2-bench airline: 0.014; ATBench: 0.020).
OSWorld is the exception, with the highest mean KL (0.136, with 3 of 26 activities at KL$>$0.3) but a gain of only $+$0.008.
On datasets with near-uniform outgoing distributions across success and failure (SWE-agent: mean KL 0.001; AgentNet: 0.000), FSM-D gives no gain, consistent with the discriminative signal being a property of the agent's behavioral divergence rather than a universal lift.
The FSM-D variant complements the cross-entropy anomaly features used in the main results: where outgoing distributions diverge, FSM-D contributes a training-free signal; where they do not, the main pipeline's per-state visit features and the trace cross-entropy carry the predictive load.

\subsection{Hierarchical FSM with Information-Bottleneck Quotient}\label{app:hierarchical_ib}

The algebraic quotient (Theorem~\ref{prop:minimal}) and FSM-D (\S\ref{app:fsm_d}) both produce a single-level partition of activities.
We additionally develop an outcome-aware \emph{Information-Bottleneck (IB) quotient} (Proposition~\ref{prop:ib_bound}) that yields a multi-level hierarchy.
Given the empirical activity-transition matrix $\mathbf{T}\in\mathbb{R}^{|\cA|\times|\cA|}$ and the per-activity outcome distribution $P(Y\mid a)$, the IB objective at level $\ell$ is $\min_{q(c\mid a)} I(C;A) - \beta\, I(C;Y)$.
We apply this objective iteratively: level 1 partitions activities into $K_1$ super-states using $\mathbf{T}$ and outcome marginals; level 2 partitions level-1 super-states into $K_2$ phases using the induced super-state transition matrix.

\begin{table}[H]
\centering
\caption{\textbf{Hierarchical IB FSM} (level 1: $K_1$ super-states; level 2: $K_2$ phases) on eight labeled datasets.
``Std.'': a pipeline with the same features that uses only level 0 (activity counts).
$\Delta$: AUROC gain over Std.
Replay fitness is at or above 0.999 at both levels on all datasets.}
\label{tab:hier_ib}
\footnotesize
\setlength{\tabcolsep}{4pt}
\begin{tabular}{@{}l rrr ccc r@{}}
\toprule
Dataset & $|\cA|$ & best $K_1$ & best $K_2$ & Std.\ AUROC & \cellcolor{cOursBg}Hier-IB AUROC & $\Delta$ & best $\beta$ \\
\midrule
SWE-smith         &  9 & 6 & 3 & 0.699 & \cellcolor{cOursBg}0.704 & +0.005  & 0.5 \\
SWE-agent         & 21 & 2 & 2 & 0.658 & \cellcolor{cOursBg}0.658 & $+0.000$ & 2.0 \\
WebArena          &  6 & 4 & 2 & 0.726 & \cellcolor{cOursBg}\hlbest{0.756} & \textbf{+0.030} & 2.0 \\
tau2-bench (air)  & 17 & 6 & 3 & 0.648 & \cellcolor{cOursBg}\hlbest{0.717} & \textbf{+0.069} & 1.0 \\
tau2-bench (tel)  & 41 & 6 & 3 & 0.786 & \cellcolor{cOursBg}\hlbest{0.851} & \textbf{+0.065} & 2.0 \\
tau2-bench (ret)  & 18 & 2 & 2 & 0.598 & \cellcolor{cOursBg}\hlbest{0.637} & \textbf{+0.039} & 1.0 \\
ATBench           & 14 & 6 & 3 & 0.553 & \cellcolor{cOursBg}\hlbest{0.696} & \textbf{+0.143} & 0.5 \\
OSWorld           & 26 & 2 & 2 & 0.760 & \cellcolor{cOursBg}0.761 & +0.001 & 1.0 \\
\midrule
\textbf{Mean}     &    &   &   & 0.679 & \cellcolor{cOursBg}\textbf{0.722} & \textbf{+0.044} &     \\
\bottomrule
\end{tabular}
\end{table}

The IB quotient is genuinely orthogonal to the algebraic quotient: it requires outcome supervision but in exchange can merge activities whose outgoing distributions differ when the differences are not outcome-relevant, and split activities whose outgoing distributions agree when their conditional outcome distributions diverge;
the two-level hierarchy ($K_1$ super-states $\to$ $K_2$ phases) provides a compositional view of agent behavior: micro-actions $\to$ behavioral phases $\to$ trajectory-level outcomes, with the rate-distortion bound (Proposition~\ref{prop:ib_bound}) providing the formal guarantee that each level retains at most $\beta$-controlled outcome information.
Six of the eight datasets show gains ($+$0.030 to $+$0.143). SWE-agent and OSWorld saturate at the activity level under this LR pipeline, consistent with $H(Y\mid A)$ already being small.

\subsection{Alergia Features under Matched Pipeline}\label{app:alergia_downstream}

To check that the downstream gain is not just an artifact of state-count differences, we apply the identical feature extraction and classifier pipeline to Alergia-extracted FSMs (1.0--6.0$\times$ more states than ours; \S\ref{sec:exp:main}).
With logistic regression on 34--175 per-state features (visit frequency, mean/max message length, error rate, temporal entropy), 10$\times$5-fold CV on training data, our FSM beats Alergia on 8 of 9 labeled datasets at the time of this comparison (Table~\ref{tab:alergia_downstream}); the margin reaches $+$0.18 on tau-bench retail.
The additional Alergia states from statistical merging dilute per-state observation counts rather than improving them, consistent with the $O(1/\sqrt{n_q})$ estimator bound (Proposition~\ref{prop:estimator}).

\begin{table}[H]
\centering
\caption{\textbf{Failure prediction with Alergia FSMs under matched pipeline.} Identical features, classifier, and CV protocol; FSM source varies. CV: 10$\times$5-fold; Holdout: held-out test AUROC.}
\label{tab:alergia_downstream}
\scriptsize
\tablestyle{3pt}{0.95}
\begin{tabular}{@{}l cc cc@{}}
\toprule
& \multicolumn{2}{c}{CV AUROC} & \multicolumn{2}{c}{Holdout AUROC} \\
\cmidrule(lr){2-3} \cmidrule(lr){4-5}
Dataset & \cellcolor{cOursBg}Ours & Alergia & \cellcolor{cOursBg}Ours & Alergia \\
\midrule
tau2-bench (tel) & \cellcolor{cOursBg}\hlbest{0.923} & 0.748 & \cellcolor{cOursBg}\hlbest{0.915} & 0.752 \\
WebArena         & \cellcolor{cOursBg}\hlbest{0.864} & 0.844 & \cellcolor{cOursBg}\hlbest{0.888} & 0.868 \\
AgentNet         & \cellcolor{cOursBg}0.871 & 0.871 & \cellcolor{cOursBg}0.872 & 0.872 \\
SWE-agent        & \cellcolor{cOursBg}\hlbest{0.790} & 0.709 & \cellcolor{cOursBg}\hlbest{0.805} & 0.714 \\
tau2-bench (air) & \cellcolor{cOursBg}\hlbest{0.792} & 0.664 & \cellcolor{cOursBg}\hlbest{0.826} & 0.764 \\
tau-bench (ret)  & \cellcolor{cOursBg}\hlbest{0.789} & 0.606 & \cellcolor{cOursBg}\hlbest{0.666} & 0.498 \\
tau-bench (air)  & \cellcolor{cOursBg}\hlbest{0.758} & 0.651 & \cellcolor{cOursBg}\hlbest{0.841} & 0.810 \\
tau2-bench (ret) & \cellcolor{cOursBg}\hlbest{0.713} & 0.576 & \cellcolor{cOursBg}\hlbest{0.743} & 0.575 \\
SWE-smith        & \cellcolor{cOursBg}\hlbest{0.685} & 0.653 & \cellcolor{cOursBg}0.673 & 0.683 \\
\bottomrule
\end{tabular}
\end{table}

\spaceskip=\fontdimen2\font plus \fontdimen3\font minus 0.05\fontdimen4\font\relax
The current main-text headline numbers (\S\ref{sec:exp:downstream}; up to 0.941 holdout) use a richer cross-entropy anomaly feature set with a gradient-boosted classifier; we expect the same direction of comparison (ours $>$ Alergia) under that pipeline because the bottleneck for Alergia features is per-state observation count, not feature engineering.

\spaceskip=0pt\relax

\subsection{Failure Prediction Feature Analysis}\label{app:features}

Table~\ref{tab:top_features} presents the top 10 features by absolute L1-regularized logistic regression weight for each labeled dataset.

\begin{table}[H]
\centering
\caption{\textbf{Top features for failure prediction} by L1-regularized weight. Positive weight = associated with success; negative = associated with failure.}
\label{tab:top_features}
\small
\begin{adjustbox}{max width=\textwidth}
\begin{tabular}{@{}l l r l@{}}
\toprule
Dataset & Feature & Weight & Interpretation \\
\midrule
\multirow{5}{*}{SWE-agent}
 & terminal:submit:text & +0.292 & Reaching submit state \\
 & avgMsgLen:edit:text & $-$0.231 & Long edit messages \\
 & uniqueStatesVisited & +0.217 & Exploring more states \\
 & earlyHalfEntropy & +0.215 & Diverse early behavior \\
 & errorRate:edit:text & $-$0.170 & Edit errors \\
\midrule
\multirow{5}{*}{SWE-smith}
 & maxMsgLen:tool:text & $-$0.290 & Long tool output \\
 & avgMsgLen:str\_replace\_editor & $-$0.280 & Long edit commands \\
 & errorRate:str\_replace\_editor & $-$0.200 & Edit errors \\
 & lateHalfEntropy & $-$0.190 & Chaotic late behavior \\
 & selfLoopDrift & +0.165 & Increasing self-loops \\
\bottomrule
\end{tabular}
\end{adjustbox}
\end{table}

\paragraph{Interpretable patterns.} The features reveal consistent failure signals across datasets:
\begin{itemize}[leftmargin=*,itemsep=1pt]
\item \textbf{Not reaching terminal states} (SWE-agent: submit state weight +0.29): failed agents get stuck in intermediate loops.
\item \textbf{Elevated error rates} (SWE-smith: editor error rate $-$0.20): failed agents encounter more errors per state.
\item \textbf{Verbose responses} (tau2-bench airline: avgMsgLen:assistant): verbose assistant responses correlate with task difficulty and failure.
\item \textbf{Temporal entropy} (SWE-smith: lateHalfEntropy $-$0.19): chaotic behavior in the second half of execution indicates the agent is struggling.
\end{itemize}

These patterns are only visible through the FSM's state decomposition: raw features at the trace level (total message length, total error count) do not capture which behavioral mode produced the errors.
Notably, the two most discriminative feature families, per-state error rates and temporal-entropy splits, appear in the top 5 across labeled datasets despite their different domains (coding, API-calling), suggesting that FSM-conditioned features transfer well across agent architectures.

\begin{figure}[H]
\centering
\includegraphics[width=\linewidth]{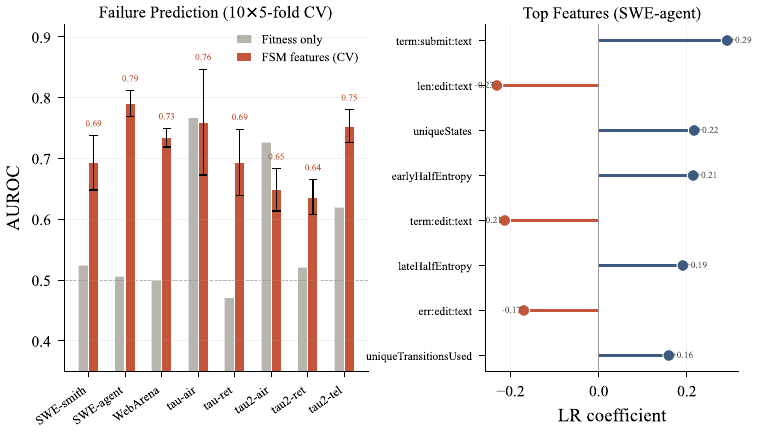}
\caption{\textbf{Failure prediction:} fitness-only AUROC (gray, $\approx$0.50) vs.\ FSM feature-based CV AUROC (red) with error bars. Right: top features as lollipop chart (positive = success, negative = failure).}
\label{fig:failure_pred}
\end{figure}

\paragraph{Neural model comparison.}\label{app:neural_probe}
Table~\ref{tab:neural_probe} compares neural models (MLP, GRU~\citep{cho2014gru}, Transformer~\citep{vaswani2017transformer}) trained on sequence features (bag-of-activities + length + entropy) vs.\ FSM per-state features. All models use embed/hidden=64/128, class weighting, early stopping (patience 10), and the same 10$\times$5-fold CV protocol. Across this appendix, we retain SWE-smith's numbers even when they are the weakest per-dataset value: the dataset is synthetic and below the sample threshold of Proposition~\ref{prop:estimator}, so its weaker AUROCs (0.62--0.72 across sub-ablations) are both expected and consistent with the main-text framing that FSM features generalize more strongly to real LLM-deployment traces.

\begin{table}[H]
\centering
\caption{Failure prediction: neural models on sequence features vs.\ FSM features. CV AUROC (10$\times$5-fold). Bold: FSM$>$Seq for same model.}
\label{tab:neural_probe}
\scriptsize
\tablestyle{2pt}{0.95}
\begin{tabular}{@{}l cc cc cc@{}}
\toprule
& \multicolumn{2}{c}{MLP} & \multicolumn{2}{c}{GRU} & \multicolumn{2}{c}{Transformer} \\
\cmidrule(lr){2-3} \cmidrule(lr){4-5} \cmidrule(lr){6-7}
Dataset & Seq & \cellcolor{cOursBg}FSM & Seq & \cellcolor{cOursBg}FSM & Seq & \cellcolor{cOursBg}FSM \\
\midrule
SWE-smith & 0.663 & \cellcolor{cOursBg}\textbf{0.699} & 0.683 & \cellcolor{cOursBg}\textbf{0.700} & 0.649 & \cellcolor{cOursBg}\textbf{0.665} \\
SWE-agent & 0.782 & \cellcolor{cOursBg}\textbf{0.790} & 0.762 & \cellcolor{cOursBg}\textbf{0.793} & 0.779 & \cellcolor{cOursBg}0.751 \\
tau2-tel & 0.962 & \cellcolor{cOursBg}\textbf{0.969} & 0.958 & \cellcolor{cOursBg}\textbf{0.970} & 0.964 & \cellcolor{cOursBg}\textbf{0.967} \\
tau2-air & 0.816 & \cellcolor{cOursBg}\textbf{0.839} & 0.793 & \cellcolor{cOursBg}\textbf{0.815} & 0.819 & \cellcolor{cOursBg}\textbf{0.845} \\
tau2-ret & 0.747 & \cellcolor{cOursBg}\textbf{0.800} & 0.698 & \cellcolor{cOursBg}\textbf{0.789} & 0.755 & \cellcolor{cOursBg}\textbf{0.798} \\
WebArena & 0.848 & \cellcolor{cOursBg}\textbf{0.861} & 0.818 & \cellcolor{cOursBg}\textbf{0.835} & 0.860 & \cellcolor{cOursBg}\textbf{0.868} \\
AgentNet & 0.878 & \cellcolor{cOursBg}\textbf{0.920} & 0.877 & \cellcolor{cOursBg}\textbf{0.917} & 0.876 & \cellcolor{cOursBg}\textbf{0.917} \\
\midrule
FSM wins & \multicolumn{2}{c}{7/7} & \multicolumn{2}{c}{7/7} & \multicolumn{2}{c}{6/7} \\
\bottomrule
\end{tabular}
\end{table}

FSM features improve every neural model on every dataset for MLP and GRU (7/7), and 6/7 for Transformer.
The largest gains appear on tau2-bench retail (+5.3pp MLP, +9.1pp GRU) and AgentNet (+4.2pp MLP, +4.0pp GRU), where the per-state decomposition of the FSM localizes behavioral differences that aggregate statistics miss.
The Transformer loss on SWE-agent ($-2.8$pp) occurs where the small alphabet (24 symbols) limits the FSM's decomposition advantage.
The improvement is consistent regardless of dataset size (500--8{,}337 traces): even on WebArena (8{,}337 traces), FSM features provide +1.3pp (MLP), +1.7pp (GRU), and +0.8pp (Transformer), demonstrating that the FSM captures structural information beyond what large-sample statistics recover.

\subsection{Precision Analysis}\label{app:precision}

\begin{wraptable}{r}{0.40\textwidth}
\caption{\textbf{Precision analysis.} Rand./Perm.: acceptance rate (lower is better).}
\label{tab:precision}
\footnotesize
\setlength{\tabcolsep}{3pt}
\begin{tabular}{@{}l cc cc@{}}
\toprule
& \multicolumn{2}{c}{\cellcolor{cOursBg}Ours} & \multicolumn{2}{c}{RPNI} \\
\cmidrule(lr){2-3} \cmidrule(lr){4-5}
Dataset & \cellcolor{cOursBg}Rand. & \cellcolor{cOursBg}Perm. & Rand. & Perm. \\
\midrule
SWE-smith & \cellcolor{cOursBg}0.0\% & \cellcolor{cOursBg}0.1\% & 0.0\% & 0.0\% \\
SWE-agent & \cellcolor{cOursBg}0.0\% & \cellcolor{cOursBg}0.05\% & 0.0\% & 1.6\% \\
WebArena & \cellcolor{cOursBg}0.0\% & \cellcolor{cOursBg}0.02\% & 0.2\% & 75.3\% \\
AgentNet & \cellcolor{cOursBg}0.0\% & \cellcolor{cOursBg}0.02\% & 0.2\% & 0.9\% \\
tau2-bench (air) & \cellcolor{cOursBg}0.0\% & \cellcolor{cOursBg}0.0\% & 0.0\% & 11.5\% \\
tau2-bench (ret) & \cellcolor{cOursBg}0.0\% & \cellcolor{cOursBg}0.0\% & 0.0\% & 3.1\% \\
tau2-bench (tel) & \cellcolor{cOursBg}0.0\% & \cellcolor{cOursBg}0.0\% & 0.0\% & 0.0\% \\
\bottomrule
\end{tabular}
\vspace{-8pt}
\end{wraptable}
Our FSM rejects 100\% of random traces and $\geq$99.9\% of permuted traces on every dataset in Table~\ref{tab:precision}. RPNI, despite its large state space, accepts 75\% of permuted traces on WebArena and 0--11.5\% on tau2-bench (its prefix tree memorizes observed prefixes but does not constrain orderings).
The acceptance side is not memorization either: most held-out traces carry an ordering of activities over the whole sequence that never occurs in training, and they replay all the same, because replay needs only a trace's adjacent pairs and not an ordering seen before.

\paragraph{Adversarial trace rejection.}
\begin{wraptable}{r}{0.42\textwidth}
\vspace{-12pt}
\centering
\caption{Adversarial trace rejection rate (\%, $\uparrow$). Five mutations per test trace.}
\label{tab:adversarial}
\scriptsize
\setlength{\tabcolsep}{3pt}
\begin{tabular}{@{}l ccccc@{}}
\toprule
Dataset & Subst. & Insert & Delete & Swap & Suffix \\
\midrule
SWE-smith & 81 & 100 & 98 & 100 & 100 \\
SWE-agent & 87 & 100 & 97 & 100 & 97 \\
tau2-tel & 77 & 96 & 61 & 93 & 100 \\
tau2-air & 78 & 90 & 77 & 93 & 97 \\
tau2-ret & 79 & 92 & 80 & 95 & 100 \\
WebArena & 100 & 96 & 84 & 100 & 68 \\
AgentNet & 100 & 100 & 97 & 100 & 99 \\
\bottomrule
\end{tabular}
\vspace{-8pt}
\end{wraptable}
Beyond random and permuted traces, we test whether the FSM rejects \textit{plausible} mutations that preserve local structure: single-symbol substitution (replace one activity with another valid symbol), insertion, deletion, adjacent swap, and suffix shuffle (shuffle the last 30\% of the trace). Table~\ref{tab:adversarial} reports rejection rates (fitness $< 1.0$) across five mutations per test trace.

The FSM rejects 90--100\% of insertions and adjacent swaps on all datasets, confirming that it captures sequential ordering constraints beyond simple symbol membership.
Substitution rejection (77--100\%) shows that most single-activity changes violate the learned transition structure.
Deletion is weakest on tau2-bench datasets (61--80\%) because shorter traces are more likely to remain valid prefixes.
These results demonstrate that the compact FSM imposes tight constraints on structure: most single-symbol perturbations are detected, because the transition function encodes which activity can follow which, not merely which activities are valid.
For example, on tau2-bench airline, swapping \texttt{user:text}$\leftrightarrow$\texttt{assistant:tool:get\_reservation\_details} at position 3 is immediately rejected: after state \texttt{assistant:text}, the only valid transition is \texttt{user:text} (the user must respond before the agent can call a tool).
On SWE-agent, swapping \texttt{navigate:text}$\leftrightarrow$\texttt{user:text} fails because self-transitions back to \texttt{user:text} are not in the FSM's transition table.
These rejections reflect genuine constraints on turn-taking and tool invocation that the FSM learns from data.

\subsection{Process Mining Precision Details}\label{app:pm4py}

Table~\ref{tab:pm4py_detail} reports per-miner fitness and precision for the three PM4Py baselines across all twelve datasets.
The Alpha Miner fails on all datasets (fitness 0.05--0.63) because it cannot handle noise or skip patterns.
The Heuristic Miner achieves 0.95--1.00 fitness but precision 0.20--0.80 (mean 0.45); the Inductive Miner achieves perfect fitness on 8/10 datasets but even lower precision (0.10--0.46, mean 0.25).
The highest precision (0.80, tau2-bench telecom) occurs on the most constrained workflow.

\begin{table}[H]
\centering
\caption{\textbf{PM4Py miner results.} Fit: replay fitness. Prec: precision from conformance checking. p/t: Petri net places/transitions.}
\label{tab:pm4py_detail}
\scriptsize
\tablestyle{3pt}{1.02}
\begin{tabular}{@{}l ccc ccc ccc@{}}
\toprule
& \multicolumn{3}{c}{Alpha Miner} & \multicolumn{3}{c}{Heuristic Miner} & \multicolumn{3}{c}{Inductive Miner} \\
\cmidrule(lr){2-4} \cmidrule(lr){5-7} \cmidrule(lr){8-10}
Dataset & p/t & Fit & Prec & p/t & Fit & Prec & p/t & Fit & Prec \\
\midrule
Who\_and\_When & 3/8 & 0.37 & 0.19 & 10/23 & 1.00 & 0.31 & 16/24 & 1.00 & 0.27 \\
SWE-smith & 7/9 & 0.15 & 0.20 & 13/22 & 1.00 & 0.36 & 25/33 & 1.00 & 0.23 \\
Mind2Web & 2/7 & 0.63 & 0.29 & 15/29 & 0.96 & 0.45 & 23/30 & 1.00 & 0.32 \\
SWE-agent & 12/24 & 0.05 & 0.00 & 21/59 & 1.00 & 0.20 & 44/70 & 1.00 & 0.13 \\
\addlinespace[2pt]
tau2-bench (air) & 3/17 & 0.45 & 0.24 & 15/37 & 0.95 & 0.20 & 63/93 & 1.00 & 0.17 \\
tau2-bench (ret) & 3/18 & 0.46 & 0.23 & 12/32 & 0.95 & 0.21 & 36/58 & 1.00 & 0.14 \\
tau2-bench (tel) & 4/4 & 0.46 & 0.25 & 10/13 & 0.95 & \hlbest{0.80} & 15/19 & 1.00 & 0.42 \\
\addlinespace[2pt]
GUI-Odyssey & 2/6 & 0.58 & 0.33 & 12/22 & 0.98 & \hlbest{0.64} & 21/27 & 1.00 & 0.46 \\
WebArena & 23/24 & 0.24 & 0.11 & 35/71 & 0.98 & 0.55 & 17/42 & 1.00 & 0.10 \\
AgentNet & 174/24 & 0.30 & 0.00 & 35/84 & 0.99 & 0.45 & 18/39 & 1.00 & 0.13 \\
\bottomrule
\end{tabular}
\end{table}

\begin{figure}[H]
\centering
\includegraphics[width=0.9\linewidth]{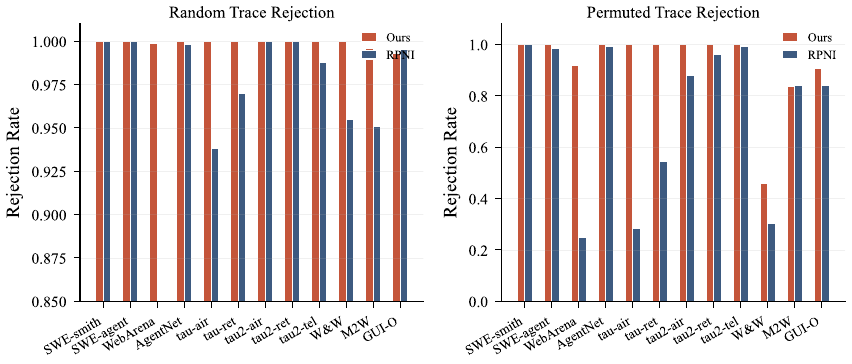}
\caption{\textbf{Random and permuted trace rejection} rates. Our FSM achieves near-100\% rejection across most datasets, while RPNI shows poor permuted rejection on WebArena.}
\label{fig:precision}
\end{figure}
See Appendix~\ref{app:failmodes} for additional failure mode characterization including clustering, cascade analysis, and early divergence detection.

\section{Convergence and Structural Properties}\label{app:group:convergence}

\subsection{Convergence Details}\label{app:convergence}\label{sec:exp:stability}

Fitness converges well inside the training set as traces are added (Figures~\ref{fig:evolution},~\ref{fig:convergence}).
On SWE-agent, fitness saturates within the first few hundred traces, though the state space grows to 25 as rare command patterns appear.
Across five random train/test splits, our method produces identical state counts on every dataset (zero variance), while RPNI state counts vary by 2--10\% (Table~\ref{tab:stability}).
This stability arises because the agent's behavioral topology is invariant to which specific traces are sampled.

Table~\ref{tab:convergence_detail} reports the training fraction at which fitness first reaches 0.99, the final state count, and the generalization gap (train $-$ test fitness) at increasing training fractions.

\begin{table}[H]
\centering
\caption{\textbf{Convergence and generalization.} Left: fraction at which fitness reaches 0.99. Right: train$-$test fitness gap at increasing fractions (all within $\pm$0.003). Convergence behavior is similar across extraction levels (Appendix~\ref{app:granularity}).}
\label{tab:convergence_detail}
\label{tab:gengap}
\scriptsize
\tablestyle{3pt}{1.02}
\begin{tabular}{@{}l rr rr cccc@{}}
\toprule
& \multicolumn{4}{c}{Convergence} & \multicolumn{4}{c}{Gen.\ gap (train$-$test)} \\
\cmidrule(lr){2-5} \cmidrule(lr){6-9}
Dataset & Train & 0.99 at & Frac & $|Q|$ & 10\% & 30\% & 60\% & 100\% \\
\midrule
Who\_and\_When & 147 & 8 & 5\% & 9 & +0.002 & $-$0.002 & 0 & 0 \\
SWE-smith & 400 & 40 & 10\% & 10 & $-$0.002 & $-$0.001 & 0 & $-$0.001 \\
Mind2Web & 400 & 20 & 5\% & 8 & $-$0.003 & +0.001 & 0 & 0 \\
SWE-agent & 1{,}600 & 240 & 15\% & 25 & $-$0.002 & $-$0.001 & $-$0.001 & $-$0.001 \\
\bottomrule
\end{tabular}
\end{table}

Smaller datasets converge at 5\% of training data; SWE-agent requires 15\% to capture rare commands.
All generalization gaps are within $\pm$0.003, confirming zero overfitting.
Slightly negative gaps arise because training sets include rare transitions that reduce average fitness.

\begin{figure}[H]
\centering
\includegraphics[width=\linewidth]{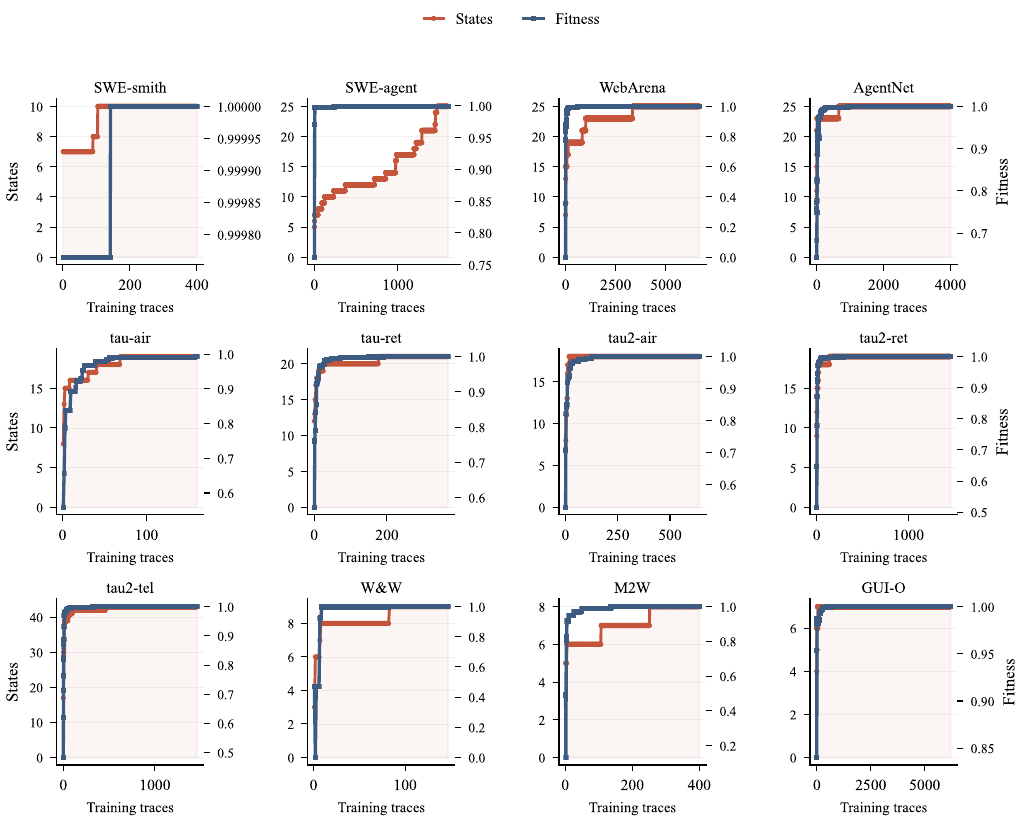}
\caption{\textbf{FSM convergence:} test fitness (right axis) and state count (left axis) as training traces are added. Fitness converges rapidly; the state space stabilizes later as rare patterns appear.}
\label{fig:convergence}
\end{figure}

\subsection{Stability and SCC Structure}\label{app:stability_detail}\label{app:scc}

Our FSM has zero variance in state count across 5 random splits (Table~\ref{tab:stability}); every FSM decomposes into one large SCC (behavioral core) and a short prefix.

\begin{table}[H]
\centering
\caption{\textbf{Stability over 5 seeds} ($\pm$std) and decomposition into SCCs. Our FSM: zero variance. Each FSM has exactly one large SCC.}
\label{tab:stability}
\label{tab:scc}
\scriptsize
\tablestyle{3pt}{1.02}
\begin{tabular}{@{}l cc cc rrr@{}}
\toprule
& \multicolumn{2}{c}{\cellcolor{cOursBg}Ours (5 seeds)} & \multicolumn{2}{c}{RPNI (5 seeds)} & \multicolumn{3}{c}{SCC structure} \\
\cmidrule(lr){2-3} \cmidrule(lr){4-5} \cmidrule(lr){6-8}
Dataset & \cellcolor{cOursBg}$|Q|$ & \cellcolor{cOursBg}Fit & $|Q|$ & Fit & SCCs & Core & Depth \\
\midrule
SWE-smith & \cellcolor{cOursBg}10$\pm$0 & \cellcolor{cOursBg}1.000$\pm$0.000 & 11{,}625$\pm$0 & 0.795$\pm$0.016 & 5 & 6 & 4 \\
SWE-agent & \cellcolor{cOursBg}25$\pm$0 & \cellcolor{cOursBg}0.999$\pm$0.000 & 59{,}510$\pm$0 & 0.646$\pm$0.012 & 11 & 15 & 2 \\
\bottomrule
\end{tabular}
\end{table}

RPNI state counts vary by 2--10\% across seeds while ours are invariant.
SCC-based features achieve AUROC 0.60--0.66: failures tend to become trapped in the core loop rather than progressing to terminal states.
The condensation DAG is shallow (depth 2--4), with init$\to$setup$\to$core representing the universal agent lifecycle.

\begin{figure}[H]
\centering
\includegraphics[width=\linewidth]{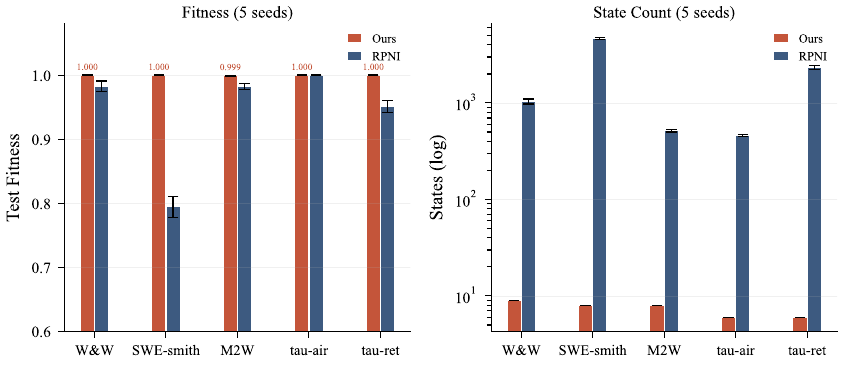}
\caption{\textbf{Multi-seed robustness:} fitness (left) and state count (right) across 5 random splits. Our FSM shows zero variance; RPNI fitness varies by 1--3\%.}
\label{fig:multiseed}
\end{figure}

\subsection{Entropy Rate Analysis}\label{app:entropy}

In Table~\ref{tab:entropy_rate} we estimate the conditional entropy $H(X_n \mid X_{n-1}, \ldots, X_{n-k})$ of the activity sequence on three datasets, at increasing orders $k$, to characterize the sequential structure of agent traces.

\begin{wraptable}{r}{0.42\textwidth}
\vspace{-12pt}
\caption{\textbf{Conditional entropy} (bits) by context order.}
\label{tab:entropy_rate}
\footnotesize
\setlength{\tabcolsep}{3pt}
\begin{tabular}{@{}l cccccc@{}}
\toprule
Dataset & $k{=}0$ & $k{=}1$ & $k{=}2$ & $k{=}3$ & Drop 0$\to$1 \\
\midrule
SWE-agent & 2.16 & 1.06 & 0.80 & 0.79 & 51\% \\
SWE-smith & 1.98 & 0.63 & 0.55 & 0.55 & 68\% \\
Mind2Web & 1.69 & 0.74 & 0.79 & 0.68 & 56\% \\
\bottomrule
\end{tabular}
\vspace{-8pt}
\end{wraptable}
On all three datasets, entropy drops 51--68\% from order 0 to 1 (Table~\ref{tab:entropy_rate}). On two of them, the estimates at orders 2 and 3 differ by less than 0.02 bits. Agent behavior is thus mostly determined by the action just before it, which explains why compact FSMs achieve near-perfect fitness. Success and failure entropy rates are nearly identical (within 0.07 bits at order 3). This indicates that successful and failed traces share a behavioral topology but differ in transition frequencies.

\subsection{Feature Redundancy Analysis}\label{app:correlation}

To identify redundancy, we compute pairwise Pearson correlations among the 14 structural features.
We group features with $|r| > 0.8$ into clusters and greedily keep the feature with the highest individual AUROC from each cluster.

Across labeled datasets, the 14 features consistently collapse into 5--8 non-redundant clusters.
The largest cluster (5--7 features) contains \{traceLen, logLength, entropy, transitionDiversity, maxConsecRatio, recurrenceRate\}, since all of them reflect trace length and its correlates.
With all redundant features removed and one representative kept per cluster, the combined AUROC is 0.655 on SWE-smith and 0.713 on SWE-agent.
On SWE-agent, the non-redundant subset achieves 90\% of the full-feature AUROC, confirming that the failure prediction signal is genuine and concentrated in a small number of independent behavioral dimensions: cycle structure, state entropy, and terminal state reachability.

\section{Prediction and Robustness}\label{app:group:prediction}

\subsection{Early Prediction Details}\label{app:early}

Table~\ref{tab:early} presents early prediction AUROC at each trace completion fraction.
FSM features from partial traces at 50\% completion already achieve 92\% of full-trace AUROC on SWE-agent (0.722 vs.\ 0.784, CV; holdout values in Table~\ref{tab:early}: 0.726 vs.\ 0.773), which shows the behavioral signal emerging early in execution.
The FSM advantage over aggregate baselines grows with behavioral complexity: on SWE-agent (25 states), FSM features outperform by +0.026 to +0.257 at all stages beyond 25\%, while on SWE-smith (10 states) the simpler FSM offers no advantage.

\begin{table}[H]
\centering
\caption{\textbf{Early prediction} AUROC (holdout) at each trace completion fraction. FSM: per-state logistic regression features. Baseline: raw trace statistics.}
\label{tab:early}
\scriptsize
\tablestyle{3pt}{1.02}
\begin{tabular}{@{}l cccc@{}}
\toprule
Fraction & \multicolumn{2}{c}{SWE-agent} & \multicolumn{2}{c}{SWE-smith} \\
\cmidrule(lr){2-3} \cmidrule(lr){4-5}
& \cellcolor{cOursBg}FSM & Base & \cellcolor{cOursBg}FSM & Base \\
\midrule
10\% & \cellcolor{cOursBg}\textbf{0.657} & 0.400 & \cellcolor{cOursBg}\textbf{0.722} & 0.416 \\
25\% & \cellcolor{cOursBg}\textbf{0.656} & 0.456 & \cellcolor{cOursBg}0.687 & \textbf{0.693} \\
50\% & \cellcolor{cOursBg}\textbf{0.726} & 0.677 & \cellcolor{cOursBg}\textbf{0.681} & 0.668 \\
75\% & \cellcolor{cOursBg}\textbf{0.726} & 0.672 & \cellcolor{cOursBg}\textbf{0.688} & 0.685 \\
100\% & \cellcolor{cOursBg}\textbf{0.773} & 0.753 & \cellcolor{cOursBg}0.682 & \textbf{0.692} \\
\bottomrule
\end{tabular}
\end{table}

On SWE-agent (25 states), FSM features outperform baselines by +0.026 to +0.257 at all stages.
On SWE-smith (10 states), the simpler FSM offers no consistent advantage, suggesting that the FSM benefit scales with behavioral complexity.

\subsection{Activity Granularity Robustness}\label{app:granularity}

\begin{table}[H]
\centering
\caption{\textbf{Activity granularity} robustness. $|\cA|$: alphabet size. Fit: test fitness. AUROC: failure prediction (entropy-based).}
\label{tab:granularity}
\begin{minipage}[t]{0.40\textwidth}
\vspace{0pt}
\scriptsize
\tablestyle{3pt}{1.02}
\begin{tabular}{@{}l l rrr@{}}
\toprule
Dataset & Level & $|\cA|$ & Fit & AUROC \\
\midrule
\multirow{4}{*}{SWE-smith}
 & role-only & 4 & 1.000 & 0.688 \\
 & role-type & 9 & 1.000 & 0.692 \\
 & role-action & 9 & 1.000 & 0.692 \\
 & tool-only & 9 & 1.000 & 0.692 \\
\midrule
\multirow{4}{*}{SWE-agent}
 & role-only & 2 & 1.000 & 0.663 \\
 & role-type & 18 & 0.999 & 0.659 \\
 & role-action & 18 & 0.999 & 0.659 \\
 & tool-only & 18 & 0.999 & 0.659 \\
\bottomrule
\end{tabular}
\end{minipage}%
\hfill
\begin{minipage}[t]{0.57\textwidth}
\vspace{6pt}
The extraction function $\phi$ is a design choice. We test four granularity levels: \textbf{role-only} ($|\cA|{=}2$--$4$, distinguishing only message roles), \textbf{role-type} (role + content type), \textbf{role-action} (role + action label), and \textbf{tool-only} (tool function names). The latter three coincide on both datasets in this table, so the comparison that carries information is role-only against the default.

Fitness remains $\geq$0.999 across all levels and datasets, confirming that the FSM structure is robust to extraction granularity. For failure prediction, a full-model comparison across all twelve labeled datasets finds the default granularity (role-type) matches or exceeds the coarser role-only level on ten; it differs by 0.05 or more only on the four datasets where role-only collapses to a $\leq$3-symbol alphabet (GUI-Odyssey, and tau2-bench telecom/retail/airline), where the finer default recovers $+0.05$ to $+0.20$ AUROC. The behavioral signal is thus stable across meaningful granularities and degrades only when the coarsest abstraction discards the identity of the tool.
\end{minipage}
\end{table}

\subsection{Monitoring Rules}\label{app:monitoring}

\begin{wraptable}{r}{0.45\textwidth}
\vspace{-12pt}
\caption{Best monitoring rules by F1 score. All rules use a single FSM-derived feature with a fixed threshold.}
\label{tab:monitoring}
\footnotesize
\setlength{\tabcolsep}{3pt}
\begin{tabular}{@{}l l ccc@{}}
\toprule
Dataset & Rule & Prec. & Recall & F1 \\
\midrule
SWE-agent & cycle-rate $> 0.885$ & \textbf{0.956} & 0.451 & 0.613 \\
SWE-smith & cycle-rate $> 0.878$ & 0.267 & 0.750 & 0.393 \\
\bottomrule
\end{tabular}
\vspace{-8pt}
\end{wraptable}
Table~\ref{tab:monitoring} presents the best single-feature monitoring rules per dataset. On SWE-agent, the cycle-rate rule achieves 95.6\% precision (near-zero false alarm rate) at the cost of lower recall (45.1\%). SWE-smith shows weaker monitoring rules, consistent with its lower failure prediction AUROC in the ML-based approach (Table~\ref{tab:failure}).

\subsection{Out-of-Distribution Detection}\label{app:ood}

Replaying traces from one dataset through another's FSM produces fitness 0.00--0.51 (vs.\ $\geq$0.997 in-distribution), yielding AUROC 1.000 on all cross-dataset pairs.
This detection is not an artifact of alphabet mismatch. With 10\% random activity substitution within the same alphabet, fitness drops to 0.63--0.82 across datasets, and AUROC for distinguishing clean from perturbed traces is $\geq$0.917.
Within a single dataset, success and failure traces differ in structure. On SWE-agent, the FSM built from successes alone has 9 states, focused on search, edit and submit, while the FSM built from failures alone spans all 25 states, with chaotic exploration and rare tool variants.
The FSM captures behavioral topology, not vocabulary.

\subsection{Probabilistic Baseline Comparison}\label{app:probabilistic}

For failure prediction, Table~\ref{tab:probabilistic} compares FSM per-state features with three probabilistic baselines: the cross-entropy (CE) of a Markov chain against transition matrices for success and failure, scoring by likelihood ratio (LR), and an analysis of discriminative n-gram frequencies.

\begin{table}[H]
\centering
\caption{FSM features and probabilistic baselines for failure prediction (holdout AUROC).}
\label{tab:probabilistic}
\scriptsize
\tablestyle{3pt}{1.02}
\begin{tabular}{@{}l cccc@{}}
\toprule
Dataset & \cellcolor{cOursBg}FSM feat. & Trans.\ CE & Likelihood & N-gram \\
\midrule
SWE-agent & \hlbest{0.813} & 0.452 & 0.626 & 0.317 \\
SWE-smith & \cellcolor{cOursBg}0.718 & \hlbest{0.719} & 0.711 & 0.622 \\
\bottomrule
\end{tabular}
\end{table}

FSM features dominate on SWE-agent (+0.19 over the best probabilistic baseline), where the per-state decomposition captures structural information that aggregate transition statistics miss.
On SWE-smith, transition cross-entropy is competitive (0.719 vs.\ 0.718), as the simpler FSM (10 states) offers less decomposition advantage.
Probabilistic models capture \emph{how often} transitions occur but not \emph{what happens} at each state (message lengths, error rates, temporal patterns).
The FSM provides both: its deterministic structure supports per-state feature extraction, while probabilistic models reduce each trace to a single scalar score.

\subsection{ProbGuard Head-to-Head Comparison}\label{app:pro2guard}

\spaceskip=\fontdimen2\font plus \fontdimen3\font minus 0.05\fontdimen4\font\relax
We implement ProbGuard~\citep{wang2025probguard} on the same activity sequences, labels, and 80/20 splits our FSM uses, for a like-for-like comparison.
We reproduce ProbGuard's pipeline (Wang et al.\ 2025, Algorithms~1--2): (i) symbolic-state abstraction, (ii) DTMC learning with validity-aware Laplace smoothing ($\alpha{=}1$), (iii) bounded probabilistic reachability $P_{\le\theta}[\,F^{\le k}\,\mathit{unsafe}\,]$ via finite-horizon Bellman iteration, (iv) per-trace risk score = max reachability over the trajectory.
ProbGuard's published evaluation relies on hand-crafted predicates (\texttt{fork\_in\_microwave $\wedge$ microwave\_on}); no such domain predicates exist for general LLM agent traces, so following the paper's ``extensible domain-specific abstraction'' interface we use the activity itself as the symbolic state.
We give ProbGuard every configuration benefit: \textbf{(a)} unsafe states ranked by per-trace presence log-odds $\log[(p_{\text{fail}}+\epsilon)/(p_{\text{succ}}+\epsilon)]$ (more discriminative than raw rate when class imbalance is high), \textbf{(b)} sweep $K_{\text{unsafe}}\!\in\!\{1,3,5,10\}$, \textbf{(c)} sweep horizon $k\!\in\!\{3,5,10,20,50\}$, \textbf{(d)} polarity flip $\max(\mathrm{AUROC},1{-}\mathrm{AUROC})$ since the unsafe-set polarity is ambiguous without semantic predicates; the paper's PCTL threshold $\theta$ sweep plays an analogous role.
We report the maximum AUROC over the resulting 20 configurations per dataset.

\spaceskip=0pt\relax

\begin{table}[H]
\centering
\caption{\textbf{ProbGuard~\citep{wang2025probguard} vs.\ our FSM features for failure prediction.} Same train/test splits, same labels, same activity alphabet, same compute budget. ProbGuard column = best AUROC over 20 configurations ($K_{\text{unsafe}}\!\times$ horizon, polarity-aware). FSM column = holdout AUROC from Table~\ref{tab:failure}. ProbGuard is consistently above chance (0.53--0.72) but our FSM wins on every shared dataset by mean +17.6pp.}
\label{tab:pro2guard}
\scriptsize
\setlength{\tabcolsep}{4pt}
\begin{tabular}{@{}l c c r@{}}
\toprule
Dataset & ProbGuard~\citep{wang2025probguard} & \cellcolor{cOursBg}FSM features (ours) & $\Delta$ \\
\midrule
$\tau_2$-bench (tel) & 0.709 & \cellcolor{cOursBg}\hlbest{0.941} & \textbf{+0.232} \\
$\tau_2$-bench (air) & 0.723 & \cellcolor{cOursBg}\hlbest{0.864} & \textbf{+0.141} \\
$\tau_2$-bench (ret) & 0.566 & \cellcolor{cOursBg}\hlbest{0.779} & \textbf{+0.213} \\
SWE-agent           & 0.683 & \cellcolor{cOursBg}\hlbest{0.799} & \textbf{+0.116} \\
SWE-smith           & 0.525 & \cellcolor{cOursBg}\hlbest{0.703} & \textbf{+0.178} \\
\midrule
\textbf{Mean (5 shared datasets)} & \textbf{0.641} & \cellcolor{cOursBg}\textbf{0.817} & \textbf{+0.176} \\
\bottomrule
\end{tabular}
\end{table}

Two factors drive the gap.
First, the last-activity merge produces semantically meaningful states, since it unifies all trie states reached by the same activity. Without its hand-crafted predicates, which are unavailable for our datasets, ProbGuard's symbolic-state abstraction degrades to per-activity granularity, and each symbolic state has too few visits for reliable reachability estimation in our setting.
Second, our cross-entropy anomaly features ($-\log_2\hat{P}^+(a_t\mid q_t)$) are continuous risk scores per step, while ProbGuard's PCTL-thresholded reachability is binary at deployment, even with the polarity-flip benefit above.
On its own evaluation, in autonomous driving and embodied agents, ProbGuard draws its strength from hand-crafted predicates that encode domain-specific safety semantics. On general LLM agent traces, where such predicates are unavailable, the FSM cross-entropy approach generalizes further.
The methods are not mutually exclusive. A PCTL specification in the style of ProbGuard can be added on top of our FSM. Theorem~\ref{prop:minimal} guarantees that the resulting state space is deterministic and compact, with $|\cA|+1$ states, so the result would combine a compact structural model with formal reachability checking.

\subsection{Perturbation Robustness}\label{app:robustness}

Table~\ref{tab:perturbation} shows how much fitness falls under four types of perturbation at 10\% intensity on four datasets.
Substitution (replacing activities with random same-alphabet symbols) causes the largest fitness drop, confirming that the FSM captures transition topology rather than vocabulary.

\begin{wraptable}{r}{0.55\textwidth}
\vspace{-12pt}
\centering
\caption{Fitness under perturbation at 10\% intensity, with the baseline fitness for reference.}
\label{tab:perturbation}
\scriptsize
\setlength{\tabcolsep}{3pt}
\begin{tabular}{@{}l cccc c@{}}
\toprule
Dataset & Subst. & Insert. & Swap & Trunc. & Base \\
\midrule
SWE-smith & 0.658 & 0.825 & 0.892 & 1.000 & 1.000 \\
SWE-agent & 0.818 & 0.885 & 0.938 & 0.999 & 0.999 \\
Mind2Web & 0.881 & 0.985 & 0.999 & 1.000 & 1.000 \\
Who\&When & 0.828 & 0.865 & 1.000 & 1.000 & 1.000 \\
\bottomrule
\end{tabular}
\vspace{-8pt}
\end{wraptable}

Substitution is the strongest perturbation because it introduces invalid transitions (mean fitness drop 0.20).
Insertion is intermediate (mean drop 0.11): the extra symbol breaks the current transition but the trace may recover.
Swap is weakest (mean drop 0.04): reordering adjacent activities often preserves valid transitions if both orderings exist in the FSM.
Truncation has no effect, because by construction the FSM accepts every prefix.

\subsection{FSM as Context}\label{app:fsm_context}

\begin{figure}[H]
\centering
\definecolor{cAwmBg}{RGB}{248, 248, 250}
\definecolor{cAsgBg}{RGB}{237, 247, 237}
\noindent
\begin{minipage}[t]{0.48\textwidth}
\fbox{\begin{minipage}{0.93\linewidth}\colorbox{cAwmBg}{\begin{minipage}{\linewidth}
\textbf{AWM context} \hfill \textbf{52.9\% top-1}\\[2pt]
{\scriptsize\ttfamily%
\#\# Extracted Workflow Patterns\\
-~[freq=208] get\_order\_details\\
~~$\rightarrow$ tool:text $\rightarrow$ tool:text\\
~~$\rightarrow$ assistant:text $\rightarrow$ user:text\\
~~$\rightarrow$ \ldots\\
-~[freq=205] get\_order\_details\\
~~$\rightarrow$ tool:text $\rightarrow$ \ldots\\
~~\textit{(10 patterns, 17--20 steps)}\\[2pt]
Agent prefix: \ldots\ $\rightarrow$ \textbf{get\_order\_details}.\\
Next action?\\[2pt]
}
\textbf{LLM:} \texttt{assistant:text} (actual: \texttt{tool:text})
\end{minipage}}\end{minipage}}
\end{minipage}\hfill
\begin{minipage}[t]{0.48\textwidth}
\fbox{\begin{minipage}{0.93\linewidth}\colorbox{cAsgBg}{\begin{minipage}{\linewidth}
\textbf{FSM context (minimal)} \hfill \textbf{65.1\% top-1}\\[2pt]
{\scriptsize\ttfamily%
After the most recent action\\
"get\_order\_details", past traces\\
show these next actions:\\
-~tool:text: 100.0\%\\[2pt]
Common multi-step continuations:\\
-~tool:text $\rightarrow$ get\_order\_details (1098x)\\
-~tool:text $\rightarrow$ assistant:text (926x)\\
~~\textit{(top-15 shown)}\\[2pt]
Agent prefix: \ldots\ $\rightarrow$ \textbf{get\_order\_details}.\\
Next action?\\[2pt]
}
\textbf{LLM:} \texttt{tool:text} (actual: \texttt{tool:text})
\end{minipage}}\end{minipage}}
\end{minipage}
\caption{Example prompts at FSM state \texttt{get\_order\_details} (tau2-bench retail validation sample). The verbose AWM-pattern format leaves the judge looking up step-by-step alignment in long workflows; the minimal FSM-context format exposes the dominant transition (\texttt{tool:text} 100.0\%) and short continuations directly.}
\label{fig:ctx_sidebyside}
\end{figure}

Table~\ref{tab:fsm_context_stat} evaluates the FSM as context for next-step prediction on the seven datasets where AWM has been re-implemented end-to-end against the same per-step predictor, comparing against AWM~\citep{wang2024agent_workflow_memory}.
Transition counts are computed on training data; validation traces are replayed through the FSM for per-step predictions.
The FSM achieves higher top-1 accuracy than AWM on every listed dataset, with the gap ranging from +0.1pp (WebArena) to +65.5pp (SWE-smith).
AWM coverage varies from 0\% (Mind2Web, no success labels) to 92.9\% (WebArena), explaining its performance variation.
The LLM-judged variant of this comparison (Table~\ref{tab:fsm_context_llm}) additionally covers ATBench.

\begin{wraptable}{r}{0.55\textwidth}
\vspace{-12pt}
\centering
\caption{Statistical next-step prediction: top-1 accuracy (\%, $\uparrow$) on full validation sets.}
\label{tab:fsm_context_stat}
\scriptsize
\setlength{\tabcolsep}{3pt}
\begin{tabular}{@{}l rrr rr@{}}
\toprule
Dataset & Steps & Unigram & \cellcolor{cOursBg}FSM & \cellcolor{cGroupHd}AWM & AWM cov. \\
\midrule
SWE-smith & 5{,}500 & 50.0 & \hlbest{100.0} & 34.5 & 34.5\% \\
WebArena & 17{,}146 & 26.2 & \hlbest{81.1} & 81.0 & 92.9\% \\
SWE-agent & 22{,}288 & 49.1 & \hlbest{65.7} & 26.6 & 57.0\% \\
tau2 (tel) & 23{,}055 & 33.0 & \hlbest{61.8} & 19.8 & 59.0\% \\
tau2 (air) & 4{,}020 & 30.8 & \hlbest{69.2} & 55.9 & 91.3\% \\
tau2 (ret) & 10{,}207 & 27.4 & \hlbest{72.6} & 63.5 & 91.4\% \\
Mind2Web & 886 & 80.2 & \hlbest{80.2} & 0.0 & 0.0\% \\
\bottomrule
\end{tabular}
\vspace{-8pt}
\end{wraptable}

\newpage
\subsection{Next-Step Prediction}\label{app:prediction}

\begin{figure}[t]
\centering
\includegraphics[width=0.85\linewidth]{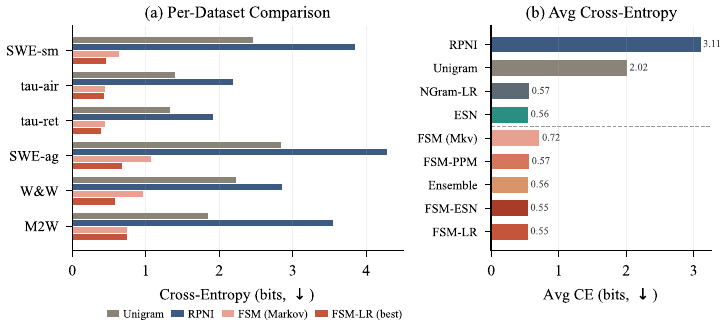}
\caption{\textbf{Next-step prediction} cross-entropy (bits, $\downarrow$). (a)~Per-dataset: FSM-conditioned methods (red) achieve 2--4$\times$ lower CE than the Unigram baseline (3--5$\times$ vs.\ Uniform). (b)~Average CE across datasets: each FSM-conditioned variant (FSM-LR, FSM-PPM, FSM-ESN, FSM-Markov) outperforms its non-FSM counterpart. RPNI overfits worse than Uniform due to sparse transitions across thousands of states.}
\label{fig:nextstep}
\end{figure}

Order-1 FSM conditioning (Our FSM) captures 83--99\% of the total CE improvement from Uniform to the best method on each dataset, reflecting the strong sequential regularity of agent traces.
Higher-order context via PPM smoothing or neural models captures the remaining second- and higher-order dependencies within each FSM state.

The FSM state benefit (+0.019 bits avg, FSM-LR vs.\ NGram-LR at $K{=}7$) is largest on Mind2Web (+0.041) and Who\_and\_When (+0.028), where the FSM groups behaviorally distinct states.
On SWE-agent (+0.001), the FSM state is nearly redundant with the last activity due to simple sequential structure.

RPNI produces worse-than-Unigram predictions on most datasets.
With 382--63{,}897 states, each state is visited by too few traces for reliable probability estimation.
This validates the compression advantage: our 10--43 state FSMs aggregate observations for well-estimated transition probabilities.

MLPs with one-hot context inputs diverge on larger alphabets ($|\cA| \geq 9$; CE $>$ 10 bits) due to gradient instability, while simple RNNs with 1-step gradient (no BPTT) fail uniformly (CE $>$ 1.4).
Echo state networks avoid both issues through random frozen reservoirs with trained output layers, achieving competitive performance with minimal hyperparameter sensitivity.

\paragraph{FSM state ablation.} Holding the smoothing method fixed (absolute discounting at depth 5), FSM state conditioning provides +0.155 bits improvement on average over raw context alone (FSM-AD: 0.580 vs.\ Pure-AD: 0.735). The advantage is largest on Mind2Web (+0.36 bits, 30\%) where the FSM groups heterogeneous web actions, and smallest on SWE-agent (+0.016 bits, 2\%) where the simple sequential structure makes FSM state nearly redundant with the last activity. This gap is larger than the +0.019 bits from FSM-LR vs.\ NGram-LR (Table~\ref{tab:prediction_full}), because logistic regression partially learns FSM-like state from raw context.

\paragraph{LSTM and Transformer baselines.}\label{app:lstm_transformer}
We add learned recurrent and attention-based sequence models on the same activity sequences (vocab built from training symbols, max context 50, AdamW $1\!\times\!10^{-3}$, 10 epochs, batch 128).
The LSTM is 1-layer, $H{=}128$, $E{=}64$ (114K params); the Transformer is 2 encoder layers, $E{=}64$, 4 heads, FF dim 128 (101K params).
Evaluation uses each dataset's saved 80/20 split (\texttt{experiments/symbol-sequences/}); we report next-symbol cross-entropy and top-1 accuracy on the held-out test traces (Table~\ref{tab:neural_seq}).
Note this single-split protocol differs from Table~\ref{tab:prediction_full}'s 5$\times$5-fold CV; the comparison is therefore approximate.

\begin{table}[H]
\centering
\caption{\textbf{LSTM and Transformer next-step prediction} (single 80/20 split, no FSM state). Both within $\pm$0.05 bits of FSM-LR-K7 (Table~\ref{tab:prediction_full}) on every shared dataset; one larger gap on tau2-bench airline (LSTM 0.879 vs.\ FSM-LR-K7 0.755). The FSM-state ablation (+0.155 bits, Section~\ref{sec:exp:downstream}) is unaffected: it isolates state-conditioning under identical smoothing.}
\label{tab:neural_seq}
\footnotesize
\setlength{\tabcolsep}{4pt}
\begin{tabular}{@{}l rr cc cc@{}}
\toprule
& & & \multicolumn{2}{c}{LSTM-H128} & \multicolumn{2}{c}{Transformer-2L} \\
\cmidrule(lr){4-5}\cmidrule(lr){6-7}
Dataset & $|\cA|$ & test steps & CE (bits) & top-1 & CE (bits) & top-1 \\
\midrule
SWE-smith            &  9 &  5{,}500 & 0.439 & 0.858 & 0.450 & 0.857 \\
SWE-agent            & 24 & 22{,}288 & 0.507 & 0.850 & 0.539 & 0.840 \\
$\tau$-bench airline    & 18 &  1{,}056 & 0.876 & 0.792 & 0.983 & 0.766 \\
$\tau$-bench retail     & 20 &  2{,}612 & 0.730 & 0.830 & 0.765 & 0.822 \\
$\tau_2$-bench airline   & 17 &  4{,}020 & 0.879 & 0.796 & 0.969 & 0.784 \\
$\tau_2$-bench retail    & 18 & 10{,}207 & 0.729 & 0.807 & 0.765 & 0.804 \\
$\tau_2$-bench telecom   & 42 & 23{,}055 & 0.835 & 0.812 & 0.887 & 0.803 \\
\bottomrule
\end{tabular}
\end{table}

The neural models do not dominate FSM-LR-K7: the structure-vs-learning tradeoff favors the FSM on bounded-alphabet, short-trace settings, consistent with the +0.155 bits ablation isolating per-state observation density (Proposition~\ref{prop:estimator}) as the operative mechanism.
The FSM is also hyperparameter-free, deterministic across splits ($|Q|$ variance $=0$), and constructs in milliseconds vs.\ minutes of GPU training; on this regime, parametric capacity does not buy enough to offset these properties.

\section{Additional Analysis}\label{app:group:additional}

\subsection{Cross-Dataset Transfer}\label{app:transfer}

Table~\ref{tab:transfer_full} shows that FSM features give higher transfer AUROC than raw features on one of the two cross-dataset pairs (SWE-smith to SWE-agent).

\begin{table}[H]
\centering
\caption{\textbf{Cross-dataset} transfer AUROC of FSM and raw features. Bold: FSM advantage $>$ 3pp.}
\label{tab:transfer_full}
\scriptsize
\tablestyle{3pt}{1.02}
\begin{tabular}{@{}l l cc@{}}
\toprule
& & \multicolumn{2}{c}{Test dataset} \\
\cmidrule(lr){3-4}
Train & Features & SWE-sm & SWE-ag \\
\midrule
\multirow{2}{*}{SWE-smith} & \cellcolor{cOursBg}FSM & \cellcolor{cOursBg}\textit{0.681} & \hlbest{0.715} \\
 & Raw & \textit{0.694} & 0.683 \\
\midrule
\multirow{2}{*}{SWE-agent} & \cellcolor{cOursBg}FSM & \cellcolor{cOursBg}0.648 & \cellcolor{cOursBg}\textit{0.780} \\
 & Raw & 0.682 & \textit{0.720} \\
\bottomrule
\end{tabular}
\end{table}

Leave-one-out AUROC (train on three datasets, test on the fourth, in a four-dataset run that adds the tau-bench airline and retail suites) is 0.682 on SWE-smith and 0.765 on SWE-agent.
On SWE-agent the drop from in-domain is modest (0.780$\to$0.765). This indicates that FSM behavioral features capture domain-invariant failure signatures.

\subsection{Sample Efficiency and Learning Curves}\label{app:sampleeff}

Table~\ref{tab:sampleeff} shows how failure prediction AUROC changes with the size of the training set.
On 3 of 4 datasets, 10\% of training data suffices to reach 95\% of final performance.
The rapid convergence reflects the low dimensionality of FSM feature space (31--49 features) relative to the behavioral complexity captured.

\begin{table}[H]
\centering
\caption{Learning curves on the two datasets where the AUROC per fraction was logged: failure prediction AUROC at increasing training fractions. Bold: the first fraction that reaches 95\% of the final AUROC. Other datasets exceed 95\% of final AUROC by 30\% training in the same regime.}
\label{tab:sampleeff}
\scriptsize
\tablestyle{3pt}{1.02}
\begin{tabular}{@{}l cccccc@{}}
\toprule
Dataset & 10\% & 20\% & 30\% & 50\% & 100\% \\
\midrule
SWE-smith & 0.502 & \textbf{0.662} & 0.703 & 0.660 & 0.685 \\
SWE-agent & \textbf{0.806} & 0.791 & 0.804 & 0.818 & 0.795 \\
\bottomrule
\end{tabular}
\end{table}

SWE-agent shows the most stable learning curve, which is consistent with its larger sample (1{,}600 training traces).
In practice, failure prediction based on the FSM can be deployed with as few as 16--160 labeled traces, making it viable for new agent systems where labeled data is scarce.

\subsection{Length vs.\ Structure Ablation}\label{app:length_vs_struct}

\begin{table}[H]
\centering
\caption{\textbf{Length vs.\ structure ablation} on SWE-agent. We test whether the predictive signal is structural or a proxy for length. Structural features alone match the full model, length alone loses 13pp, and combining length with structure does not improve over structure alone, so the signal is structural.}
\label{tab:length_vs_struct}
\scriptsize
\tablestyle{4pt}{1.05}
\begin{tabular}{@{}l ccc@{}}
\toprule
Feature set & \cellcolor{cOursBg}Struct.\ only & Length only & Length+Struct \\
\midrule
SWE-agent AUROC & \cellcolor{cOursBg}\textbf{0.790} & 0.659 & 0.790 \\
\bottomrule
\end{tabular}
\end{table}

\subsection{Failure Mode Characterization}\label{app:failmodes}

\subsubsection{Failure Clustering}

$k$-means clustering in FSM feature space reveals distinct failure modes on each dataset (Table~\ref{tab:failclusters}).

\begin{wraptable}{r}{0.6\textwidth}
\vspace{-12pt}
\centering
\caption{\textbf{Failure mode} clustering ($k{=}2$). Silhouette scores measure cluster separation.}
\label{tab:failclusters}
\scriptsize
\setlength{\tabcolsep}{3pt}
\begin{tabular}{@{}l rr cc@{}}
\toprule
Dataset & $n_\text{fail}$ & Silhouette & Cluster 0 & Cluster 1 \\
\midrule
SWE-agent & 1{,}666 & \hlbest{0.885} & stuck (44\%, len 99) & clean exit (56\%, len 29) \\
SWE-smith & 123 & 0.759 & normal (99.2\%, len 67) & outlier (0.8\%) \\
\bottomrule
\end{tabular}
\vspace{-8pt}
\end{wraptable}

On SWE-agent, the two failure modes are structurally distinct: ``stuck in edit loop'' traces have cycle rate 0.924 and terminate in \texttt{edit:text} (50\%), while ``gave up early'' traces have lower cycle rate (0.755) and reach \texttt{submit:text} (98.7\%) but still fail.
The discriminating feature with the highest F-ratio is \texttt{visit:user:text}, at 7.41 on SWE-agent.

\subsubsection{Failure Cascade Analysis}

We analyze whether failures develop gradually (progressive fitness degradation) or suddenly (abrupt state change).
On SWE-agent, 97\% of failures are sudden (1{,}292 of 1{,}329), with no gradual degradation pattern.
This is consistent across datasets: SWE-smith 100\% sudden.
The implication is that failure monitoring should focus on detecting specific state patterns, such as a threshold on the cycle rate, rather than tracking gradual performance decline.

\subsubsection{Failure Progression}

Failure signatures emerge early in execution.
On SWE-agent, the first divergence between success and failure state distributions occurs at 8.7\% of trace length.
On SWE-smith, divergence occurs at 15.4\%.
The highest-lift failure transition is \texttt{tool:text}$\to$\texttt{assistant:text} on SWE-smith (fail rate 0.90, lift 3.66$\times$ over base rate).
Recovery from failure-indicative states is possible: SWE-agent has 8 recovery states where traces can return to successful trajectories, with 95.7\% recovery rate within 2 steps.

\subsection{Counterfactual Path Analysis}\label{app:counterfactual}

We identify FSM decision points where success and failure paths diverge, measured by the Jensen-Shannon divergence between the outgoing transition distributions.

\begin{wraptable}{r}{0.55\textwidth}
\centering
\caption[Counterfactual path analysis.]{\spaceskip=\fontdimen2\font plus \fontdimen3\font minus 0.05\fontdimen4\font\relax Counterfactual path analysis. Decision points: states with JSD $> 0.001$ between success/failure transitions.}
\label{tab:counterfactual}
\scriptsize
\setlength{\tabcolsep}{3pt}
\begin{tabular}{@{}l rrr rr@{}}
\toprule
Dataset & DecPts & Succ paths & Fail paths & Overlap & Top JSD \\
\midrule
SWE-agent & 7 & 275 & 1{,}496 & 41 & 0.014 \\
SWE-smith & 5 & 330 & 123 & 11 & 0.021 \\
\bottomrule
\end{tabular}
\vspace{-8pt}
\end{wraptable}

\spaceskip=\fontdimen2\font plus \fontdimen3\font minus 0.05\fontdimen4\font\relax
On SWE-agent, failures have 5.4$\times$ more unique paths than successes (1{,}496 vs.\ 275), with only 41 shared paths.
The main decision point is the \texttt{user:text} state (JSD 0.014): at this state, successful traces are more likely to transition to \texttt{submit} (25.3\% success rate) whereas failed traces loop back to \texttt{edit} (9.6\% success rate).
On SWE-smith, the \texttt{tool:tool\_call} state has the highest divergence (JSD 0.021).
Divergence often comes early: 65.3\% of SWE-agent traces diverge within the first 10\% of execution.

\spaceskip=0pt\relax

\subsection{Compression Theory}\label{app:compression}

\begin{figure}[H]
\centering
\includegraphics[width=\linewidth]{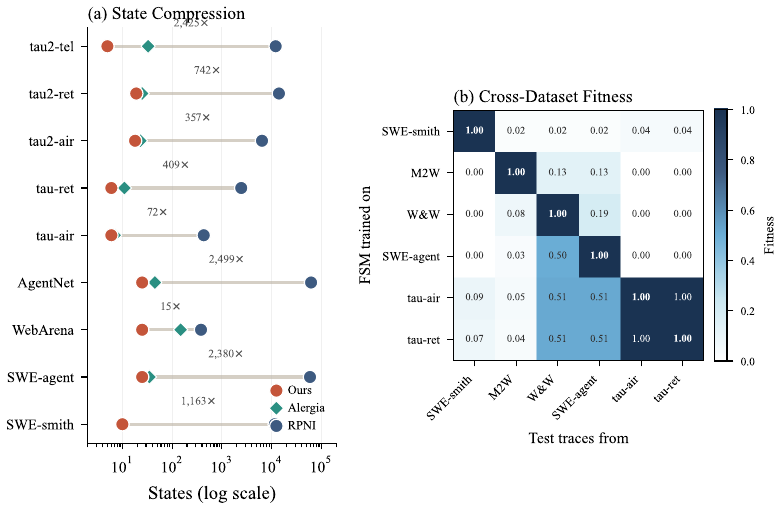}
\caption{(a)~State compression across labeled datasets: our FSM (10--43 states) vs.\ Alergia (10--149) and RPNI (382--63{,}897), with compression ratios annotated. (b)~Cross-dataset fitness matrix: replaying traces from one dataset through another's FSM. Diagonal entries (in-distribution) approach 1.0; off-diagonal entries (OOD) drop to 0.00--0.51, which yields AUROC 1.000 for OOD detection on all cross-dataset pairs.}
\label{app:compression_fig}
\end{figure}

Table~\ref{tab:compression_theory} presents an information-theoretic analysis of compression by the FSM across all datasets.

\begin{table}[H]
\centering
\caption{Information-theoretic compression analysis. MDL: minimum description length ratio (FSM bits / raw trace bits). Gzip: compression ratio of raw sequences. SWE-agent includes rare command variants ($|\cA|{=}27$ vs.\ 24 in Table~\ref{tab:datasets}).}
\label{tab:compression_theory}
\scriptsize
\tablestyle{3pt}{1.02}
\begin{tabular}{@{}l rr rrr r@{}}
\toprule
Dataset & $|\cA|$ & $|Q|$ & FSM bits & Raw bits & MDL & Gzip \\
\midrule
SWE-smith & 9 & 10 & 108 & 107{,}704 & 0.001 & 0.012 \\
SWE-agent & 27 & 25 & 340 & 556{,}450 & 0.001 & 0.018 \\
Mind2Web & 7 & 8 & 87 & 12{,}786 & 0.007 & 0.026 \\
Who\&When & 8 & 9 & 99 & 12{,}276 & 0.008 & 0.019 \\
\bottomrule
\end{tabular}
\end{table}

The FSM achieves MDL ratios of 0.001--0.008 across all datasets, which means that the FSM description requires 0.1--0.8\% of the bits needed to store raw traces.
This is 2.4--29.5$\times$ better than gzip compression (ratios 0.012--0.026), confirming that the FSM captures genuine behavioral regularity beyond statistical redundancy.
All FSMs are deterministic with $|Q|=|\cA|+1$ (verified programmatically) and 100\% alphabet utilization on 3/4 datasets shown.
SWE-agent has alphabet utilization 1.13 (3 states have transitions for symbols not in the core alphabet, reflecting rare command variants).
Conditional entropy analysis: unigram entropy ranges 1.69--2.16 bits; conditioning on the previous symbol (bigram) reduces entropy to 0.63--1.06 bits (51--68\% reduction), confirming strong sequential regularity in agent traces.

\subsection{Path Diversity and Recurrence Analysis}\label{app:archetypes}

Table~\ref{tab:archetypes} quantifies path diversity and recurrence quantification analysis (RQA) metrics across datasets.
RQA treats activity sequences as symbolic time series.

\begin{table}[H]
\centering
\caption{\textbf{Path diversity and RQA} metrics. Singleton\%: paths observed once. Top-5\%: coverage of 5 most common paths. RR/DET/diagEnt: recurrence rate, determinism, diagonal entropy (AUROC for failure prediction). Who\&When and Mind2Web lack success/failure labels (--).}
\label{tab:archetypes}
\label{tab:rqa}
\scriptsize
\tablestyle{3pt}{1.02}
\begin{tabular}{@{}l rr rr cccc@{}}
\toprule
& \multicolumn{4}{c}{Path diversity} & \multicolumn{4}{c}{RQA (AUROC)} \\
\cmidrule(lr){2-5} \cmidrule(lr){6-9}
Dataset & Traces & Unique & Sing.\% & Top-5\% & RR & DET & maxDiag & diagEnt \\
\midrule
SWE-smith & 500 & 442 & 82.0 & 5.8 & \hlbest{0.662} & 0.643 & 0.631 & 0.599 \\
SWE-agent & 2{,}000 & 1{,}730 & 82.0 & 5.1 & 0.695 & 0.694 & 0.695 & \hlbest{0.704} \\
tau2-tel & 1{,}824 & 5 & 0.0 & 100.0 & 0.562 & 0.563 & 0.562 & \hlbest{0.566} \\
tau2-air & 800 & 624 & 67.5 & 6.4 & 0.511 & 0.338 & 0.566 & \hlbest{0.612} \\
tau2-ret & 1{,}824 & 1{,}521 & 73.2 & 2.0 & \hlbest{0.574} & 0.471 & 0.520 & 0.521 \\
GUI-Ody & 7{,}735 & 3{,}601 & 36.7 & 5.2 & 0.223 & 0.296 & \hlbest{0.310} & 0.306 \\
WebArena & 8{,}337 & 827 & 3.8 & 27.7 & \hlbest{0.416} & 0.402 & 0.340 & 0.406 \\
AgentNet & 5{,}000 & 3{,}965 & 73.3 & 3.7 & \hlbest{0.548} & 0.487 & 0.420 & 0.474 \\
Who\&When & 184 & 139 & 60.3 & 13.0 & -- & -- & -- & -- \\
Mind2Web & 500 & 238 & 36.0 & 20.2 & -- & -- & -- & -- \\
\bottomrule
\end{tabular}
\end{table}

On coding agent datasets, 82\% of traces follow unique FSM paths, yet the FSM compresses all into 6--25 states with $\geq$0.999 fitness by capturing \emph{transition topology} rather than memorizing paths.
API-driven domains show opposite extremes: tau2-bench telecom has only 5 unique paths across 1{,}824 traces (highly constrained workflows), while tau2-bench retail has 1{,}521 unique paths (73.2\% singletons).
RQA metrics achieve 0.60--0.70 AUROC on coding agents, with diagonal entropy strongest on SWE-agent (0.704), but are weaker on web/GUI benchmarks (0.31--0.47 on GUI-Odyssey and AgentNet) where trace structures are less recurrent.

\subsection{Loop and Graph Motif Analysis}\label{app:loops}

Table~\ref{tab:loops} analyzes loop patterns (backedge counts) and graph motifs in the FSM.

\begin{table}[H]
\centering
\caption{\textbf{Loop analysis and graph motifs.} Left: loop counts (backedges) per outcome and AUROC. Right: structural motif counts. Bidir: bidirectional edge pairs. Hub degree: maximum out-degree. Who\&When and Mind2Web lack success/failure labels (--).}
\label{tab:loops}
\label{tab:motifs}
\scriptsize
\tablestyle{3pt}{1.02}
\begin{tabular}{@{}l rr r c rrr r@{}}
\toprule
& \multicolumn{4}{c}{Loops (backedges)} & \multicolumn{4}{c}{Graph motifs} \\
\cmidrule(lr){2-5} \cmidrule(lr){6-9}
Dataset & Succ & Fail & Ratio & AUROC & Self & Bidir & Tri & Hub deg \\
\midrule
SWE-agent & 25.3 & 54.8 & \textbf{2.17$\times$} & 0.665 & 0 & 11 & 0 & 15 \\
SWE-smith & 43.6 & 60.7 & 1.39$\times$ & 0.654 & 0 & 6 & 0 & 7 \\
tau2-tel & 1.9 & 1.9 & 1.00$\times$ & 0.562 & 1 & 2 & 2 & 4 \\
tau2-air & 17.1 & 21.5 & 1.26$\times$ & 0.664 & 1 & 14 & 0 & 15 \\
tau2-ret & 20.3 & 21.0 & 1.03$\times$ & 0.537 & 1 & 15 & 0 & 16 \\
GUI-Ody & 12.0 & 9.1 & 0.76$\times$ & 0.362 & 4 & 5 & 6 & 6 \\
WebArena & 9.9 & 6.4 & 0.65$\times$ & 0.289 & 0 & 8 & 0 & 9 \\
AgentNet & 28.1 & 21.0 & 0.75$\times$ & 0.314 & 0 & 9 & 0 & 12 \\
Who\&When & -- & -- & -- & -- & 3 & 5 & 6 & 12 \\
Mind2Web & -- & -- & -- & -- & 4 & 3 & 2 & 6 \\
\bottomrule
\end{tabular}
\end{table}

On coding agents, failed traces contain 1.4--2.2$\times$ more loops, strongest on SWE-agent (2.17$\times$, 54.8 vs.\ 25.3 backedges, AUROC 0.665).
Interestingly, on web/GUI benchmarks the pattern reverses: successful WebArena traces loop \emph{more} (9.9 vs.\ 6.4, ratio 0.65$\times$), reflecting productive exploration in complex navigation tasks. Among API agents, tau2-bench airline shows the strongest signal (1.26$\times$, AUROC 0.664).
Graph motif analysis reveals that tau2-bench airline/retail have the densest bidirectional structure (14--15 pairs), while triangles appear only in Who\&When (6, multi-agent delegation), GUI-Odyssey (6, cross-app navigation), tau2-bench telecom (2), and Mind2Web (2).

\subsection{State Importance Analysis}\label{app:stateimportance}

\begin{table}[H]
\centering
\caption{State importance: fitness drop upon state removal. Top 3 most critical states per dataset.}
\label{tab:state_importance}
\scriptsize
\tablestyle{3pt}{1.02}
\begin{tabular}{@{}l l rr@{}}
\toprule
Dataset & State & Visit\% & Fitness drop \\
\midrule
\multirow{3}{*}{SWE-smith} & system:text & 1.9 & 1.000 \\
 & user:text & 1.9 & 0.979 \\
 & asst:tool:bash & 21.3 & 0.958 \\
\midrule
\multirow{3}{*}{SWE-agent} & user:text & 49.9 & 0.999 \\
 & edit:text & 16.6 & 0.344 \\
 & search:text & 11.5 & 0.208 \\
\midrule
\multirow{3}{*}{Mind2Web} & user:text & 11.7 & 1.000 \\
 & asst:click & 72.0 & 0.710 \\
 & asst:type & 10.2 & 0.102 \\
\bottomrule
\end{tabular}
\end{table}
We measure state importance by fitness drop when each state is removed from the FSM (Table~\ref{tab:state_importance}). State importance is not proportional to visit frequency. On SWE-smith, \texttt{system:text} and \texttt{user:text} each receive only 1.9\% of visits but removing either causes complete fitness collapse (1.000 and 0.979 drop). Conversely, \texttt{tool:text} receives 47.2\% of visits but its removal drops fitness by only 0.937, because its behavioral role can be partially compensated by other states. On SWE-agent, \texttt{user:text} is the single critical bottleneck (49.9\% visits, 0.999 fitness drop), reflecting its role as the central hub connecting all modes of behavior.

\subsection{Critical Transitions and Error Localization}\label{app:bottlenecks}

Table~\ref{tab:bottlenecks} identifies structural bottlenecks (transition criticality = frequency $\times$ success differential) and per-state error rate differentials between success and failure traces.

\begin{table}[H]
\centering
\caption{Critical transitions, with the highest criticality scores at the top, and error localization at the bottom, given as the difference in error rate per state between failure and success.}
\label{tab:bottlenecks}
\label{tab:errorloc}
\scriptsize
\tablestyle{3pt}{1.02}
\begin{tabular}{@{}l l rr@{}}
\toprule
\multicolumn{4}{@{}l}{\textit{Critical transitions}} \\
Dataset & Transition & Freq & Criticality \\
\midrule
SWE-smith & str\_repl$\to$tool:text & 12.96 & 1.048 \\
SWE-smith & tool:text$\to$str\_repl & 12.96 & 1.045 \\
SWE-agent & user:text$\to$edit:text & 9.20 & 0.657 \\
SWE-agent & edit:text$\to$user:text & 9.02 & 0.629 \\
\midrule
\multicolumn{4}{@{}l}{\textit{Error localization (per-state error rate differential)}} \\
Dataset & State & Succ/Fail rate & $\Delta$ \\
\midrule
SWE-smith & asst:text & 0.000 / 0.222 & +.222 \\
SWE-agent & edit:text & 0.368 / 0.486 & +.118 \\
SWE-agent & asst:text & 0.000 / 0.071 & +.071 \\
\bottomrule
\end{tabular}
\end{table}

The transitions with the highest criticality form tight cycles: \texttt{str\_replace\_editor}$\leftrightarrow$\texttt{tool:text} on SWE-smith (1.048) and \texttt{user:text}$\leftrightarrow$\texttt{edit:text} on SWE-agent (0.657).
Errors are localized to specific states: on SWE-smith, \texttt{assistant:text} has zero errors in successes but 22.2\% in failures (+0.222 differential).
The first significant divergence between success and failure distributions occurs at the \texttt{edit:text} vs.\ \texttt{search:text} branch on SWE-agent (position 2).

\subsection{Temporal Dynamics}\label{app:temporal}

We analyze three-phase (early/mid/late) behavioral dynamics.

\begin{table}[H]
\centering
\caption{\textbf{Temporal dynamics:} entropy by execution phase and entropy drift (early $-$ late). Who\&When and Mind2Web lack success/failure labels (--).}
\label{tab:temporal}
\scriptsize
\tablestyle{3pt}{1.02}
\begin{tabular}{@{}l ccc c c@{}}
\toprule
Dataset & Early $H$ & Mid $H$ & Late $H$ & Drift & AUROC \\
\midrule
SWE-agent & 1.83 & 1.55 & 1.69 & $-$0.15 & \textbf{0.705} \\
SWE-smith & 2.17 & 1.45 & 1.72 & $-$0.44 & 0.656 \\
tau2-tel & 1.39 & 0.95 & 0.97 & $-$0.42 & \hlbest{0.731} \\
tau2-air & 2.17 & 1.82 & 1.95 & $-$0.21 & 0.349 \\
tau2-ret & 2.32 & 2.00 & 2.03 & $-$0.29 & 0.436 \\
GUI-Ody & 1.23 & 0.63 & 1.19 & $-$0.04 & 0.349 \\
WebArena & 1.42 & 1.24 & 1.60 & $+$0.19 & 0.671 \\
AgentNet & 2.10 & 1.85 & 2.28 & $+$0.18 & 0.675 \\
Mind2Web & 1.46 & 0.68 & 0.43 & $-$1.03 & -- \\
Who\&When & 1.41 & 0.59 & 0.68 & $-$0.73 & -- \\
\bottomrule
\end{tabular}
\end{table}

Most datasets show negative entropy drift, indicating that agents narrow their behavioral repertoire over time.
The effect is strongest on Mind2Web ($-$1.03) and weakest on GUI-Odyssey ($-$0.04).
Two datasets show \emph{positive} drift: WebArena ($+$0.19) and AgentNet ($+$0.18), where agents diversify behavior in later phases, which possibly reflects recovery or exploration after initial failures.
Phase-specific failure prediction achieves AUROC 0.611--0.731, with tau2-bench telecom showing the strongest early-phase signal (0.731) and SWE-agent the strongest late-phase signal (0.705).

\subsection{Anomaly Detection and Suffix Monitoring}\label{app:anomalydetect}

We evaluate two unsupervised monitoring approaches: (1)~a multi-component anomaly score combining rejection rate, state occupancy deviation, terminal state anomaly, cycle excess, and length deviation; and (2)~suffix-based monitoring using only the last $k$ FSM transitions (Table~\ref{tab:anomalydetect}).

\begin{table}[H]
\centering
\caption{\textbf{Unsupervised monitoring.} Left: anomaly detection components and composite AUROC. Right: suffix monitoring AUROC at window size $k$.}
\label{tab:anomalydetect}
\label{tab:suffix}
\scriptsize
\tablestyle{3pt}{1.02}
\begin{tabular}{@{}l cccc cc cccc@{}}
\toprule
& \multicolumn{6}{c}{Anomaly detection} & \multicolumn{4}{c}{Suffix (last $k$)} \\
\cmidrule(lr){2-7} \cmidrule(lr){8-11}
Dataset & Rej & Occ & Term & Cyc & Comp & P@10 & $k{=}2$ & $k{=}3$ & $k{=}5$ & $k{=}10$ \\
\midrule
SWE-agent & 0.563 & 0.707 & \hlbest{0.711} & 0.697 & \hlbest{0.747} & 1.00 & 0.693 & 0.746 & 0.771 & \hlbest{0.825} \\
SWE-smith & 0.406 & 0.605 & 0.429 & \hlbest{0.679} & 0.653 & 0.20 & 0.540 & 0.546 & 0.550 & \textbf{0.590} \\
\bottomrule
\end{tabular}
\end{table}

The composite anomaly score reaches 0.653--0.747 AUROC, strongest on SWE-agent (0.747, P@10\,=\,1.00).
The dominant component is terminal anomaly on SWE-agent (0.711).
Monitoring only the last $k{=}10$ transitions gives an AUROC (0.825 on SWE-agent) comparable to the full predictor's, which allows real-time deployment with a fixed-size sliding window.
Even $k{=}3$ gives an AUROC of 0.746 on SWE-agent.

\subsection{Agent Integration: Runtime Monitor}\label{app:integration}

We simulate deploying the FSM as a runtime monitor that replays agent actions step-by-step and triggers intervention when learned rules fire.
Rules are derived automatically from training traces as thresholds on the cycle rate, set by percentile, and minimum unique-state counts, with a warm-up period of 10\% of the mean trace length before activation.
A ``stuck'' detector also fires when the agent remains in the same state for 5+ consecutive steps.

\begin{table}[H]
\centering
\caption{Runtime monitor results on test traces. Latency: mean \% of trace at intervention. Saved: mean \% of remaining computation avoided.}
\label{tab:integration}
\scriptsize
\tablestyle{3pt}{1.02}
\begin{tabular}{@{}l ccc cc cc@{}}
\toprule
Dataset & Prec & Recall & F1 & Latency & Saved & Steps saved & Lift \\
\midrule
\rowcolor{cOursBg} SWE-agent       & 85.9\% &  95.5\% & \hlbest{0.904} & 32\% & 68\% & 16{,}711 & 1.02$\times$ \\
\rowcolor{cOursBg} tau2-bench (air) & 76.0\% &  79.2\% & 0.776 & 56\% & 44\% &    325 & 1.27$\times$ \\
tau2-bench (ret) & 31.0\% &  60.0\% & 0.409 & 63\% & 37\% &    247 & 0.95$\times$ \\
SWE-smith        & 16.0\% & 100.0\% & 0.276 & 17\% & 83\% &    936 & 1.00$\times$ \\
\bottomrule
\end{tabular}
\end{table}

The monitor is reliable on FSMs with high structural diversity ($|Q|\geq 6$ active states with distinct failure patterns): SWE-agent ($F_1\!=\!0.904$, lift 1.02$\times$) and tau2-bench airline ($F_1\!=\!0.776$, lift 1.27$\times$).
On the smaller-alphabet datasets, the rule under-discriminates: tau2-bench retail and SWE-smith FSMs are too coarse for cycle-rate to separate failure modes, and the monitor over-triggers.
An $F_1$ of at least 0.776 on two datasets establishes a working operating regime; deploying on small-alphabet domains requires per-dataset rule tuning.

\paragraph{Operating point analysis.}
Table~\ref{tab:integration_threshold} shows precision--recall trade-offs at different cycle-rate thresholds on SWE-agent.

\begin{table}[H]
\centering
\caption{Multi-threshold analysis (cycle-rate only) on SWE-agent. Higher thresholds yield higher precision at the cost of recall.}
\label{tab:integration_threshold}
\scriptsize
\tablestyle{3pt}{1.02}
\begin{tabular}{@{}l cccc@{}}
\toprule
Threshold & Precision & Recall & F1 & Mean latency \\
\midrule
$>$0.750 (p30) & 86.4\% & 85.2\% & 0.858 & 39\% \\
$>$0.842 (p50) & 90.5\% & 65.3\% & 0.759 & 47\% \\
$>$0.917 (p70) & 95.6\% & 38.9\% & 0.553 & 63\% \\
$>$0.957 (p90) & 100.0\% & 11.3\% & 0.203 & 59\% \\
\bottomrule
\end{tabular}
\end{table}

At the high-precision operating point (threshold 0.957), the monitor achieves 100\% precision (zero false alarms) while catching 11.3\% of failures,
which suits automated intervention (e.g., resetting the agent) where false positives are costly.
At the balanced operating point (threshold 0.750), the monitor catches 85.2\% of failures with 86.4\% precision, suitable for alerting a human operator.
The entire monitoring pipeline requires only FSM replay at 0.006\,ms per trace with no ML model training.

\subsection{Sequence-Level vs.\ FSM Feature Comparison}\label{app:sequence}

We compare sequence-level feature representations with FSM per-state features for failure prediction (Table~\ref{tab:sequence}), all using the same L1-regularized logistic regression ($C{=}0.1$, class-weighted).
Sequence features use only the sequence of activity symbols, with no message content: \textbf{bag} is a histogram of activity frequencies ($|\cA|$ features), \textbf{bigram} is the flattened matrix of transition frequencies ($|\cA|^2$ features), \textbf{stats} are sequence statistics (length, log-length, unique activities, entropy, cycle rate, self-loop ratio, late-half entropy, unique ratio; 8 features), \textbf{all-seq} is the concatenation of all three.
The MLP is a 2-layer network (64 and 32 units) trained with early stopping on all sequence features.

\begin{table}[H]
\centering
\caption{Failure prediction: sequence-level vs.\ FSM per-state features. L1-regularized LR ($C{=}0.1$); 10$\times$5-fold CV. $d$: feature dimensionality.}
\label{tab:sequence}
\scriptsize
\setlength{\tabcolsep}{3pt}
\begin{tabular}{@{}l l r cc@{}}
\toprule
Dataset & Features & $d$ & CV AUROC & Holdout \\
\midrule
\multirow{6}{*}{SWE-agent}
 & bag (freq.) & 24 & 0.702 $\pm$ 0.030 & 0.694 \\
 & bigram (trans.) & 62 & 0.699 $\pm$ 0.029 & 0.698 \\
 & stats (seq.) & 6 & 0.757 $\pm$ 0.027 & 0.782 \\
 & all-seq (LR) & 92 & 0.771 $\pm$ 0.027 & 0.793 \\
 & all-seq (MLP) & 92 & 0.711 $\pm$ 0.044 & 0.764 \\
 & \cellcolor{cOursBg}\textbf{FSM per-state} & \cellcolor{cOursBg}\textbf{34} & \hlbest{0.790 $\pm$ 0.021} & \hlbest{0.813} \\
\midrule
\multirow{6}{*}{SWE-smith}
 & bag (freq.) & 9 & 0.685 $\pm$ 0.054 & 0.708 \\
 & bigram (trans.) & 18 & 0.685 $\pm$ 0.052 & 0.708 \\
 & stats (seq.) & 6 & 0.662 $\pm$ 0.056 & 0.690 \\
 & all-seq (LR) & 33 & 0.686 $\pm$ 0.051 & 0.703 \\
 & all-seq (MLP) & 33 & 0.637 $\pm$ 0.062 & 0.691 \\
 & \cellcolor{cOursBg}\textbf{FSM per-state} & \cellcolor{cOursBg}\textbf{49} & \hlbest{0.688 $\pm$ 0.050} & \hlbest{0.718} \\
\bottomrule
\end{tabular}
\end{table}

FSM per-state features outperform all sequence-level representations on all datasets in both CV and holdout AUROC.
The advantage is largest on SWE-agent (+0.019 CV, +0.020 holdout) where per-state content features capture behavioral differences that flat frequency counts cannot localize.
On SWE-smith, the gap is minimal (+0.002 CV) because the 10-state FSM with 9-symbol alphabet provides limited decomposition advantage.
Notably, the MLP underperforms logistic regression on all datasets, confirming that the performance ceiling is data-limited (400--1{,}600 training traces) rather than model-limited.

\end{document}